\documentclass[11pt]{article}

\usepackage[utf8]{inputenc}
\usepackage[T1]{fontenc}
\usepackage[a4paper,left=2cm,right=2cm,top=3cm]{geometry}
\usepackage{indentfirst}
\usepackage{amsmath,amssymb,amsthm,amsfonts,amsbsy}
\usepackage{mathtools}
\usepackage{nicefrac}
\usepackage{dsfont}

\usepackage{bm}
\usepackage{physics}
\usepackage{xfrac}
\usepackage{framed}

\usepackage[cal=euler]{mathalfa}
\usepackage{libertine}
\usepackage[libertine,smallerops]{newtxmath}

\usepackage{graphicx}
\usepackage{wrapfig}
\usepackage[font=small]{caption}
\usepackage{subcaption}
\usepackage{booktabs}
\usepackage{multirow}
\usepackage{float}
\usepackage{placeins}
\usepackage{enumitem}
\usepackage{algorithm}
\usepackage{algpseudocode}

\usepackage{authblk}
\usepackage[
  backend=biber,
  style=numeric-comp,
  sorting=none,
  giveninits=true,
  natbib=true
]{biblatex}
\DeclareFieldFormat{doi}{%
  \mkbibacro{DOI}\addcolon\space\nolinkurl{#1}%
}
\AtEveryBibitem{%
  \clearfield{url}%
  \clearfield{month}%
  \clearfield{eprint}%
  \clearfield{address}%
  \clearfield{location}%
  \clearfield{editor}%
}
\defbibheading{largebibliography}[\refname]{%
  \section*{\centerline{\huge{#1}}\vspace*{0.5cm}}%
  \addcontentsline{toc}{section}{#1}%
}
\usepackage[dvipsnames]{xcolor}
\usepackage{url}
\usepackage{doi}
\usepackage{hyperref}
\hypersetup{
  pdfauthor={Yordan Raykov and Rodrigo Veiga},
  pdftitle={Information-Geometric Forward Policy Training in GFlowNets},
  colorlinks,
  linktocpage=true,
  pdfstartpage=1,
  pdfstartview=FitV,
  breaklinks=true,
  pdfpagemode=UseOutlines,
  pageanchor=true,
  plainpages=false,
  bookmarksnumbered,
  bookmarksopen=true,
  bookmarksopenlevel=1,
  hypertexnames=true,
  pdfhighlight=/O,
  urlcolor=teal,
  linkcolor=Blue,
  citecolor=NavyBlue
}

\usepackage{etoolbox}
\usepackage[toc,page]{appendix}
\appto\appendix{\counterwithin{equation}{section}}

\newtheorem{prop}{Proposition}
\newtheorem{thrm}{Theorem}
\newtheorem{cor}{Corollary}
\DeclareMathOperator{\Diag}{Diag}
\DeclareMathOperator{\Cov}{Cov}

\newcommand{\KL}{\mathrm{KL}}
\DeclareMathOperator{\softmax}{softmax}

\title{Information-Geometric Forward Policy Training in GFlowNets\vspace*{0.8em}}

\author{Yordan Raykov}
\author{Rodrigo Veiga}
\affil{\small School of Mathematical Sciences, University of Nottingham,
Nottingham, UK}

\date{}

\begin{document}

\maketitle

\vspace{-2.5ex}
\begin{abstract}
Generative Flow Networks (GFlowNets) have emerged as a flexible framework for
amortised inference over discrete and mixed discrete-continuous objects,
requiring only an unnormalised target density specified through a reward. In
this work, we formulate forward-policy training in GFlowNets through the
information geometry of the induced trajectory sampler. Treating the forward
policy as an induced trajectory sampler, we show that its intrinsic first-order
geometry is given by the Fisher-Rao metric of the trajectory family, and that
the associated natural gradient provides the canonical local update whenever
the corresponding Fisher information is computable or accurately
approximable. We derive an exact decomposition of the trajectory Fisher into
per-step conditional second moments, which clarifies when temporal score
interactions vanish and when dense couplings remain under shared
parameterisation. This leads to three computational regimes: settings with
tractable exact Fisher information, settings where Monte Carlo estimators of
the expected Fisher are sufficient, and structure-exploitable settings in
which target locality or factorisation yields accurate approximations of the
Fisher expectation. In the latter case, graphical-model tools such as exact
marginalisation, separator methods, and belief propagation provide principled
surrogates for natural-gradient updates. The resulting framework turns target
structure into optimisation geometry and yields a tractable route to
structure-aware forward-policy training in GFlowNets. We illustrate the
framework empirically through examples comparing convergence and exploration
behaviour under Riemannian and Euclidean optimisation.
\end{abstract}

\section{Introduction}

GFlowNets are amortised samplers particularly suitable for distributions over discrete or compositional objects~\cite{bengio2021flow,lahlou2023_continuous}. Given an unnormalised target density specified through a reward \(R(x)\), a GFlowNet learns a stochastic construction policy whose terminal distribution is proportional to \(R\) \cite{bengio2023gflownetfoundations}. This is particularly appealing in discrete, discrete-continuous or ordinal inference problems, where samples from the target distribution are expensive or unavailable due to an intractable normalisation constant; however unnormalised reward evaluations are feasible~\cite{zhang2022discreteprobabilistic,
deleu2022bayesian,hu2023gflownetem}. Sequential construction is attractive because it turns sampling from a complex combinatorial distribution into a sequence of local decisions, yielding expressive
multimodal samplers without requiring a single global parametric family over the full object space. The difficulty is that GFlowNets learn from trajectories generated by their current forward policy while the target is observed only through unnormalised rewards.
Unlike MCMC, where iterative proposals are continually corrected by target-density evaluations, a GFlowNet seeks to amortise exploration into a forward policy (i.e. direct sampler). Regions that the current policy assigns negligible probability are therefore rarely sampled, even if they have high reward, so they contribute little training signal.

GFlowNets approach this amortised sampling problem by imposing flow-consistency constraints on the construction process rather than requiring samples from the target distribution. Flow matching and detailed balance enforce local conservation constraints whose solutions produce terminal probabilities proportional to reward \cite{bengio2023gflownetfoundations}. Trajectory balance moves this consistency condition to complete trajectories, improving long-horizon credit assignment without changing the target fixed point \cite{malkin2022trajectory}, while subtrajectory balance interpolates between local and global constraints to improve stability \cite{madan2023learning}. These objectives define correct fixed points, but training still depends on which trajectories are observed \cite{shen2023towards}.

This places the distribution induced by the forward policy at the centre of the training
problem. Let \(\theta\) denote the parameters of the forward policy. In a tabular
parameterisation, \(\theta\) contains one action logit for each admissible state-action
pair; in a neural parameterisation, it contains the shared network weights. The policy \(\pi_\theta\) induces a trajectory law
\(  q_\theta(\tau)=\prod_{t=0}^{T-1}\pi_\theta(a_t\mid s_t) \),
where \(\tau\) denotes a trajectory, \(s_t\) the state at step \(t\), and \(a_t\) the corresponding action. The policy hence induces a terminal law \(q_\theta(x)\), which is the amortised approximation to the
reward-proportional target distribution. Thus the object being optimised is not merely
the coordinate vector \(\theta\), but a parameterised family of probability laws. Although GFlowNet objectives characterise the desired terminal law at exact optima,
practical training is mediated by finite-capacity forward, backward, and flow
parameterisations. When balance constraints are only approximately realisable over the
state directed acyclic graph (DAG), the learned terminal law may deviate from the target
\cite{shen2023towards,silva2025gflownets}. This motivates asking how to update the induced
distribution itself, rather than only how to update its coordinates.

Although GFlowNet trajectories are discrete, a differentiable forward policy traces a smooth family \(\{q_\theta\}\) of discrete probability laws as its parameters vary. The interior of the probability simplex carries the Fisher-Rao geometry \cite{amari1982differential}, which
describes the second-order local change in KL divergence between nearby probability laws. Pulling this geometry back through the map
\(\theta\mapsto q_\theta\) gives the Fisher information \(F(\theta)\) on policy
parameters. The natural gradient is consequently the steepest descent direction per unit local change in the induced trajectory law, rather than per unit Euclidean change in \(\theta\); its induced tangent direction
on the statistical model is invariant under smooth full-rank reparameterisations
\cite{amari1998natural}. Applying this principle to GFlowNets is not automatic: state-dependent action masks produce different local probability simplices; weakly visited states and nearly deterministic decisions yield small or zero Fisher curvature; and shared neural parameters couple decisions across states and construction steps. The computational question is therefore when the trajectory Fisher can
be formed, solved, or approximated reliably. It may be diagonal or block-diagonal in simple tabular cases, consistently estimable from
trajectories in Monte Carlo regimes, or approximable through graphical-model machinery when the target factorises.

\paragraph{Related work.} 

A complementary line of GFlowNet work addresses poor coverage by changing data collection or introducing additional samplers. Adaptive Teachers train an
auxiliary behaviour model to prioritise regions where the primary sampler has high error \cite{kim2025adaptive}; SA-GFN~\cite{madan2025towards} and ACE~\cite{dallantonia2026avoid} likewise use separate exploration policies. Boosted GFlowNets instead train a sequence of samplers on residual rewards \cite{dallantonia2026boosted}. These approaches alter where training data come from or how multiple samplers are combined. They are complementary to our question: holding the canonical forward sampler and balance objective fixed, what geometry should govern its parameter update?

Natural gradients are classical in statistical learning and policy optimisation \cite{amari1998natural,kakade2001natural,martens_2020}, and scalable Fisher approximations such as K-FAC are well established \cite{martens2015kfac}. Our contribution is therefore not a generic Fisher-based optimiser. Unlike natural policy gradients for expected-return objectives, we retain the GFlowNet balance objective and use the trajectory law induced by its forward policy to define the local metric. We derive the
resulting Fisher structure under sequential construction and analyse when sampled or structure-informed approximations preserve its natural direction. This also distinguishes the model Fisher used here from arbitrary
outer products of training gradients, whose use to form an empirical Fisher can be misleading \cite{kunstner2019limitations}.

Our use of Fisher-Rao geometry is also distinct both from continuous relaxations and from geometric flow matching for discrete data. Concrete and Gumbel-Softmax replace categorical samples by differentiable continuous
surrogates \cite{maddison2017the,jang2017categorical}, whereas
Fisher-Flow and categorical flow matching transport probability mass on categorical statistical manifolds
\cite{davis2024fisher,cheng2024categorical}. We retain the validity-preserving discrete construction process and use Fisher-Rao geometry to update its trajectory sampler.

\paragraph{Main contributions.}
\begin{itemize}[leftmargin=*,nosep]
	\item We formulate forward-policy training on the statistical manifold of induced trajectory laws and identify the corresponding trajectory Fisher on its identifiable tangent space.
	\item We derive an exact decomposition of the trajectory Fisher into per-step conditional second moments, and make explicit how this yields occupancy-weighted block structure in tabular policies but dense couplings under shared parameterisation.
	\item We organise natural-gradient GFlowNet training into exact, sampled, and structure-exploitable regimes, and give relative-error guarantees, with additive sufficient conditions, under which target- or construction-informed Fisher surrogates approximate the damped natural-gradient direction.
	\item We show empirically, across tabular and neural discrete benchmarks and on Bayesian DAG learning with the real Sachs protein-signalling data, that Fisher-Rao preconditioning can improve convergence and distribution matching relative to Euclidean training using exact, sampled, and structured Fisher operators.

\end{itemize}

\paragraph{Outline.} Section~\ref{sec:preliminaries} reviews standard background on natural gradients and Fisher-Rao geometry. Sections~\ref{sec:fisher_rao_structure}-\ref{sec:examples} present our contributions: Section~\ref{sec:fisher_rao_structure} derives the Fisher-Rao geometry
 of policy-induced samplers, Section~\ref{sec:natgrad_gflownet_training} develops the three training regimes and approximation guarantees, and Section~\ref{sec:examples} reports numerical examples. The supplement contains
proofs and additional experiments.

 \paragraph{Reproducibility.} Code in \url{https://github.com/rodsveiga/infogeometric_gflows}.

\section{Preliminaries: natural gradients and Fisher-Rao geometry}
\label{sec:preliminaries}

Let \(\mathcal M=\{q_\theta:\theta\in\Theta\}\) be a differentiable statistical model, with open parameter space \(\Theta\subseteq\mathbb R^d\). In statistical optimisation the object of interest is the probability law \(q_\theta\), not the particular coordinate vector \(\theta\). A Euclidean update is therefore coordinate-dependent: two smooth parameterisations of the same family can induce different Euclidean gradient directions on the same statistical model.

The Fisher information matrix
\begin{equation}
\label{eq:prelim_fisher}
F(\theta)
\coloneqq
\mathbb E_{z\sim q_\theta}
\Big[
\nabla_\theta \log q_\theta(z)\,
\nabla_\theta \log q_\theta(z)^\top
\Big]
\end{equation}
defines the Fisher-Rao Riemannian metric on the identifiable tangent space of \(\mathcal M\), under standard regularity conditions \cite{amari2000methods}. Its role is characterised by the local KL divergence: for a small parameter displacement $\Delta\theta\in \mathbb R^d$ with $\theta+\Delta\theta\in \Theta$,
\begin{equation}
\label{eq:prelim_kl_expansion}
\KL(q_\theta\|q_{\theta+\Delta\theta})
=
\tfrac{1}{2}\,\Delta\theta^\top F(\theta)\Delta\theta
+
o(\|\Delta\theta\|^2),
\end{equation}
so the Fisher metric measures local displacement in distribution space rather than in ambient parameter space. For a differentiable objective \(\mathcal J:\Theta\to\mathbb R\), the identifiable natural gradient is
\begin{equation}
\label{eq:prelim_natgrad}
\widetilde\nabla_\theta \mathcal J(\theta)
\coloneqq
F(\theta)^\dagger\nabla_\theta \mathcal J(\theta),
\end{equation}
where \(F^\dagger\) denotes an inverse, pseudo-inverse, or damped inverse, depending on identifiability and conditioning. Equivalently, up to positive scaling, the natural-gradient direction is characterised
by the local KL trust-region problem:
\begin{equation}
\label{eq:prelim_trust_region}
\max_{\Delta\theta}\;
\nabla_\theta\mathcal J(\theta)^\top\Delta\theta
\quad\text{s.t.}\quad
\tfrac12\Delta\theta^\top F(\theta)\Delta\theta\le\varepsilon ,
\end{equation}
for some $\varepsilon > 0$.Thus, the natural gradient is the steepest-ascent direction when distance is measured by local KL geometry. 
This yields two standard consequences: the induced natural-gradient tangent vector on
\(\mathcal M\) is invariant under smooth reparameterisations of identifiable
coordinates, and the Fisher metric rescales directions according to their local effect
on the distribution. Full derivations, including the KL expansion, trust-region characterisation, reparameterisation calculation, and local descent properties, are given in Appendix~\ref{app:prelim_natgrad_properties}.

Focusing on finite discrete distributions, let \(\mathcal X\) be finite and let \(\mathcal M=\{p_\theta:\theta\in\Theta\}\) be a differentiable family of strictly positive probability mass functions on \(\mathcal X\). Then
\begin{equation}
\label{eq:prelim_discrete_fisher}
F(\theta)
=
\sum_{x\in\mathcal X}
p_\theta(x)\,
\nabla_\theta\log p_\theta(x)\,
\nabla_\theta\log p_\theta(x)^\top.
\end{equation}
Hence finite discrete models are Fisher-Rao manifolds in the same
information-geometric sense as continuous models. Under the square-root embedding \(u_i=2\sqrt{p_i}\), the simplex interior is isometric to the positive orthant of the sphere of radius \(2\), making this continuous geometry explicit; the precise statement and boundary caveats are given in Appendix~\ref{app:prelim_discrete_geometry}.

For policy-induced samplers, the relevant discrete building block is the categorical transition rule. At a state \(s\), a forward policy chooses an action \(a\in\mathcal A(s)\) from a categorical distribution \(\pi_\theta(\cdot\mid s)\). If this categorical law is parameterised locally by logits \(\ell\in\mathbb R^{|\mathcal A(s)|}\), then \(\nabla_\ell\log \pi(a\mid s)=\mathrm{e}_a-\pi(\cdot\mid s)\), where $\mathrm{e}_a$ is the standard basis vector corresponding to 
action $a$. The corresponding Fisher block is
\begin{equation}
\label{eq:prelim_cat_block}
C(\pi(\cdot\mid s))
=
\Diag(\pi(\cdot\mid s))
-
\pi(\cdot\mid s)\pi(\cdot\mid s)^\top ,
\end{equation}
where \(\Diag(v)\) denotes the diagonal matrix with diagonal entries given by
the vector \(v\). This block is singular in the all-ones direction, reflecting the invariance of softmax probabilities to additive shifts of the logits. Although the Fisher-Rao metric is always well defined on a fixed-support discrete family, its computability depends on structure. For an exponential family
\(p_\theta(x)=\exp\{\theta^\top T(x)-A(\theta)\}\), the Fisher is \(F(\theta)=\nabla_\theta^2 A(\theta)=\Cov_{p_\theta}[T(X)]\).
Thus Fisher computation is tractable when the required moments are tractable, as in factorised or low-treewidth models, but may be as hard as inference itself in globally coupled combinatorial models. Representative examples are collected in
Appendix~\ref{app:prelim_discrete_geometry}. 

\section{Fisher-Rao geometry of GFlowNets}
\label{sec:fisher_rao_structure}

A GFlowNet forward policy defines a stochastic construction process on a state
DAG: starting from a source state \(s_0\), actions extend a partial object until a terminal object \(x\in\mathcal X\) is produced.  For a trajectory
\(
\tau=(s_0,a_0,s_1,a_1,\dots,s_T)\) with \( s_{t+1}=f(s_t,a_t)\),
the induced trajectory law is
\begin{equation}
\label{eq:traj_law_main}
q_\theta(\tau)
=
\prod_{t=0}^{T-1}\pi_\theta(a_t\mid s_t).
\end{equation}
The terminal law is obtained by marginalisation,
\(q_\theta(x)
=
\sum_{\tau:\Phi(\tau)=x} q_\theta(\tau)
\), where \(\Phi(\tau)\) denotes the terminal object of trajectory \(\tau\).
GFlowNet objectives such as Trajectory Balance (TB) or Detailed Balance (DB) train the forward
policy together with auxiliary backward-policy or flow quantities so that, at a realisable optimum, the induced terminal law matches a target proportional to reward. We summarise the standard realisability assumptions in Appendix~\ref{app:gflownet_realizability}.

The sequential factorisation in \eqref{eq:traj_law_main} is available at the trajectory
level and here we study the corresponding trajectory Fisher:
\begin{equation}
\label{eq:exact_fisher_main}
F(\theta)
\coloneqq
\mathbb E_{\tau\sim q_\theta}
\Big[
\nabla_\theta \log q_\theta(\tau)\,
\nabla_\theta \log q_\theta(\tau)^\top
\Big].
\end{equation}
The induced terminal Fisher coincides with \eqref{eq:exact_fisher_main} in canonical
constructions with a unique trajectory per terminal object. In non-canonical
constructions the two differ by the conditional covariance of trajectory scores given
the terminal object; see Appendix~\ref{app:policy_induced_fisher}. We work with the trajectory Fisher because it is attached directly to the sequential forward policy and admits the per-step decomposition below. Let \(\mathcal J(\theta)\) denote any
differentiable objective for the forward
sampler, and write \(h(\theta)\coloneqq\nabla_\theta \mathcal J(\theta)\). The identifiable score tangent subspace of the trajectory model is
\begin{equation}
\label{eq:score_tangent_main}
\mathcal U_\theta
\coloneqq
\operatorname{range}F(\theta)
=
\operatorname{span}\!\left\{
\nabla_\theta\log q_\theta(\tau):q_\theta(\tau)>0
\right\}.
\end{equation} 
Because \(F\) may be singular, interpreting \(F^\dagger h\) as the solution of
\(Fv=h\) on the identifiable tangent space requires the objective gradient to have no
component in a score-null direction. This compatibility holds automatically for
score-form gradients.

\begin{prop}[Score-form gradients lie in the Fisher range]
\label{prop:score_compatible_gradient}
Let \(z(\tau)=\nabla_\theta\log q_\theta(\tau)\) satisfy
\(\mathbb E_{q_\theta}\|z(\tau)\|^2<\infty\), and let
\(r(\tau)\) be a measurable real-valued function of the trajectory satisfying
\(\mathbb E_{q_\theta}[r(\tau)^2]<\infty\). If
\(h=\mathbb E_{q_\theta}[z(\tau)r(\tau)]\), then
\[
h\in\mathcal U_\theta
\qquad\text{and}\qquad
h^\top F^\dagger h
\le
\mathbb E_{q_\theta}[r(\tau)^2].
\]
\end{prop}

\noindent\emph{Proof.}
See Appendix~\ref{app:score_compatibility}.
\par\medskip
Section~\ref{sec:natgrad_gflownet_training} makes this representation explicit for
the TB forward-policy gradient. Thus score-form GFlowNet gradients do not request movement along
\(\ker(F)\), although they may still have components in directions of weak positive
curvature. This range property is distinct from the martingale-centering argument
below, which removes cross-time terms from the Fisher.

Let \(u_t(\theta;\tau)\coloneqq\nabla_\theta\log\pi_\theta(a_t\mid s_t)\), with
\(u_t=0\) after termination, so $\nabla_\theta\log q_\theta(\tau)=\sum_{t\ge 0}u_t(\theta;\tau)$. Whenever \(h(\theta)\in\mathcal U_\theta\), the preliminaries identify
\(F(\theta)^\dagger h(\theta)\) as the undamped natural-gradient direction. Sequential
construction gives the following additional structure.

\begin{thrm}[Trajectory Fisher decomposition]
\label{thm:riemannian_policy_induced}
Assume each local policy is differentiable and strictly positive on a valid-action
support independent of \(\theta\), and assume a bounded construction horizon or, more generally,
\(\sum_{t\ge0}\mathbb E\|u_t\|^2<\infty\). Let
\(\mathcal F_t\coloneqq\sigma(s_0,a_0,\dots,a_{t-1},s_t)\), where
\(\sigma(\cdot)\) denotes the sigma-algebra generated by the enclosed random variables.
Then
\(\mathbb E[u_t(\theta;\tau)\mid\mathcal F_t]=0\),
and
\begin{equation}
\label{eq:exact_temporal_fisher_main}
F(\theta)
=
\sum_{t\ge0}
\mathbb E_{\tau\sim q_\theta}
\big[
u_t(\theta;\tau)u_t(\theta;\tau)^\top
\big].
\end{equation}
\end{thrm}

\noindent\emph{Proof.}
See Appendix~\ref{app:policy_induced_fisher}.
\par\medskip
The per-step scores are therefore martingale differences, which removes temporal cross
terms from the trajectory Fisher. Combined with the preliminaries, this decomposition
shows that \(F^\dagger h\) represents the intrinsic natural-gradient tangent direction
whenever \(h\in\mathcal U_\theta\). It does not imply diagonal or block-diagonal structure in
parameter coordinates; that depends on the policy parameterisation.

In the tabular case this distinction becomes explicit. Suppose every non-terminal
state \(s\in\mathcal S\) has its own logit vector
\(\theta_s\in\mathbb R^{|\mathcal A(s)|}\), and write
\(\pi_s=\softmax(\theta_s)\) and
\(C(\pi_s)=\Diag(\pi_s)-\pi_s\pi_s^\top\). Define its expected occupancy of state \(s\) by
\begin{equation}
\label{eq:occupancy_main}
d_\theta(s)
\coloneqq
\mathbb E_{\tau\sim q_\theta}
\Big[
\sum_{t\ge 0}\mathds{1}\{s_t=s\}
\Big] ,
\end{equation}
where \(\mathds{1}\{\cdot\}\) denotes the indicator function.
\begin{cor}[Exact occupancy-weighted Fisher for tabular policies]
\label{cor:exact_tabular_fisher_main}
Under this tabular softmax parameterisation, the Fisher and corresponding
exact statewise natural-gradient component are
\[
F(\theta)
\!=\!
\bigoplus_{s\in\mathcal S} d_\theta(s)\,C(\pi_s),
\qquad
\widetilde\nabla_{\theta_s}\mathcal J(\theta)
\!=\!
\big(d_\theta(s)\,C(\pi_s)\big)^\dagger
\nabla_{\theta_s}\mathcal J(\theta).
\]
\end{cor}
\noindent\emph{Proof.}
See Appendix~\ref{app:tabular_shared_specializations}.
\par\medskip
For shared-parameter policies, the temporal decomposition remains exact but parameter
coordinates are coupled. If \(\pi_\theta(\cdot\mid s)=\softmax(\ell_\theta(s))\) and
\(J_\theta(s)\coloneqq \nabla_\theta \ell_\theta(s)
\in\mathbb R^{|\mathcal A(s)|\times \dim(\theta)}\)
is the logit Jacobian, then
\begin{equation}
\label{eq:shared_exact_fisher_main}
F(\theta)
=
\mathbb E_{\tau\sim q_\theta}
\left[
\sum_{t\ge 0}
J_\theta(s_t)^\top
C(\pi_\theta(\cdot\mid s_t))
J_\theta(s_t)
\right].
\end{equation}
The derivation is given in
Appendix~\ref{app:tabular_shared_specializations}. Thus the tabular parameterisation yields an exact block-diagonal Fisher, whereas shared representations generally induce dense curvature. Exact block structure can also arise
beyond state-tabular policies, but only when the conditional factorisation of the policy is aligned with its parameter blocks.

For a hierarchical action, conditional score centering eliminates cross moments
between factors; disjoint parameter blocks then yield an exact block-diagonal
Fisher. The derivation is given in Appendix~\ref{app:conditional_fisher_factorization}. Crucially, target
factorisation alone is insufficient. In Bayesian DAG learning, a decomposable
score such as Bayesian Gaussian equivalent (BGe) makes child-node parent sets natural target units, but an
order-free edge-insertion policy couples admissible actions through acyclicity,
and a shared Linear Transformer couples parameter blocks
\cite{deleu2022bayesian}. Child-node, output-head, or layerwise blocks are then
surrogates for a dense neural Fisher, not consequences of reward factorisation.

It is useful to separate three operations that are otherwise easy to conflate. The
trajectory law fixes the exact Fisher \(F\). Knowledge of the target and construction process is then used to build a surrogate \(\widetilde F\), for example from sampled
occupancies, nodewise blocks, or message-passing beliefs; gradient estimation may
separately produce \(\widetilde h\). Finally, these quantities are combined in a damped
linear solve. The Moore--Penrose pseudo-inverse does not approximate \(F\): when
\(h\in\mathcal U_\theta\), it solves the exact singular system on that identifiable
subspace. The errors \(\widetilde F-F\) and \(\widetilde h-h\) measure the two
approximations, while damping controls inverse stability.
\begin{thrm}[Target-informed Fisher approximation]
\label{thm:structure_dependent_approx}
For fix \(\theta\), let \(F\) be the exact trajectory Fisher,
\(\widetilde F\succeq0\) a target- or
construction-informed surrogate, and let \(h,\widetilde h\) be the exact and approximate
objective gradients, respectively. For \(\lambda>0\), define
\[
F_\lambda=F+\lambda I,
\qquad
\widetilde F_\lambda=\widetilde F+\lambda I,
\qquad
v_\lambda=F_\lambda^{-1}h,
\qquad
\widetilde v_\lambda=\widetilde F_\lambda^{-1}\widetilde h.
\]
Here \(v_\lambda\) is the damped exact-Fisher reference direction, whereas
\(\widetilde v_\lambda\) also uses the Fisher and gradient surrogates.
For \(A\succ0\), write \(\|x\|_A\coloneqq(x^\top A x)^{1/2}\).
If the damped relative Fisher error
\[
\rho_\lambda
\coloneqq
\left\|
F_\lambda^{-1/2}
(\widetilde F-F)
F_\lambda^{-1/2}
\right\|_{\mathrm{op}}
<1,
\]
then
\begin{equation}
\label{eq:structured_natdir_bound_main}
\|\widetilde v_\lambda-v_\lambda\|_{F_\lambda}
\le
\frac{\rho_\lambda}{1-\rho_\lambda}
\|v_\lambda\|_{F_\lambda}
+
\frac{1}{1-\rho_\lambda}
\|\widetilde h-h\|_{F_\lambda^{-1}}.
\end{equation}
If \(\widetilde h=h\), its pairing with the exact gradient also satisfies
\[
\frac{1}{1+\rho_\lambda}\,h^\top F_\lambda^{-1}h
\le
h^\top\widetilde F_\lambda^{-1}h
\le
\frac{1}{1-\rho_\lambda}\,h^\top F_\lambda^{-1}h.
\]
\end{thrm}

\noindent\emph{Proof.}
See Appendix~\ref{app:policy_induced_approximation}.
\par\medskip
Per-step surrogate errors control the additive operator error, but
Theorem~\ref{thm:structure_dependent_approx} additionally measures that error
relative to the curvature scale of \(F+\lambda I\); an additive bound alone is
not sufficient.

For a block surrogate \(B\succeq0\), the theorem has a direct coupling
interpretation. Define
\[
\gamma_{B,\lambda}
\coloneqq
\left\|
(B+\lambda I)^{-1/2}
(F-B)
(B+\lambda I)^{-1/2}
\right\|_{\mathrm{op}}.
\]

If \(\gamma_{B,\lambda}<1\), then
\[
(1-\gamma_{B,\lambda})(B+\lambda I)
\preceq F+\lambda I
\preceq
(1+\gamma_{B,\lambda})(B+\lambda I),
\]
and, for \(\widetilde F=B\), the relative error in Theorem~\ref{thm:structure_dependent_approx} is at most \(\gamma_{B,\lambda}/(1-\gamma_{B,\lambda})\). Thus \(\gamma_{B,\lambda}<1/2\) is a simple sufficient condition for \(\rho_\lambda<1\), and hence for the theorem's bound. Factorisation can therefore suggest a sparse block surrogate, but the normalised
coupling omitted by that surrogate determines whether the approximation is
guaranteed to preserve the damped Fisher geometry.

When \(\widetilde h=h\), the guarantee also yields strict local descent for the
step-size range and smoothness conditions in
Appendix~\ref{app:approximation_consequences}.  Let
\(\delta=\|\widetilde F-F\|_{\mathrm{op}}\) and
\(\mu_\lambda=\lambda_{\min}(F+\lambda I)\ge\lambda\). If
\(\delta<\mu_\lambda\), then
\(\|\widetilde F_\lambda^{-1}\|_{\mathrm{op}}
\le(\mu_\lambda-\delta)^{-1}\); hence \(\delta<\lambda\) is a conservative
sufficient condition. For tabular softmax policies, the minimum positive state
occupancy \(d_{\min}\) and minimum valid-action probability \(p_{\min}\) give the smallest positive-eigenvalue bound
\(\lambda_{\min}^{+}(F)\ge d_{\min}p_{\min}\). A useful global bound is
generally unrealistic for shared neural policies; curvature is directional and
weakens under poor visitation, action collapse, or softmax-invariant
perturbations. Precise bounds, including the relative-logit sensitivity, are
given in Appendix~\ref{app:fisher_conditioning}.

Theorem~\ref{thm:structure_dependent_approx} therefore replaces an unconditional
claim about approximate Fishers with an explicit relative-error criterion. The next
section shows how exact, sampled, and structure-informed Fisher operators enter
GFlowNet training.

\section{Natural-gradient GFlowNet training and computational regimes}
\label{sec:natgrad_gflownet_training}

We now instantiate the preceding geometry for TB and separate
gradient estimation from the computational route used to access the Fisher. Let
\(\phi\) denote backward-policy parameters and \(\xi\) a source-flow parameter. For a
trajectory \(\tau:s_0\to x\), the TB residual~\cite{malkin2022trajectory} is
\begin{equation}
\label{eq:tb_residual_training}
\delta_{\theta,\phi,\xi}(\tau)
=
\log Z_\xi
+
\sum_{t=0}^{T-1}\log \pi_\theta(a_t\mid s_t)
-
\log R(x)
-
\sum_{t=0}^{T-1}\log P_\phi(a_t\mid s_{t+1}).
\end{equation}
For a differentiable scalar loss \(\ell:\mathbb R\to\mathbb R_+\) and a per-update
training distribution \(\mu\), set
\begin{equation}
\label{eq:tb_loss_training}
\mathcal J_{\mathrm{TB}}(\theta,\phi,\xi)
=
\mathbb E_{\tau\sim \mu}
\big[
\ell(\delta_{\theta,\phi,\xi}(\tau))
\big],
\end{equation}
where \(\ell(u)=u^2\) gives the usual squared TB loss. Holding \(\mu\) fixed during
the update, assume \(\mu\ll q_\theta\), as holds for training distributions
supported on valid trajectories under the strictly positive policy. Writing
\(z(\tau)=\nabla_\theta\log q_\theta(\tau)=\sum_tu_t(\theta;\tau)\), the
forward-policy gradient satisfies
\[
\begin{aligned}
h(\theta)
&=
\mathbb E_{\tau\sim\mu}
\!\left[
\ell'(\delta_{\theta,\phi,\xi}(\tau))\,z(\tau)
\right] \\
&=
\mathbb E_{\tau\sim q_\theta}
\left[
z(\tau)\,
\frac{\mu(\tau)}{q_\theta(\tau)}
\ell'(\delta_{\theta,\phi,\xi}(\tau))
\right].
\end{aligned}
\]

Proposition~\ref{prop:score_compatible_gradient} therefore applies whenever the
scalar multiplier of \(z(\tau)\) is square-integrable. Importance-corrected variants
follow from the same change-of-measure argument.

With exact or sampled estimators \(\widehat h\) and \(\widehat F\), the practical
update is
\begin{equation}
\label{eq:gflownet_natgrad_update}
\theta^+
=
\theta-\eta\,
\bigl(\widehat F(\theta)+\lambda I\bigr)^{-1}
\widehat h(\theta),
\end{equation}
where \(\widehat F\) estimates or approximates the trajectory Fisher in
\eqref{eq:exact_temporal_fisher_main}. Backward-policy and flow parameters are updated
separately. A dense Fisher for \(p=\dim(\theta)\) requires \(O(p^2)\) storage and an
unstructured direct solve costs \(O(p^3)\). Neural implementations can instead use
automatic-differentiation Fisher--vector products and damped conjugate gradients;
K-FAC, low-rank, or layerwise blocks can serve as surrogates or preconditioners.

There are three computational routes to \(\widehat F\). In the \emph{exact} regime,
trajectory occupancies or marginals are available by dynamic programming, giving
\(\widehat F=F\); Corollary~\ref{cor:exact_tabular_fisher_main} is the tabular example.
In the \emph{sampled} regime, trajectories
\(\tau^{(1)},\dots,\tau^{(N)}\sim\mu\) estimate the decomposed Fisher by
\begin{equation}
\label{eq:sampled_fisher_training}
\widehat F_N(\theta)
=
\frac1N\sum_{i=1}^N
\sum_{t\ge0}
\omega_i\,
u_t(\theta;\tau^{(i)})u_t(\theta;\tau^{(i)})^\top,
\end{equation}
where \(\omega_i=1\) for on-policy and
\(\omega_i=q_\theta(\tau^{(i)})/\mu(\tau^{(i)})\) for importance-corrected off-policy
sampling.

In the optional \emph{structure-exploitable} regime, local target or construction
structure guides a surrogate when exact marginalisation is unavailable and naive
trajectory sampling is noisy. A canonical abstraction factorises the reward as  $R(x)=\prod_{\alpha\in\mathcal C}\psi_\alpha(x_\alpha)$, where \(\psi_\alpha\) are non-negative local factors on subsets \(x_\alpha\). Neither the Fisher decomposition, the approximation theorem, nor the sampled
regime assumes this factorisation, and it does not itself factorise the policy
Fisher. It can suggest policy-aligned blocks or enable tractable occupancy
summaries through dynamic programming or message passing; otherwise it defines a surrogate
whose quality is measured by \(\rho_\lambda\). This use of structure complements
target-side credit-assignment methods
\cite{bengio2023gflownetfoundations,pan2023better,jang2024learning,falet2024delta}.

All three regimes use the update \eqref{eq:gflownet_natgrad_update} and differ only in how the Fisher operator is accessed. Appendix
\ref{app:computational_regimes_detailed} gives consistency results,
problem-specific constructions, and pseudocode.

\section{Numerical Examples}
\label{sec:examples}

Our experiments test whether useful natural-gradient directions can be
obtained without forming the exact Fisher. We compare exact, sampled, and
structured approximations on controlled benchmarks and on Bayesian DAG
learning with the real Sachs protein-signalling data, evaluating convergence,
distribution matching, and mode discovery against Euclidean and exploration
baselines.

\subsection{Triangle}

The triangle-count ERGM serves as a pedagogical reference case: its reward
couples edges through higher-order interactions, yet exact occupancies remain
tractable. Let $n$ be the number of nodes,
$E=\{\{i,j\}:1\leq i<j\leq n\}$ the set of potential undirected edges, and
$G=(G_{ij})_{\{i,j\}\in E}\in\{0,1\}^{|E|}$ the binary edge-indicator vector.
The triangle count is $T(G)\coloneqq\sum_{1\le i<j<k\le n} G_{ij}G_{ik}G_{jk}$ and the corresponding target is
\(p_\beta(G)\propto\exp\{\beta T(G)\}\). Although the Fisher of this
one-parameter exponential family with respect to \(\beta\) is simply
\(\operatorname{Var}_{p_\beta}[T(G)]\), GFlowNet training instead depends on
the trajectory Fisher of the forward policy. Triangle interactions make this
Fisher intractable at moderate \(n\), so we use \(n=6\) and \(\beta=0.2\),
giving \(|E|=15\) edge decisions and \(2^{15}=32{,}768\) terminal graphs.
The GFlowNet processes the edges in a fixed order, deciding whether to include
(\(a=1\)) or exclude (\(a=0\)) each one, so every terminal graph has a unique
length-$15$ trajectory. The task therefore lies in the exact tabular regime of
Corollary~\ref{cor:exact_tabular_fisher_main}: the statewise Fisher is block
diagonal and its exact occupancies \(d_\theta(s)\) are computed by forward DP.
Fig.~\ref{fig:triangle_case} shows that this Fisher preconditioning
accelerates convergence, reducing both TB loss and KL faster than Euclidean
training.

\begin{figure}[ht!]
  \centering
  \includegraphics[width=0.6\linewidth]{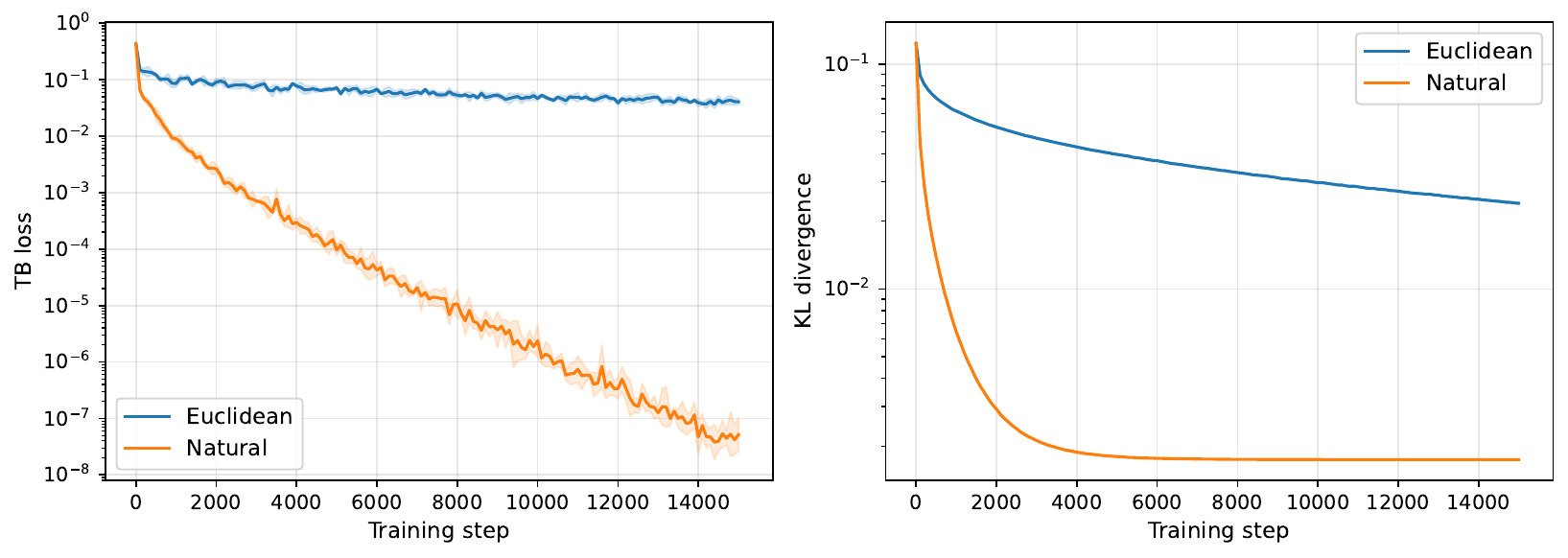}\caption{TB loss (left; geometric mean with $\pm1$ std in log space) and $\mathrm{KL}(q_\theta \| p_\beta)$ (right; mean $\pm$ std) on the triangle-count ERGM ($n=6$, $\beta=0.2$), over five seeds.}
\label{fig:triangle_case}
\end{figure}

\subsection{Hypergrid}

We follow the 2D hypergrid benchmark of \cite{malkin2022trajectory}.
The environment is a grid of side $H$ with states $(i,j) \in \{0,\dots,H-1\}^2$.
From any state the agent may move right $(i,j)\to(i+1,j)$, move up $(i,j)\to(i,j+1)$,
or stop and collect a terminal reward. Denoting the coordinates of state $s$ as $s^{(1)}=i$ and $s^{(2)}=j$, and setting $x_d = \tfrac{s^{(d)}}{(H-1)} \in [0,1]$, the reward is
\begin{equation}
    R(s) =R_0
    + \tfrac{1}{2} \prod_{d=1}^{2} 
\mathds{1}{\left\{ \tfrac{1}{4} < \abs{x_d - \tfrac{1}{2}}  < \tfrac{1}{2}\right\}} \; + \; 2  \prod_{d=1}^{2} \mathds{1}{ \left\{\tfrac{3}{10} < \abs{x_d -\tfrac{1}{2}} <  \tfrac{2}{5} \right\}} \;,
\end{equation}
with $R_0 = 10^{-3}$. The two indicator products create a low off-centre plateau and four narrow high-reward peaks, making $\pi\propto R$ multimodal.

Using a tabular forward policy and squared TB loss,
Figs.~\ref{fig:hypergrid}(a-b) show that the natural gradient reduces both TB
loss and $\mathrm{TV}(q_\theta,\pi)$ faster than Euclidean descent and ACE~\cite{dallantonia2026avoid},
while distributing mass more evenly across high-reward modes. Full settings and terminal distributions are given in Appendix~\ref{app:plots_hypergrid}.
\begin{figure}[ht!]
  \centering
  \includegraphics[width=1\linewidth]{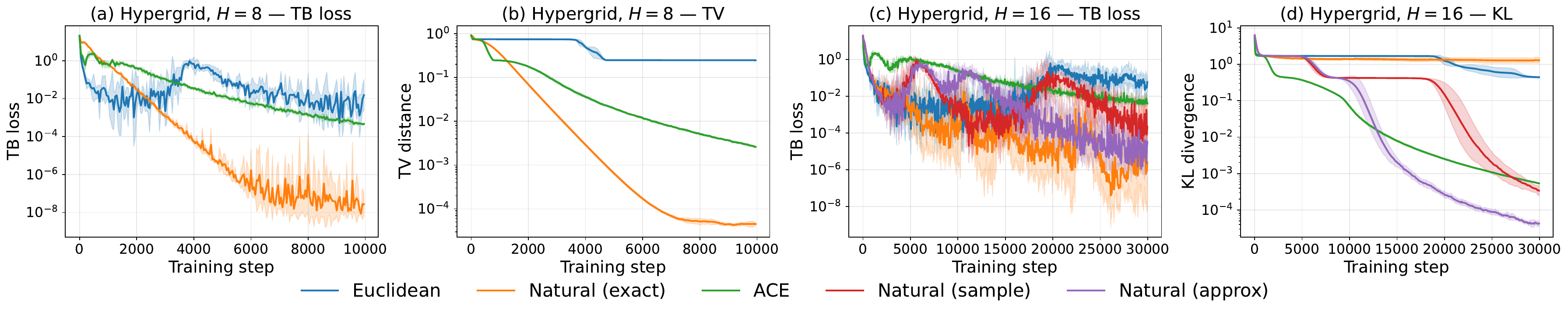}
\caption{2D hypergrid results, mean $\pm$ std over 5 seeds.
  \textbf{(a-b)} $H{=}8$: TB loss and $\mathrm{TV}(q_\theta, \pi)$, comparing Euclidean, Natural, and ACE.
  \textbf{(c-d)} $H{=}16$: TB loss and (geometric mean) $\mathrm{KL}(q_\theta\|\pi)$, comparing Euclidean, Natural (exact, sample, approx), and ACE.}
\label{fig:hypergrid}
\end{figure}

At $H{=}16$, the plateau and peaks form three reward levels,  making mass
calibration across levels, rather than mode location, the main challenge. Figs.~\ref{fig:hypergrid}(c-d) compare three Fisher-computation routes:  Natural~(exact), using forward-DP occupancies from~\eqref{eq:occupancy_main}; Natural~(sample), using batch occupancies from~\eqref{eq:sampled_fisher_training}; and Natural~(approx), using the factorised
marginal surrogate of Section~\ref{sec:natgrad_gflownet_training}. Fig.~\ref{fig:weighted_hyper_colormaps} in
Appendix~\ref{app:plots_hypergrid} shows that the Fisher approximations 
closely reproduce the target's three-level structure, whereas
the exact-Fisher variant does not fully recover it. ACE also matches the
target well, while Euclidean retains diffuse residual mass between modes.
The improved target matching of the approximate variants may reflect an implicit
regularising effect, but this requires
further investigation.

\subsection{Deceptive Hypergrid}
\label{sec:deceptive_grid_main}

To assess Fisher preconditioning beyond statewise tabular policies, we use the
two-dimensional deceptive grid of~\cite{kim2025adaptive} at
$H\in\{128,256\}$ with a shared neural forward policy. Its reward pairs an
easily reached moderate-reward region with hundreds or thousands of narrow
high-reward modes, making full mode coverage difficult under standard TB training.
Following AdaFisher's optimiser structure~\cite{gomes2025adafisher}, we
retain Adam's first-moment recursion but construct a GFlowNet-specific
preconditioner from the trajectory Fisher, in place of Adam's coordinate-wise
second-moment rescaling for the forward policy.
Algorithm~\ref{alg:fisher_adam_gflownet} gives the full TB update. We use
either a conjugate-gradient Fisher solve or a K-FAC approximation and compare
against Euclidean Adam and SGD (with and without $\varepsilon$-exploration) and
Adaptive Teacher. Appendix~\ref{app:deceptive_grid_world} gives the reward,
architecture, and full evaluation.

Fig.~\ref{fig:deceptive_visited} shows that both Fisher variants accelerate
mode discovery relative to Euclidean training and attain coverage
comparable to Adaptive Teacher. Unlike Adaptive Teacher, they use neither an
auxiliary exploration policy nor another explicit exploration mechanism: the
intervention is the geometry of the parameter update. The gain is most visible
at $H=256$, while the SGD and simple $\varepsilon$-exploration remain
far from full coverage.
\begin{figure}[ht!]
  \centering
  \includegraphics[width=\linewidth]{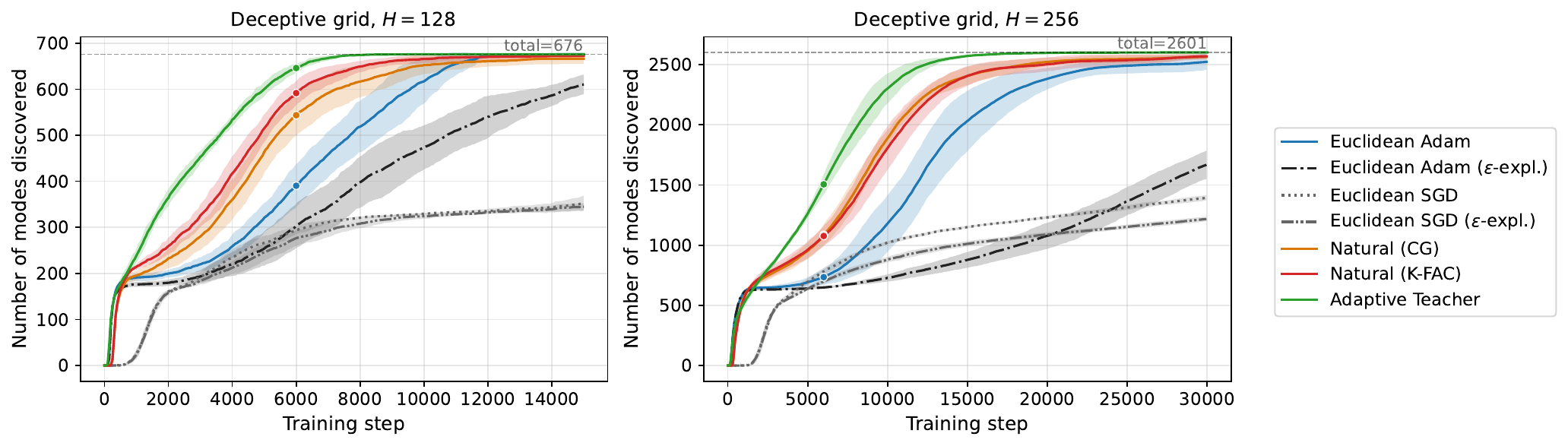}
  \caption{Cumulative modes discovered on the deceptive grid for $H=128$
  (left) and $H=256$ (right), mean $\pm$ one standard deviation over five
  seeds. Dots mark the $6{,}000$-step budget of~\cite{kim2025adaptive}, whose
  policy has width $128$; ours has width $32$.}
  \label{fig:deceptive_visited}
\end{figure}

\subsection{Lazy Random Walk Benchmarks}
\label{sec:lrw}

We additionally evaluate the exact, sampled, and factorised Fisher
approximations on the 8-Gaussians and Rings \emph{lazy-random-walk} benchmarks of~\cite{dallantonia2026avoid}.
On 8-Gaussians, all three natural-gradient variants reduce final TV from
\(0.190\) to approximately \(0.170\). On Rings, they recover \(24\)--\(25\)
of the \(36\) modes, compared with \(16\) for Euclidean TB, Wasserstein TB,
and ACE. Full details, learning curves, terminal distributions, and tables are given in Appendix~\ref{app:lrw_benchmarks}.

\subsection{Bayesian DAG structure learning}
\label{sec:sachs_dag}
\label{sec:dag_results_polished}

We evaluate several GFlowNet training methods for scalable amortised Bayesian
structure learning on the real Sachs protein-signalling
data~\cite{sachs2005causal}; Appendix~\ref{app:sachs_dag_details} provides experimental details and a controlled synthetic study. This also tests a
structure-exploitable regime under global acyclicity constraints.
Following~\cite{deleu2022bayesian}, DAG-building trajectories add
acyclicity-preserving edges and terminate with STOP, with target
reward \(R(G)=\exp\{N^{-1/2}\,\mathrm{BGe}(G\mid\mathcal D)-|E(G)|/2\}\), where \(G\) is a candidate DAG with directed-edge set \(E(G)\),
\(\mathcal D\) is the \(N\)-sample dataset, and \(\mathrm{BGe}(G\mid\mathcal D)\)
is the decomposable Bayesian Gaussian equivalent (BGe) score
\cite{geiger1994learning}; see Appendix~\ref{app:sachs_dag_details}, Eq.~\eqref{eq:bge_decomposition}. BGe decomposes over child-node parent sets, but acyclicity, STOP, and shared
policy parameters couple the trajectory Fisher across candidate edges.

A shared Linear Transformer maps directed-edge tokens for the current partial
DAG to a masked categorical distribution over valid edge additions and STOP.
Because its parameters are reused across edges and graph depths, the BGe
decomposition does not make its Fisher block diagonal. At sampled partial
graphs \(G\), let \(J_\theta(G)\) denote the logit Jacobian. Fisher-vector products use the exact masked-categorical
covariance \(C(\pi_\theta(\cdot\mid G))\) through
\(J_\theta(G)^\top C\!\left(\pi_\theta(\cdot\mid G)\right)J_\theta(G)\). Our
\emph{Damped Fisher} benchmark applies the resulting sampled non-diagonal
Fisher system using preconditioned conjugate gradients (PCG); the layer-block
curvature estimate serves only as a numerical preconditioner for the solve.
We compare Damped Fisher with Adam and KL-controlled Adam using paired \(500\)-step
continuations from common checkpoints over three seeds. The 17-edge Sachs
consensus network is used only for structural diagnostics.

Table~\ref{tab:sachs_fisher_main} shows that Damped Fisher lowers loss AUC
from \(0.751\) for Adam to \(0.734\), but is effectively tied with
KL-controlled Adam at \(0.734\). It also lowers P90-residual AUC from
\(0.881\) to \(0.870\). Directed AUPRC, E-SHD, and expected graph size show no reliable separation; Fig.~\ref{fig:sachs_posterior_graphs}
compares the Damped Fisher edge marginals with the Sachs consensus. Thus, on Sachs, Fisher-PCG improves local
optimisation relative to plain Adam but not to KL-controlled Adam, suggesting
that much of the gain comes from the shared step-norm and KL controls. The
experiment nevertheless demonstrates stable use of a sampled, non-diagonal
neural Fisher on a real-data DAG problem.
\begin{table}[ht!]
  \centering
  \small
  \setlength{\tabcolsep}{3.2pt}
  \caption{Sachs continuation results over three paired seeds (mean \(\pm\) SEM; best means in bold). AUC integrates relative fixed-probe quantities over updates (lower is better); P90 is the \(90\)th-percentile absolute TB residual. Structural metrics use the 17-edge consensus network only as a reference.}
  \label{tab:sachs_fisher_main}
  \begin{tabular}{lccccc}
    \toprule
    Method
    & Loss AUC \(\downarrow\)
    & P90 AUC \(\downarrow\)
    & Dir.\ AUPRC \(\uparrow\)
    & E-SHD \(\downarrow\)
    & \(\mathbb E|G|\) \\
    \midrule
    Adam
    & \(0.751{\pm}0.005\)
    & \(0.881{\pm}0.009\)
    & \(0.359{\pm}0.021\)
    & \(28.83{\pm}0.02\)
    & \(26.56{\pm}0.22\) \\
    KL-controlled Adam
    & \(\boldsymbol{0.734\pm0.006}\)
    & \(0.876{\pm}0.009\)
    & \(\boldsymbol{0.361\pm0.018}\)
    & \(28.98{\pm}0.11\)
    & \(26.66{\pm}0.14\) \\
    Damped Fisher
    & \(\boldsymbol{0.734\pm0.004}\)
    &  \(\boldsymbol{0.870\pm0.007}\)
    & \(0.349{\pm}0.047\)
    & \(\boldsymbol{28.80\pm0.19}\)
    & \(26.28{\pm}0.45\) \\
    \bottomrule
  \end{tabular}
\end{table}

\begin{figure}[ht!]
  \centering
  \includegraphics[width=0.7\textwidth]{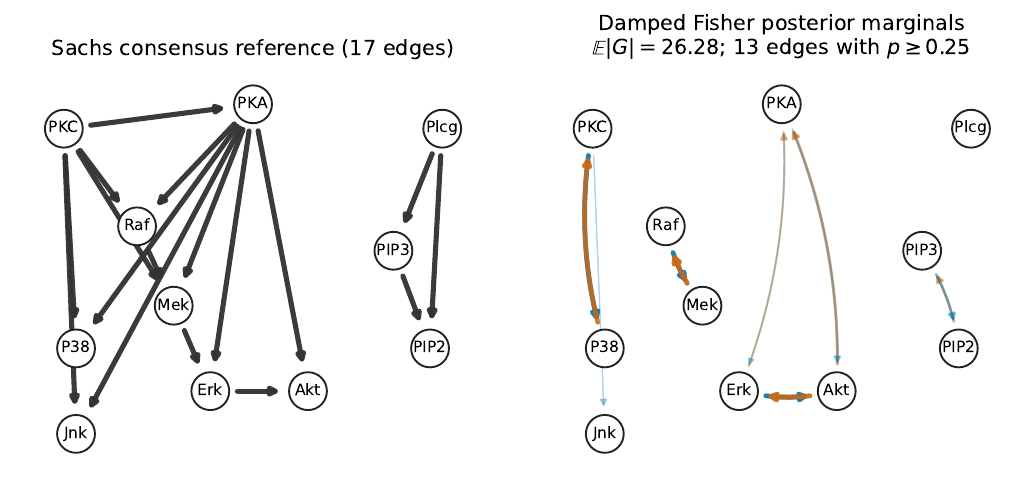}
  \caption{Sachs consensus reference and Damped Fisher posterior edge marginals averaged over three seeds. For clarity, only marginals  \(\ge 0.25\) are shown; width and opacity encode unthresholded values, and  blue/orange denote consensus agreement/disagreement. The thresholded panel is neither a MAP graph nor should be read as definitive causal recovery from observational data.}
\label{fig:sachs_posterior_graphs}
\end{figure}

\section{Discussion}
\label{sec:discussion}

We introduced a Fisher-Rao formulation of GFlowNet training by treating the forward
policy as a policy-induced statistical sampler. This gives a natural-gradient update
defined by the trajectory Fisher, which decomposes into per-step score second moments.
The theory also gives exact tabular formulas and perturbation bounds showing when
approximate Fisher surrogates preserve the natural direction.

Empirically, exact natural gradients provide a useful reference, while sampled and
structured approximations retain much of the gain over Euclidean TB training. Across the
benchmarks, Fisher-Rao updates often improve convergence, stability, and high-reward mode
discovery, especially when target locality or factorisation supports informative
structured approximations. More broadly, the results suggest that target structure can be
used not only for reward decomposition or credit assignment, but also to construct the
optimisation geometry of the sampler itself.

A key limitation is reliance on stable Fisher inversion and informative
structural approximations; weakly factorised, nonlocal, or high-dimensional
settings may require additional damping, low-rank approximations, or scalable
Fisher-vector solvers.

\newpage
\printbibliography[heading=largebibliography,title={References}]

\newpage
\appendix
\section*{\centerline{\huge{-- Appendices --}}\vspace*{0.5cm}}
\label{sec:app}
\section{Natural-gradient properties in smooth statistical models}
\label{app:prelim_natgrad_properties}

This appendix records the standard geometric facts about natural gradients used in the
main text. Throughout, let
\[
\mathcal M=\{q_\theta:\theta\in\Theta\}
\]
be a differentiable statistical model, with \(\Theta\subseteq\mathbb R^d\) open, and let
\(q_\theta(z)\) denote a positive density or probability mass function with respect to
a fixed dominating measure \(\nu\), with support independent of \(\theta\). Define
\[
F(\theta)
=
\mathbb E_{z\sim q_\theta}
\big[
\nabla_\theta \log q_\theta(z)\,
\nabla_\theta \log q_\theta(z)^\top
\big]
\]
and
\[
\mathcal U_\theta\coloneqq\operatorname{range}F(\theta).
\]
We assume the usual regularity conditions allowing
differentiation under the integral sign and interchange of differentiation and
expectation. The subspace \(\mathcal U_\theta\) is the identifiable score tangent space:
\(F(\theta)\) is positive definite on \(\mathcal U_\theta\), and
\(F(\theta)^\dagger\) acts as its inverse there and as zero on
\(\ker F(\theta)\).

\subsection{KL quadratic expansion}

The Fisher information is the local quadratic form of the KL divergence. For
\(\Delta\theta\in\mathbb R^d\) small,
\begin{align}
\KL(q_\theta\|q_{\theta+\Delta\theta})
&=
\mathbb E_{z\sim q_\theta}
\left[
\log q_\theta(z)-\log q_{\theta+\Delta\theta}(z)
\right].
\end{align}
A Taylor expansion of \(\log q_{\theta+\Delta\theta}(z)\) around \(\theta\) gives
\[
\log q_{\theta+\Delta\theta}(z)
=
\log q_\theta(z)
+
\Delta\theta^\top \nabla_\theta \log q_\theta(z)
+
\frac12
\Delta\theta^\top \nabla_\theta^2 \log q_\theta(z)\,\Delta\theta
+
o(\|\Delta\theta\|^2).
\]
Substituting and taking expectation,
\begin{align}
\KL(q_\theta\|q_{\theta+\Delta\theta})
&=
-\Delta\theta^\top
\mathbb E_{q_\theta}\!\left[\nabla_\theta \log q_\theta(z)\right]
-
\frac12
\Delta\theta^\top
\mathbb E_{q_\theta}\!\left[\nabla_\theta^2 \log q_\theta(z)\right]
\Delta\theta
+
o(\|\Delta\theta\|^2).
\end{align}
Since
\[
\mathbb E_{q_\theta}\!\left[\nabla_\theta \log q_\theta(z)\right]
=
\nabla_\theta \int q_\theta(z)\,\mathrm d\nu(z)
=
0,
\]
and, under standard regularity conditions,
\[
-\mathbb E_{q_\theta}\!\left[\nabla_\theta^2 \log q_\theta(z)\right]
=
F(\theta),
\]
we obtain
\begin{equation}
\label{eq:appendix_kl_quadratic}
\KL(q_\theta\|q_{\theta+\Delta\theta})
=
\frac12\,\Delta\theta^\top F(\theta)\Delta\theta
+
o(\|\Delta\theta\|^2).
\end{equation}

\subsection{Variational characterisation of the natural gradient}

Let \(\mathcal J:\Theta\to\mathbb R\) be differentiable and suppose its gradient is
compatible with the statistical model,
\(\nabla_\theta\mathcal J(\theta)\in\mathcal U_\theta\). If it is not, the
statements below apply to its orthogonal projection onto \(\mathcal U_\theta\).
The natural gradient is
\[
\widetilde\nabla_\theta \mathcal J(\theta)
=
F(\theta)^\dagger \nabla_\theta \mathcal J(\theta).
\]
It is characterised as the unique maximiser over the identifiable tangent space of
\begin{equation}
\label{eq:appendix_natgrad_variational_ascent}
\max_{v\in\mathcal U_\theta}
\left\{
\langle \nabla_\theta \mathcal J(\theta),v\rangle
-
\frac12\,v^\top F(\theta)v
\right\}.
\end{equation}
Indeed, the first-order condition on \(\mathcal U_\theta\) is
\[
\nabla_\theta \mathcal J(\theta)-F(\theta)v=0,
\]
whose unique solution in \(\mathcal U_\theta\) is
\[
v=F(\theta)^\dagger\nabla_\theta \mathcal J(\theta).
\]
Thus \(\widetilde\nabla_\theta \mathcal J(\theta)\) is the steepest-ascent direction in
the Fisher geometry, while \(-\widetilde\nabla_\theta \mathcal J(\theta)\) is the
corresponding steepest-descent direction.

Equivalently, using \eqref{eq:appendix_kl_quadratic}, the solution of the local
trust-region problem
\[
\arg\min_{v\in\mathcal U_\theta}
\left\{
\langle \nabla_\theta \mathcal J(\theta),v\rangle
:\;
\frac12 v^\top F(\theta)v\le \varepsilon
\right\},
\]
When the identifiable gradient is nonzero, the solution for any
\(\varepsilon>0\) lies on the ray generated by
\(-F(\theta)^\dagger\nabla_\theta \mathcal J(\theta)\).

\subsection{Reparameterisation invariance}

Let \(\vartheta=\phi(\theta)\) be a smooth invertible reparameterisation, with Jacobian
\(J_\phi(\theta)\). First consider locally identifiable coordinates, so the Fisher
matrices below are nonsingular. Writing \(q_\vartheta=q_{\theta(\vartheta)}\), the
score transforms as
\[
\nabla_\vartheta \log q_\vartheta(z)
=
J_{\theta}(\vartheta)^\top \nabla_\theta \log q_\theta(z),
\]
where \(J_\theta(\vartheta)=\partial \theta/\partial \vartheta\). Hence the Fisher matrix
in \(\vartheta\)-coordinates is
\begin{equation}
\label{eq:appendix_fisher_reparam}
F_\vartheta(\vartheta)
=
J_{\theta}(\vartheta)^\top
F_\theta(\theta)
J_{\theta}(\vartheta).
\end{equation}
Similarly, the Euclidean gradient transforms as
\[
\nabla_\vartheta \mathcal J
=
J_{\theta}(\vartheta)^\top \nabla_\theta \mathcal J.
\]
Therefore
\begin{align}
\widetilde\nabla_\vartheta \mathcal J
&=
F_\vartheta(\vartheta)^{-1}\nabla_\vartheta \mathcal J
\nonumber\\
&=
\big(J_{\theta}^\top F_\theta J_{\theta}\big)^{-1}
J_{\theta}^\top \nabla_\theta \mathcal J
\nonumber\\
&=
J_\phi(\theta)\,F_\theta(\theta)^{-1}\nabla_\theta \mathcal J
\nonumber\\
&=
J_\phi(\theta)\,\widetilde\nabla_\theta \mathcal J.
\end{align}
Thus the natural-gradient vector field transforms intrinsically under smooth changes of
coordinates, whereas the Euclidean gradient depends on the chosen parameterisation.
For a redundant parameterisation, the same statement holds for the induced tangent
vector on the identifiable quotient space. The Moore--Penrose inverse selects a
representative in \(\mathcal U_\theta\), but that Euclidean minimum-norm representative
need not itself transform covariantly under arbitrary changes of redundant coordinates.

\subsection{Local descent property}

For minimising \(\mathcal J\), the natural-gradient update with step size \(\eta>0\) is
\[
\theta^+
=
\theta-\eta\,\widetilde\nabla_\theta \mathcal J(\theta).
\]
Assume that \(\mathcal J\) is locally smooth relative to the Fisher metric in the sense
that, for all sufficiently small \(v\in\mathcal U_\theta\),
\begin{equation}
\label{eq:appendix_fisher_smooth}
\mathcal J(\theta+v)
\le
\mathcal J(\theta)
+
\langle \nabla_\theta \mathcal J(\theta),v\rangle
+
\frac{L}{2}v^\top F(\theta)v.
\end{equation}
Substituting \(v=-\eta \widetilde\nabla_\theta \mathcal J(\theta)\) gives
\begin{align}
\mathcal J(\theta^+)
&\le
\mathcal J(\theta)
-
\eta\left(1-\frac{L\eta}{2}\right)
\|\nabla_\theta \mathcal J(\theta)\|_{F(\theta)^\dagger}^2,
\end{align}
where
\(\|a\|_{F^\dagger}^2\coloneqq a^\top F^\dagger a\).
Hence any step size \(0<\eta<2/L\) yields descent whenever
\(\|\nabla_\theta\mathcal J(\theta)\|_{F(\theta)^\dagger}>0\). Under the compatibility
assumption above, this is equivalent to
\(\nabla_\theta\mathcal J(\theta)\neq0\).

\section{Fisher-Rao geometry for discrete distributions}
\label{app:prelim_discrete_geometry}

This appendix expands the finite-support Fisher-Rao facts used in
Section~\ref{sec:preliminaries}. The key point is that discrete distributions carry the
same intrinsic information geometry as continuous statistical models, but the
computability of the corresponding Fisher matrix depends on whether its required
moments are tractable, often through a low-treewidth factorisation.

\subsection{Simplex geometry and categorical Fisher blocks}

Let
\[
\Delta^{m-1}_{++}
:=
\left\{
p=(p_1,\dots,p_m)\in\mathbb R^m_{>0}:
\sum_{i=1}^m p_i=1
\right\}
\]
denote the simplex interior. The Fisher-Rao metric on the simplex interior is
\[
\langle \xi,\zeta\rangle_p^{\mathrm{FR}}
=
\sum_{i=1}^m \frac{\xi_i\zeta_i}{p_i},
\qquad
\xi,\zeta\in T_p\Delta^{m-1}_{++},
\qquad
\sum_i \xi_i=\sum_i\zeta_i=0.
\]
This metric admits a simple geometric representation through the square-root embedding
\begin{equation}
\label{eq:appendix_sqrt_map}
\Psi:\Delta^{m-1}_{++}\to \mathbb S^{m-1}_+(2),
\qquad
\Psi(p)_i = 2\sqrt{p_i},
\end{equation}
where
\[
\mathbb S^{m-1}_+(2)
:=
\left\{
u\in\mathbb R^m_{>0}:
\sum_{i=1}^m u_i^2 = 4
\right\}
\]
is the positive orthant of the sphere of radius \(2\). For a tangent vector
\(\xi\in T_p\Delta^{m-1}_{++}\), the differential is
\[
(D\Psi_p\xi)_i=\frac{\xi_i}{\sqrt{p_i}}.
\]
Therefore the Euclidean metric on the sphere pulls back as
\[
\langle D\Psi_p\xi,D\Psi_p\zeta\rangle_{\mathbb R^m}
=
\sum_{i=1}^m \frac{\xi_i\zeta_i}{p_i}
=
\langle \xi,\zeta\rangle_p^{\mathrm{FR}}.
\]
Thus \(\Psi\) is an isometry from the simplex interior onto the positive orthant of the
sphere.

The requirement \(p_i>0\) is essential for this smooth Riemannian description: both
the metric and the differential of the square-root map contain division by \(p_i\) or
\(\sqrt{p_i}\). The square-root map and its induced distance extend continuously to the
closed simplex, but the same smooth
metric chart does not cover support-changing boundary points. On a fixed-support face, the construction applies to its strictly positive
coordinates. This is a geometry of probability distributions, not a continuous
relaxation of the underlying discrete objects.

The local building block used repeatedly in the paper is the categorical distribution.
Let \(p\in\Delta^{m-1}_{++}\) be parameterised by logits
\[
\ell=(\ell_1,\dots,\ell_m)\in\mathbb R^m,
\qquad
p_i=\frac{\exp(\ell_i)}{\sum_{j=1}^m\exp(\ell_j)}.
\]
For \(X\sim\mathrm{Cat}(p)\),
\[
\log p(X)=\ell_X-\log\sum_{j=1}^m \exp(\ell_j),
\qquad
\nabla_\ell\log p(X)=\mathrm e_X-p.
\]
Hence the Fisher information in logit coordinates is
\begin{align}
F_{\mathrm{cat}}(p)
&=
\mathbb E\big[(\mathrm e_X-p)(\mathrm e_X-p)^\top\big]
\nonumber\\
&=
\sum_{i=1}^m p_i\,\mathrm e_i\mathrm e_i^\top - pp^\top
\nonumber\\
&=
\Diag(p)-pp^\top.
\label{eq:appendix_cat_fisher}
\end{align}
This matrix is positive semidefinite and satisfies
\[
(\Diag(p)-pp^\top)\mathbf1=0,
\]
reflecting the invariance of softmax probabilities to adding a constant to all logits.
Because \(p_i>0\) for every \(i\), its rank is \(m-1\) and its identifiable logit
subspace is \(\mathbf1^\perp\). Consequently, the exact identifiable natural gradient
uses \(F_{\mathrm{cat}}(p)^\dagger\), equivalently the inverse restricted to
\(\mathbf1^\perp\), while a damped direction uses
\((F_{\mathrm{cat}}(p)+\lambda I)^{-1}\).

\subsection{When the discrete Fisher is computable}

For a finite support \(\mathcal X\), the Fisher matrix of a differentiable family
\(\{p_\theta:\theta\in\Theta\}\) is
\[
F(\theta)
=
\sum_{x\in\mathcal X}
p_\theta(x)\,
\nabla_\theta\log p_\theta(x)\,
\nabla_\theta\log p_\theta(x)^\top.
\]
This expression is conceptually simple, but evaluating it requires expectations under
\(p_\theta\). Exact Fisher computation is therefore tractable only when the required
moments can be evaluated exactly. Sampling or deterministic approximations can instead
provide an approximate Fisher operator.

A first tractable case is a discrete exponential family
\[
p_\theta(x)
=
\exp\!\big(\theta^\top T(x)-A(\theta)\big),
\qquad
A(\theta)=\log \sum_x \exp(\theta^\top T(x)).
\]
Then
\[
\nabla_\theta\log p_\theta(x)
=
T(x)-\nabla_\theta A(\theta),
\]
and hence
\begin{equation}
\label{eq:appendix_expfam_fisher}
F(\theta)
=
\operatorname{Cov}_{p_\theta}[T(X)].
\end{equation}
Thus, Fisher computation reduces to the computation of first and second moments of the
sufficient statistic. It is tractable whenever those moments are tractable.

As a concrete graph example, consider an independent-edge logistic model. Let
\(G=(G_e)_{e\in E}\in\{0,1\}^{|E|}\), and suppose
\[
p_\theta(G)
=
\prod_{e\in E}
\mathrm{Bern}\!\left(G_e;s_e(\theta)\right),
\qquad
s_e(\theta):=\frac{1}{1+\exp(-x_e^\top\theta)}.
\]
Then
\[
\nabla_\theta\log p_\theta(G)
=
\sum_{e\in E}
\bigl(G_e-s_e(\theta)\bigr)x_e.
\]
Since the edges are independent,
\begin{equation}
\label{eq:appendix_logistic_graph_fisher}
F(\theta)
=
\sum_{e\in E}
s_e(\theta)\bigl(1-s_e(\theta)\bigr)
x_e x_e^\top.
\end{equation}
This is the graph analogue of the Fisher matrix in logistic regression. It is tractable
because the distribution factorises over edges.

More generally, suppose
\[
p_\theta(x)
=
\frac{1}{Z(\theta)}
\prod_{\alpha\in\mathcal C}\psi_\alpha(x_\alpha;\theta)
\]
factorises according to a factor graph of small treewidth, with positive differentiable
local factors \(\psi_\alpha\). Writing
\(S_\theta(x)\coloneqq
\sum_{\alpha\in\mathcal C}\nabla_\theta\log\psi_\alpha(x_\alpha;\theta)\),
the model score is
\(\nabla_\theta\log p_\theta(x)
=S_\theta(x)-\mathbb E_{p_\theta}[S_\theta(X)]\), and therefore
\(F(\theta)=\operatorname{Cov}_{p_\theta}[S_\theta(X)]\). Junction-tree or exact message-passing algorithms can compute the marginal and
cross-moment information in this covariance. Thus, even when the state space is
combinatorial, the Fisher is exactly computable when the graphical structure supports exact inference. This controls the cost of summing over configurations, but does not remove the \(O(\dim(\theta)^2)\) storage needed to form a dense Fisher; parameter
dimension and graphical inference complexity remain separate considerations.

\subsection{When the discrete Fisher is intractable}

The preceding examples show that discrete Fisher geometry is operational when the model
factorises strongly enough or when sufficient-statistic moments are tractable. In
globally coupled combinatorial models, the geometry still exists, but computing it
inherits the difficulty of expectation under the model.

For example, a triangle-biased graph model has the form
\[
p_\beta(G)\propto \exp\!\big(\beta\,T(G)\big),
\qquad
T(G)=\sum_{i<j<k}G_{ij}G_{ik}G_{jk}.
\]
As a one-parameter exponential family,
\[
F(\beta)=\operatorname{Var}_{p_\beta}[T(G)].
\]
However, evaluating this variance requires expectations under a distribution with
higher-order dependencies among edges. For a graph on \(n\) vertices, direct exact
evaluation sums over \(2^{\binom n2}\) edge configurations. Without additional
exploitable structure, this rapidly becomes computationally prohibitive, and Fisher
estimation requires Monte Carlo or another approximation scheme.

The same issue appears in high-treewidth combinatorial Markov random fields and models
with globally coupled sufficient statistics. Local factors do not by themselves make
their induced moments tractable: exact Fisher computation still requires expectations
under the globally normalised distribution. Thus the Fisher--Rao metric is
mathematically well defined on the fixed-support model. Still, its direct use depends on
whether the relevant expectations can be evaluated, sampled, or approximated from
structure.

For a policy-induced sampler, the categorical covariance \(C(\pi_\theta(\cdot\mid s))\) is explicit in action-logit coordinates. Its pull-back to shared policy parameters is \(J_\theta(s)^\top C(\pi_\theta(\cdot\mid s))J_\theta(s)\), where \(J_\theta(s)\) is the Jacobian of the action logits with respect to \(\theta\). This
matrix is generally dense. The remaining global expectation weights these statewise contributions by occupancies
under the trajectory law. Exact dynamic programming, Monte Carlo trajectory sampling,
exact low-treewidth message passing, approximate belief propagation, and truncated local approximations are different ways of accessing that expectation.

\section{Exact Fisher geometry for policy-induced samplers}
\label{app:policy_induced_fisher}

This appendix proves the trajectory-level Fisher identities used in
Section~\ref{sec:fisher_rao_structure}. Throughout, let
\[
\tau=(s_0,a_0,s_1,a_1,\dots,s_T),
\qquad
q_\theta(\tau)=\prod_{t=0}^{T-1}\pi_\theta(a_t\mid s_t),
\]
and define the per-step score increment
\[
u_t(\theta;\tau):=\nabla_\theta \log \pi_\theta(a_t\mid s_t).
\]
For notational convenience, we extend trajectories beyond termination by setting
\[
u_t(\theta;\tau)=0
\qquad\text{for all }t\ge T(\tau),
\]
so that all sums may be written over \(t\ge0\). We assume either a bounded construction
horizon or \(\sum_{t\ge0}\mathbb E\|u_t\|^2<\infty\); in the latter case the partial
score sums converge in \(L^2\). We also use the pre-action filtration
\[
\mathcal F_t
:=
\sigma(s_0,a_0,s_1,a_1,\dots,s_{t-1},a_{t-1},s_t),
\]
where \(\sigma(\cdot)\) denotes the generated sigma-algebra. Thus \(\mathcal F_t\)
contains the trajectory history up to state \(s_t\), but not the action \(a_t\).

\subsection{Terminal and trajectory Fisher geometries}

Let \(x=\Phi(\tau)\) denote the terminal object induced by a trajectory \(\tau\). The
trajectory law \(q_\theta(\tau)\) induces the terminal law
\[
q_\theta(x)
=
\sum_{\tau:\Phi(\tau)=x}q_\theta(\tau).
\]
The corresponding Fisher matrices are
\[
F_{\mathrm{traj}}(\theta)
=
\mathbb E_{\tau\sim q_\theta}
\left[
\nabla_\theta\log q_\theta(\tau)
\nabla_\theta\log q_\theta(\tau)^\top
\right],
\]
and
\[
F_{\mathrm{term}}(\theta)
=
\mathbb E_{x\sim q_\theta}
\left[
\nabla_\theta\log q_\theta(x)
\nabla_\theta\log q_\theta(x)^\top
\right].
\]

\begin{prop}[Terminal versus trajectory Fisher geometry]
\label{prop:terminal_vs_trajectory_fisher}
Define the trajectory score
\[
S_\theta(\tau):=\nabla_\theta\log q_\theta(\tau).
\]
Then the following hold on the support where \(q_\theta(x)>0\).
\begin{enumerate}
\item[(i)] If the construction is canonical, so that each terminal object \(x\) has a
unique trajectory \(\tau(x)\), then
\[
\nabla_\theta\log q_\theta(x)=S_\theta(\tau(x)),
\qquad
F_{\mathrm{term}}(\theta)=F_{\mathrm{traj}}(\theta).
\]

\item[(ii)] In general,
\[
\nabla_\theta\log q_\theta(x)
=
\mathbb E_{q_\theta(\tau\mid x)}[S_\theta(\tau)],
\]
and
\[
F_{\mathrm{traj}}(\theta)
=
F_{\mathrm{term}}(\theta)
+
\mathbb E_{x\sim q_\theta}
\left[
\operatorname{Cov}_{q_\theta(\tau\mid x)}
\big(S_\theta(\tau)\big)
\right].
\]
\end{enumerate}
\end{prop}

\paragraph{Proof.}
For any terminal object \(x\) with \(q_\theta(x)>0\),
\[
\nabla_\theta q_\theta(x)
=
\nabla_\theta
\sum_{\tau:\Phi(\tau)=x}q_\theta(\tau)
=
\sum_{\tau:\Phi(\tau)=x}
q_\theta(\tau)\nabla_\theta\log q_\theta(\tau).
\]
Dividing by \(q_\theta(x)\) gives
\[
\nabla_\theta\log q_\theta(x)
=
\sum_{\tau:\Phi(\tau)=x}
q_\theta(\tau\mid x)S_\theta(\tau)
=
\mathbb E_{q_\theta(\tau\mid x)}[S_\theta(\tau)].
\]
If the construction is canonical, \(q_\theta(\tau\mid x)\) is a point mass at
\(\tau(x)\), giving part (i).

For part (ii), decompose the conditional second moment:
\begin{align}
F_{\mathrm{traj}}(\theta)
&=
\mathbb E_{x\sim q_\theta}
\left[
\mathbb E_{q_\theta(\tau\mid x)}
\big[
S_\theta(\tau)S_\theta(\tau)^\top
\big]
\right]
\nonumber\\
&=
\mathbb E_{x\sim q_\theta}
\left[
\mathbb E[S_\theta(\tau)\mid x]
\mathbb E[S_\theta(\tau)\mid x]^\top
\right]
+
\mathbb E_{x\sim q_\theta}
\left[
\operatorname{Cov}(S_\theta(\tau)\mid x)
\right].
\end{align}
The first term is \(F_{\mathrm{term}}(\theta)\) by the identity above. This proves the
decomposition. \(\square\)

\subsection{Proof of Theorem~\ref{thm:riemannian_policy_induced}}

\paragraph{Conditional centering of the per-step score.}
Conditioning on \(\mathcal F_t\) fixes the current state \(s_t\), while
\(a_t\sim\pi_\theta(\cdot\mid s_t)\). Therefore
\begin{align}
\mathbb E\!\left[u_t(\theta;\tau)\mid \mathcal F_t\right]
&=
\sum_{a\in\mathcal A(s_t)}
\pi_\theta(a\mid s_t)
\nabla_\theta\log \pi_\theta(a\mid s_t)
\nonumber\\
&=
\sum_{a\in\mathcal A(s_t)}
\nabla_\theta \pi_\theta(a\mid s_t)
\nonumber\\
&=
\nabla_\theta
\sum_{a\in\mathcal A(s_t)}\pi_\theta(a\mid s_t)
=
0.
\end{align}
Thus \(u_t\) is conditionally centred given the pre-action history. Equivalently,
\(u_t\) is \(\mathcal F_{t+1}\)-measurable and satisfies
\[
\mathbb E[u_t\mid\mathcal F_t]=0.
\]

\paragraph{Exact temporal decomposition.}
Since
\[
\nabla_\theta\log q_\theta(\tau)
=
\sum_{t\ge0}u_t(\theta;\tau),
\]
we have
\[
F(\theta)
=
\mathbb E_{\tau\sim q_\theta}
\left[
\left(\sum_{t\ge0}u_t\right)
\left(\sum_{t'\ge0}u_{t'}\right)^\top
\right].
\]
For \(t<t'\), the variable \(u_t\) is \(\mathcal F_{t'}\)-measurable. Hence, by the
tower property and the conditional-centering identity,
\[
\mathbb E[u_tu_{t'}^\top]
=
\mathbb E
\left[
u_t\,\mathbb E[u_{t'}^\top\mid\mathcal F_{t'}]
\right]
=
0.
\]
The case \(t'>t\) follows by transposition. Therefore all temporal cross terms vanish
and
\[
F(\theta)
=
\sum_{t\ge0}
\mathbb E_{\tau\sim q_\theta}
\big[
u_t(\theta;\tau)u_t(\theta;\tau)^\top
\big],
\]
which is \eqref{eq:exact_temporal_fisher_main}. \(\square\)

\subsection{Tabular and shared-parameter specialisations}
\label{app:tabular_shared_specializations}

We prove Corollary~\ref{cor:exact_tabular_fisher_main} and derive the shared-parameter identity
in \eqref{eq:shared_exact_fisher_main}. Assume first that the forward policy is tabular softmax, so each state
\(s\in\mathcal S\) has its own logit vector
\(\theta_s\in\mathbb R^{|\mathcal A(s)|}\), with
\[
\pi_\theta(a\mid s)
=
\frac{\exp(\theta_{s,a})}{\sum_{a'\in\mathcal A(s)}\exp(\theta_{s,a'})}.
\]
Let \(\pi_s:=\pi_\theta(\cdot\mid s)\). For \(a\in\mathcal A(s)\),
\[
\nabla_{\theta_s}\log\pi_\theta(a\mid s)=e_a-\pi_s,
\]
where \(e_a\) is the one-hot vector indexed by action \(a\). Therefore the local
categorical Fisher block is
\begin{align}
\mathbb E_{a\sim\pi_s}
\big[
(e_a-\pi_s)(e_a-\pi_s)^\top
\big]
&=
\Diag(\pi_s)-\pi_s\pi_s^\top
=:C(\pi_s).
\end{align}
Since the tabular score increment at time \(t\) is supported only on the coordinate
block associated with the visited state \(s_t\), the exact temporal Fisher decomposition
groups by state:
\begin{align}
F(\theta)
&=
\sum_{t\ge0}
\mathbb E[u_tu_t^\top]
\nonumber\\
&=
\bigoplus_{s\in\mathcal S}
\left(
\mathbb E_{\tau\sim q_\theta}
\left[
\sum_{t\ge0}\mathds{1}\{s_t=s\}
\right]
\right)
C(\pi_s)
\nonumber\\
&=
\bigoplus_{s\in\mathcal S}
d_\theta(s)C(\pi_s),
\end{align}
where
\[
d_\theta(s)
=
\mathbb E_{\tau\sim q_\theta}
\left[
\sum_{t\ge0} \mathds{1}\{s_t=s\}
\right]
\]
is the occupancy measure. This proves the Fisher identity in
Corollary~\ref{cor:exact_tabular_fisher_main}. Applying the
natural-gradient formula blockwise gives
\[
\widetilde\nabla_{\theta_s}\mathcal J(\theta)
=
\big(d_\theta(s)C(\pi_s)\big)^\dagger
\nabla_{\theta_s}\mathcal J(\theta),
\]
which gives its statewise natural-gradient identity.
Now suppose logits are produced by a shared map
\[
\ell_\theta(s)\in\mathbb R^{|\mathcal A(s)|},
\qquad
\pi_\theta(\cdot\mid s)=\softmax(\ell_\theta(s)).
\]
Let \(J_\theta(s)\) denote the Jacobian of \(\ell_\theta(s)\) with respect to the global
parameter vector. By the chain rule,
\[
u_t(\theta;\tau)
=
J_\theta(s_t)^\top
\nabla_\ell\log\pi_\theta(a_t\mid s_t).
\]
Since
\[
\nabla_\ell\log\pi_\theta(a_t\mid s_t)
=
e_{a_t}-\pi_\theta(\cdot\mid s_t),
\]
taking the conditional second moment over
\(a_t\sim\pi_\theta(\cdot\mid s_t)\) yields
\[
\mathbb E[u_tu_t^\top\mid s_t]
=
J_\theta(s_t)^\top
C(\pi_\theta(\cdot\mid s_t))
J_\theta(s_t).
\]
Substituting this into the temporal Fisher decomposition gives
\begin{equation}
\label{eq:appendix_shared_exact_fisher}
F(\theta)
=
\mathbb E_{\tau\sim q_\theta}
\left[
\sum_{t\ge0}
J_\theta(s_t)^\top
C(\pi_\theta(\cdot\mid s_t))
J_\theta(s_t)
\right].
\end{equation}
This is the shared-parameter identity used in
\eqref{eq:shared_exact_fisher_main}. \(\square\)

\subsection{Conditional factorisation and exact Fisher blocks}
\label{app:conditional_fisher_factorization}

Exact block structure can arise beyond state-tabular policies when the
conditional factorisation of the policy aligns with disjoint parameter blocks.
Suppose that at a state \(s\), an action is generated hierarchically as
\[
\pi_\theta(c,b\mid s)
=
\pi_\theta^{(1)}(c\mid s)\,
\pi_\theta^{(2)}(b\mid c,s).
\]
Writing
\(u^{(1)}=\nabla_\theta\log\pi_\theta^{(1)}(c\mid s)\) and
\(u^{(2)}=\nabla_\theta\log\pi_\theta^{(2)}(b\mid c,s)\), conditional score
centering gives \(\mathbb E[u^{(2)}\mid c,s]=0\). Since \(u^{(1)}\) is fixed
after conditioning on \((c,s)\),
\[
\mathbb E[u^{(1)}u^{(2)\top}\mid s]
=
\mathbb E_{c\mid s}\!\left[
u^{(1)}\mathbb E[u^{(2)\top}\mid c,s]
\right]
=0,
\]
and the transposed cross moment vanishes in the same way. The local Fisher
contribution therefore decomposes exactly as
\[
\mathbb E[
  (u^{(1)}+u^{(2)})(u^{(1)}+u^{(2)})^\top
  \mid s]
=
\mathbb E[u^{(1)}u^{(1)\top}\mid s]
+
\mathbb E_{c\mid s}\!\left[
  \mathbb E[u^{(2)}u^{(2)\top}\mid c,s]
\right].
\]
If the two factors use disjoint parameter blocks, these terms occupy disjoint
blocks and the Fisher is block diagonal. The argument applies recursively to
deeper conditional factorisations.

Bayesian DAG learning illustrates why target factorisation alone is
insufficient. A decomposable score such as BGe makes child-node parent sets
natural target units. If a fixed node order makes acyclicity automatic and the
policy chooses one parent set per child using disjoint parameters, the
conditional factorisation above yields exact child-node Fisher blocks. In an
order-free edge-insertion construction, however, acyclicity couples the
available actions; a shared neural policy can additionally couple the parameter
blocks. Child-node, output-head, or layerwise blocks are then approximations to
a dense Fisher rather than consequences of the decomposable reward.

\subsection{Rank deficiency and pseudo-inverses}

For any categorical probability vector \(p\in\Delta^{m-1}\),
\[
C(p)=\Diag(p)-pp^\top
\]
is positive semidefinite and satisfies
\[
C(p)\mathbf 1=0.
\]
Thus the Fisher block is singular in logit coordinates, reflecting the fact that logits
are identifiable only up to additive constants. Shared neural policies can have
additional null directions from redundant parameterisations. We distinguish two ways
of handling this:
\begin{enumerate}
\item work in an identifiable coordinate system, for example by fixing one reference
logit, or equivalently use the Moore--Penrose pseudo-inverse on
\(\mathcal U_\theta=\operatorname{range}F(\theta)\);
\item use the explicitly regularised direction
\[
(F+\lambda I)^{-1}h,
\qquad
\lambda>0.
\]
\end{enumerate}
The first gives the exact natural direction on the identifiable score tangent space.
The second is a damped approximation on the full parameter space; it is always
well-defined but is not the Moore-Penrose pseudo-inverse. Keeping the two conventions separate is essential when interpreting approximation guarantees for neural policies.

\section{Approximation of natural gradients}
\label{app:policy_induced_approximation}

This appendix proves the approximation result used in
Section~\ref{sec:fisher_rao_structure}. We use the notation of the main text. The exact
objective gradient is \(h(\theta):=\nabla_\theta\mathcal J(\theta)\), and the undamped
natural direction on the identifiable score subspace is \(F(\theta)^\dagger h(\theta)\).
The exact trajectory Fisher decomposes as
\[
F(\theta)
=
\sum_{t\ge0}M_t(\theta),
\qquad
M_t(\theta)
:=
\mathbb E_{\tau\sim q_\theta}
\big[
u_t(\theta;\tau)u_t(\theta;\tau)^\top
\big].
\]
A chosen approximation provides per-step Fisher contributions
\(\widetilde M_t(\theta)\), a Fisher surrogate
\[
\widetilde F(\theta)
:=
\sum_{t\ge0}\widetilde M_t(\theta),
\]
and a gradient surrogate \(\widetilde h(\theta)\). Theorem~
\ref{thm:structure_dependent_approx} compares the full-space damped pair
\(F+\lambda I\) and
\(\widetilde F+\lambda I\). The exact pseudo-inverse
case is recovered only after restricting both operators to a common identifiable
subspace, because a pseudo-inverse is not continuous when the surrogate rank changes.

\subsection{Proof of Theorem~\ref{thm:structure_dependent_approx}}

Define the self-adjoint relative perturbation
\[
E
\coloneqq
F_\lambda^{-1/2}
(\widetilde F-F)
F_\lambda^{-1/2}.
\]
Then
\[
\widetilde F_\lambda
=
F_\lambda^{1/2}(I+E)F_\lambda^{1/2}.
\]
The assumption \(\|E\|_{\mathrm{op}}=\rho_\lambda<1\) implies
\[
(1-\rho_\lambda)I\preceq I+E\preceq(1+\rho_\lambda)I,
\qquad
\|(I+E)^{-1}\|_{\mathrm{op}}\le\frac1{1-\rho_\lambda}.
\]
Let
\(b=F_\lambda^{-1/2}h\) and
\(e_h=F_\lambda^{-1/2}(\widetilde h-h)\). Since
\[
F_\lambda^{1/2}v_\lambda=b,
\qquad
F_\lambda^{1/2}\widetilde v_\lambda
=(I+E)^{-1}(b+e_h),
\]
we have
\[
F_\lambda^{1/2}
(\widetilde v_\lambda-v_\lambda)
=
\big((I+E)^{-1}-I\big)b
+
(I+E)^{-1}e_h.
\]
The identity
\[
(I+E)^{-1}-I=-(I+E)^{-1}E
\]
therefore gives
\[
\|\widetilde v_\lambda-v_\lambda\|_{F_\lambda}
\le
\frac{\rho_\lambda}{1-\rho_\lambda}\|b\|
+
\frac1{1-\rho_\lambda}\|e_h\|
=
\frac{\rho_\lambda}{1-\rho_\lambda}
\|v_\lambda\|_{F_\lambda}
+
\frac1{1-\rho_\lambda}
\|\widetilde h-h\|_{F_\lambda^{-1}},
\]
which proves \eqref{eq:structured_natdir_bound_main}.

If \(\widetilde h=h\), then
\[
h^\top\widetilde F_\lambda^{-1}h
=
b^\top(I+E)^{-1}b.
\]
The eigenvalues of \((I+E)^{-1}\) lie in
\([1/(1+\rho_\lambda),1/(1-\rho_\lambda)]\), which proves the final
claim. \(\square\)

The block-surrogate implication stated after the theorem follows from the same
normalisation. Set \(A_B=B+\lambda I\) and
\(R_B=A_B^{-1/2}(F-B)A_B^{-1/2}\). If
\(\gamma_{B,\lambda}=\|R_B\|_{\mathrm{op}}<1\), then
\((1-\gamma_{B,\lambda})A_B\preceq F_\lambda
\preceq(1+\gamma_{B,\lambda})A_B\). Therefore
\[
\left\|
F_\lambda^{-1/2}(F-B)F_\lambda^{-1/2}
\right\|_{\mathrm{op}}
\le
\frac{\gamma_{B,\lambda}}{1-\gamma_{B,\lambda}}.
\]
Indeed, the numerator of the corresponding generalised Rayleigh quotient is at most
\(\gamma_{B,\lambda}x^\top A_Bx\) in absolute value, while its denominator is at
least \((1-\gamma_{B,\lambda})x^\top A_Bx\).
Thus \(\gamma_{B,\lambda}<1/2\) is sufficient for the relative-error hypothesis
\(\rho_\lambda<1\).

\begin{figure}[!ht]
  \centering
  \includegraphics[width=\textwidth]{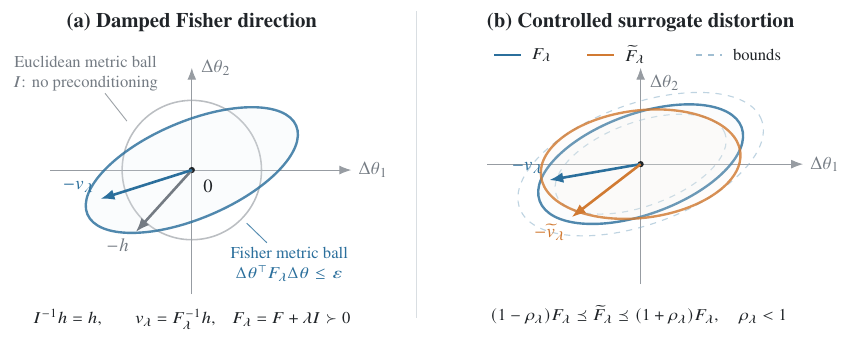}
  \caption{Local interpretation of damped Fisher preconditioning and
  Theorem~\ref{thm:structure_dependent_approx}.
  (a) Euclidean geometry uses the identity metric, so
  \(I^{-1}h=h\) and the gradient is not preconditioned. The Fisher matrix describes
  the local KL curvature of the policy family; consequently,
  \(-v_\lambda=-F_\lambda^{-1}h\) rescales parameter directions according to their
  distributional effect, with damping
  \(F_\lambda=F+\lambda I\succ0\) ensuring a positive-definite metric.
  (b) When \(\widetilde h=h\), \(\rho_\lambda<1\) controls the relative distortion
  between the exact and surrogate metric balls and hence the difference between
  \(v_\lambda\) and \(\widetilde v_\lambda\); the dashed ellipses indicate the
  scaled exact-metric bounds. The general theorem additionally accounts for
  gradient error \(\widetilde h-h\).}
  \label{fig:fisher_preconditioning_geometry}
\end{figure}

\subsection{Descent and absolute-error consequences}
\label{app:approximation_consequences}

\begin{cor}[Local improvement under a surrogate Fisher]
\label{cor:surrogate_fisher_descent}
Let \(h=\nabla\mathcal J(\theta)\), \(\widetilde h=h\), and suppose the assumptions
of Theorem~\ref{thm:structure_dependent_approx} hold. If, locally,
\[
\mathcal J(\theta+s)
\le
\mathcal J(\theta)+h^\top s+\frac L2\|s\|_{F_\lambda}^2,
\]
then
\[
\mathcal J(\theta-\eta\widetilde v_\lambda)
\le
\mathcal J(\theta)
-
\eta\left(1-\frac{L\eta}{2(1-\rho_\lambda)}\right)
h^\top\widetilde F_\lambda^{-1}h.
\]
Hence every \(0<\eta<2(1-\rho_\lambda)/L\) for which the proposed step remains
in the neighbourhood where the smoothness inequality holds gives strict descent unless
\(h=0\).
\end{cor}

\begin{proof}
The relative-error assumption implies
\[
(1-\rho_\lambda)F_\lambda
\preceq\widetilde F_\lambda,
\qquad\text{and hence}\qquad
\|\widetilde v_\lambda\|_{F_\lambda}^2
\le
\frac1{1-\rho_\lambda}
h^\top\widetilde F_\lambda^{-1}h.
\]
For a step \(s=-\eta\widetilde v_\lambda\) in that neighbourhood, applying the local
smoothness inequality gives the stated bound. Positive definiteness of
\(\widetilde F_\lambda\) gives strict descent when \(h\neq0\).
\end{proof}

\begin{cor}[Absolute errors and inverse stability]
\label{cor:additive_fisher_constants}
Under the notation of Theorem~\ref{thm:structure_dependent_approx}, define the
full-space damped curvature floor
\[
\mu_\lambda
\coloneqq
\lambda_{\min}(F_\lambda)
=\lambda_{\min}(F)+\lambda
\]
and suppose
\[
\delta\coloneqq\|\widetilde F-F\|_{\mathrm{op}}<\mu_\lambda,
\qquad
\|\widetilde h-h\|\le\varepsilon.
\]
Then
\[
\|F_\lambda^{-1}\|_{\mathrm{op}}=\mu_\lambda^{-1},
\qquad
\|\widetilde F_\lambda^{-1}\|_{\mathrm{op}}
\le(\mu_\lambda-\delta)^{-1},
\]
and
\begin{equation}
\label{eq:additive_natdir_bound_main}
\|\widetilde v_\lambda-v_\lambda\|
\le
\frac{\varepsilon}{\mu_\lambda}
+
\frac{\delta\|\widetilde h\|}
{\mu_\lambda(\mu_\lambda-\delta)}.
\end{equation}
\end{cor}

\begin{proof}
Weyl's eigenvalue inequality gives
\[
\lambda_{\min}(\widetilde F_\lambda)
\ge
\lambda_{\min}(F_\lambda)
-\|\widetilde F-F\|_{\mathrm{op}}
\ge\mu_\lambda-\delta>0.
\]
The inverse-norm bounds follow. Adding and subtracting
\(F_\lambda^{-1}\widetilde h\) gives
\[
\widetilde v_\lambda-v_\lambda
=
F_\lambda^{-1}(\widetilde h-h)
+
F_\lambda^{-1}(F-\widetilde F)
\widetilde F_\lambda^{-1}\widetilde h.
\]
Applying the inverse-norm bounds to this identity gives
\eqref{eq:additive_natdir_bound_main}.
\end{proof}

Because \(\mu_\lambda\ge\lambda\), the condition \(\delta<\lambda\) is a
conservative sufficient condition for Corollary~\ref{cor:additive_fisher_constants}
and for \(\rho_\lambda<1\). Without a relative-error condition,
positive semidefiniteness still gives the coarser bound
\[
\|\widetilde v_\lambda-v_\lambda\|
\le
\frac{\|\widetilde h-h\|}{\lambda}
+
\frac{\|\widetilde F-F\|_{\mathrm{op}}\|\widetilde h\|}{\lambda^2}.
\]
For an undamped comparison, \(F\) and \(\widetilde F\) must instead share an
identifiable subspace on which the smallest positive curvature is bounded away from
zero. We write this curvature as
\(\mu_+(F)\coloneqq\lambda_{\min}^{+}(F)\).

\subsection{Policy-specific Fisher conditioning}
\label{app:fisher_conditioning}

\begin{prop}[Policy-specific Fisher conditioning]
\label{prop:gflownet_fisher_conditioning}
The following bounds make the relevant curvature explicit.
\begin{enumerate}[label=(\roman*),nosep]
\item For a tabular softmax policy, restrict each active state block to the
sum-zero logit subspace. If \(d_\theta(s)\ge d_{\min}>0\) and
\(\pi_\theta(a\mid s)\ge p_{\min}>0\) for every active state and valid action, then
\[
\mu_+(F)\ge d_{\min}p_{\min}.
\]
\item For a shared softmax policy, let
\(P_s=I-\mathbf1\mathbf1^\top/|\mathcal A(s)|\). If
\(\pi_\theta(a\mid s)\ge p_{\min}\) for every valid action at every visited state,
then every parameter direction \(w\) satisfies
\[
w^\top Fw
\ge
p_{\min}
\mathbb E_{\tau\sim q_\theta}\!\left[
\sum_{t\ge0}\|P_{s_t}J_\theta(s_t)w\|^2
\right].
\]
\item Suppose a factorisation-aligned policy makes one categorical parent-set
decision per child using disjoint parameter blocks, so that
\(F=\bigoplus_jF_j\) and, on the sum-zero logit subspace,
\(F_j=d_jC(\pi_j)\). If \(d_j>0\) is the effective occupancy of each nontrivial
child block \(j\) and every valid parent set has probability at least \(p_j>0\),
then
\[
\mu_+(F)\ge\min_j d_jp_j.
\]
\end{enumerate}
\end{prop}

\begin{proof}
For any categorical probability vector \(p\), let
\(P=I-\mathbf1\mathbf1^\top/m\), where \(m\) is the number of categories. For every
\(x\in\mathbb R^m\),
\begin{align}
x^\top C(p)x
&=
\min_{c\in\mathbb R}\sum_{i=1}^m p_i(x_i-c)^2 \\
&\ge
p_{\min}\min_{c\in\mathbb R}\sum_{i=1}^m(x_i-c)^2
=
p_{\min}x^\top P x.
\label{eq:appendix_categorical_lower_bound}
\end{align}
For a tabular policy, \(P\) is the identity on the sum-zero logit subspace.
Combining this inequality with
\(F=\bigoplus_s d_\theta(s)C(\pi_s)\) proves (i). Applying the same inequality inside
\eqref{eq:appendix_shared_exact_fisher} proves (ii). Applying the categorical bound
to each \(F_j=d_jC(\pi_j)\) and taking the smallest positive eigenvalue of their
direct sum proves (iii).
\end{proof}

\subsection{Score-form gradients and the Fisher range}
\label{app:score_compatibility}

To prove Proposition~\ref{prop:score_compatible_gradient}, define the linear operator
\(T:L^2(q_\theta)\to\mathbb R^{\dim(\theta)}\) by
\[
Tr=\mathbb E_{q_\theta}[z(\tau)r(\tau)].
\]
Its adjoint is \(T^\ast v=z(\tau)^\top v\), and
\[
TT^\ast
=
\mathbb E_{q_\theta}[z(\tau)z(\tau)^\top]
=F.
\]
The moment assumption on \(z\) makes \(T\) bounded. Because its codomain is
finite-dimensional, \(\operatorname{range}T\) is closed, and therefore
\(\operatorname{range}(TT^\ast)=\operatorname{range}T\). Thus \(h=Tr\) lies in
\(\operatorname{range}T=\operatorname{range}F=\mathcal U_\theta\). Moreover,
standard pseudo-inverse identities give
\[
h^\top F^\dagger h
=
\langle r,T^\ast(TT^\ast)^\dagger Tr\rangle_{L^2(q_\theta)}
=
\|\Pi_{\operatorname{range}T^\ast}r\|_{L^2(q_\theta)}^2
\le
\|r\|_{L^2(q_\theta)}^2.
\]
This proves the claim. \(\square\)

\subsection{Monte Carlo approximation}

Theorem~\ref{thm:structure_dependent_approx} is deterministic: it bounds the
natural-direction error in terms of a gradient error and a Fisher-operator error. The
sampled case is obtained by applying the same bound to random estimators. Suppose
\[
\widehat h_N(\theta)\xrightarrow{p}h(\theta),
\qquad
\widehat M_{t,N}(\theta)\xrightarrow{p}M_t(\theta)
\]
for the relevant per-step Fisher contributions, and suppose the summed operator error
satisfies
\[
\sum_{t\ge0}
\|\widehat M_{t,N}(\theta)-M_t(\theta)\|_{\mathrm{op}}
\xrightarrow{p}0.
\]
Let
\[
\widehat F_N(\theta):=\sum_{t\ge0}\widehat M_{t,N}(\theta).
\]
For any fixed \(\lambda>0\), define
\[
\widehat v_{N,\lambda}
\coloneqq
(\widehat F_N+\lambda I)^{-1}\widehat h_N,
\qquad
v_\lambda
\coloneqq
(F+\lambda I)^{-1}h.
\]
Then
\[
\widehat v_{N,\lambda}\xrightarrow{p}v_\lambda.
\]
Indeed, the relative error
\[
\left\|
(F+\lambda I)^{-1/2}
(\widehat F_N-F)
(F+\lambda I)^{-1/2}
\right\|_{\mathrm{op}}
\xrightarrow{p}0,
\]
and Theorem~\ref{thm:structure_dependent_approx} applies with probability tending to
one. For the undamped direction, convergence instead requires a fixed identifiable
subspace with \(\mu_+(F)>0\), together with eventual preservation of that subspace and
operator error smaller than \(\mu_+(F)\). This distinction is important for empirical
outer-product neural Fishers, whose rank is at most the number of sampled per-step score
increments.

\subsection{A generic importance-weighted stochastic estimator}

For completeness, we record the change-of-measure identity underlying the sampled
gradient estimator \(\widehat h(\theta)\) used in the natural-gradient update
\eqref{eq:gflownet_natgrad_update}. Let \(\mu_k(\tau)\) be the
proposal distribution used to collect trajectories at iteration \(k\), and let
\(\nu_k(\tau)\) be the reference distribution under which the objective is defined.
Assume \(\mu_k(\tau)>0\) whenever \(\nu_k(\tau)>0\). Then, for any measurable function
\(g(\tau)\) that is integrable under \(\nu_k\) (and may be vector-valued),
\begin{equation}
\label{eq:appendix_change_of_measure}
\mathbb E_{\tau\sim \nu_k}[g(\tau)]
=
\mathbb E_{\tau\sim \mu_k}
\left[
\frac{\nu_k(\tau)}{\mu_k(\tau)}\,g(\tau)
\right].
\end{equation}
Taking
\[
g(\tau)
=
\ell'\!\big(\delta_{\theta_k,\phi_k,\xi_k}(\tau)\big)\,
z_{\theta_k}(\tau)
\]
gives the unbiased estimator
\begin{equation}
\widehat h_k
=
\frac1N\sum_{i=1}^N
w_k(\tau_i)\,
\ell'\!\big(\delta_{\theta_k,\phi_k,\xi_k}(\tau_i)\big)\,
z_{\theta_k}(\tau_i),
\qquad
w_k(\tau):=\frac{\nu_k(\tau)}{\mu_k(\tau)},
\qquad
\tau_i\sim\mu_k.
\label{eq:appendix_importance_weighted_gradient}
\end{equation}
Here \(z_{\theta_k}(\tau)=\nabla_\theta\log q_{\theta_k}(\tau)\) denotes the full
trajectory score. In practice, one may clip or self-normalise the weights
\(w_k\), which introduces bias but can reduce variance.

\section{GFlowNet realisability, support, and terminal-law characterisation}
\label{app:gflownet_realizability}

This appendix separates a basic representability question from the optimisation geometry
studied in the main text. Before optimising a forward-policy family, the state graph and
policy class must be able to represent the desired terminal law
\(\pi(x)\propto R(x)\). The statements below make those design conditions explicit;
they are not additional geometric results.

Let \(\mathcal G=(\mathcal S,\mathcal E)\) be a finite directed acyclic graph with a
distinguished source \(s_0\) and terminal set \(\mathcal X\), with
\(s_0\notin\mathcal X\). We assume that terminal states are sinks and that every sink
reachable from \(s_0\) is terminal. Consequently, a normalised forward policy terminates
in \(\mathcal X\). A complete trajectory is a directed path
\[
\tau=(s_0,s_1,\dots,s_T),
\qquad
(s_t,s_{t+1})\in\mathcal E,
\qquad
s_T\in\mathcal X.
\]
Write \(\mathcal T(x)\) for the set of complete trajectories ending at
\(x\in\mathcal X\). Let \(P_F(s'\mid s)\) denote a forward Markov policy on outgoing
edges, and let \(P_B(s\mid s')\) denote a backward Markov policy on incoming edges. The
terminal law induced by the forward policy is
\begin{equation}
\label{eq:app_terminal_law_from_forward}
q_F(x)
:=
\sum_{\tau=(s_0,\dots,s_T)\in \mathcal T(x)}
\prod_{t=0}^{T-1} P_F(s_{t+1}\mid s_t).
\end{equation}
Given a nonnegative reward \(R:\mathcal X\to\mathbb R_+\), define the target terminal law
\begin{equation}
\label{eq:app_target_terminal_law}
\pi(x)
=
\frac{R(x)}{\sum_{x'\in\mathcal X}R(x')},
\qquad x\in\mathcal X,
\end{equation}
whenever the denominator is finite and positive.

\begin{figure}[!ht]
  \centering
  \includegraphics[width=\textwidth]{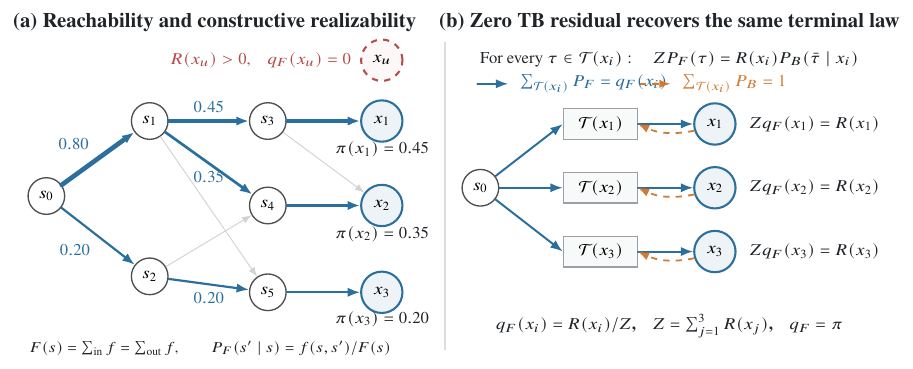}
  \caption{Reachability, constructive realisability, and terminal-law
  characterisation.
  (a) In a three-terminal illustration, the masses \(\pi(x_i)\) are routed along selected
  source-to-terminal paths; edge widths and labels represent the resulting flow
  \(f(s,s')\). Flow conservation defines
  \(P_F(s'\mid s)=f(s,s')/F(s)\), while a positive-reward terminal outside the
  source-reachable subgraph necessarily has \(q_F(x_u)=0\). Faint arrows are
  admissible but unused edges.
  (b) The boxes collect all trajectories in the same DAG ending at each \(x_i\).
  At zero TB residual, summing the pathwise identity within each box gives
  \(Zq_F(x_i)=R(x_i)\), because the corresponding reverse-path probabilities sum to
  one. Thus \(q_F=\pi\). The construction and summation apply to any finite terminal
  set, not only three categories.}
  \label{fig:gflownet_realizability_terminal_law}
\end{figure}

\paragraph{Reachability and constructive realisability.}
Reachability is necessary by design. If there is no path from \(s_0\) to a terminal
\(x\), then \(\mathcal T(x)=\varnothing\), so
Eq.~\eqref{eq:app_terminal_law_from_forward} gives \(q_F(x)=0\) for every forward
policy. A target assigning \(R(x)>0\), and hence \(\pi(x)>0\), to such a terminal
cannot be represented on that state graph.

Conversely, reachability is sufficient at the level of an unrestricted statewise
Markov policy under two explicit qualifications: every terminal with \(\pi(x)>0\)
must be reachable, and the policy must be allowed to assign zero probability to unused
edges. To see this, for each
terminal state \(x\) with \(\pi(x)>0\), choose one directed path
\[
\tau_x=(s_0,s^{(x)}_1,\dots,s^{(x)}_{T_x}=x)
\]
from \(s_0\) to \(x\). The selected paths need not form a tree: they may share
prefixes, merge, or split.

Define a nonnegative edge flow by
\[
f(s,s')
:=
\sum_{x:\,(s,s')\in\tau_x}\pi(x).
\]
Define a state flow by
\[
F(s_0):=1,
\qquad
F(s):=\sum_{s':(s',s)\in\mathcal E} f(s',s),
\quad s\neq s_0.
\]
Every selected path entering a nonterminal state \(s\) also leaves it exactly once, so
flow conservation gives
\[
F(s)
=
\sum_{s':(s,s')\in\mathcal E} f(s,s').
\]
For terminal states, since terminals are sinks and no selected path passes through a
terminal before ending, we have
\[
F(x)=\pi(x),
\qquad x\in\mathcal X.
\]

Now define the forward policy on positive-flow nonterminal states by
\begin{equation}
\label{eq:app_realizability_forward_policy}
P_F(s'\mid s):=\frac{f(s,s')}{F(s)}.
\end{equation}
On zero-flow states, choose any normalised outgoing distribution; those states are not
visited under the constructed policy. Such a distribution exists at every
source-reachable nonterminal because every source-reachable sink is terminal; states
outside the source-reachable subgraph can be discarded. Moreover, if \(F(s)=0\), then
every outgoing edge has \(f(s,s')=0\). Thus the identity
\(F(s)P_F(s'\mid s)=f(s,s')\) holds on both positive- and zero-flow states.

We show by induction over a topological ordering of the DAG that the probability of
reaching any state \(s\) under \(P_F\) is \(F(s)\). This is true at the source because
\(F(s_0)=1\). If the claim holds for all parents of \(s'\), then
\[
\Pr(s')
=
\sum_{s:(s,s')\in\mathcal E}
\Pr(s)P_F(s'\mid s)
=
\sum_{s:(s,s')\in\mathcal E}
F(s)P_F(s'\mid s)
=
\sum_{s:(s,s')\in\mathcal E} f(s,s')
=
F(s').
\]
Therefore, for a terminal \(x\),
\[
q_F(x)=\Pr(x)=F(x)=\pi(x).
\]
This proves the claim.

This construction is an existence argument for an unrestricted statewise Markov policy,
not a guarantee for an arbitrary parameterisation. A parameterised GFlowNet must be
expressive enough to realise compatible conditionals at all states. In particular, a
finite-logit softmax has full support and cannot implement the zero probabilities used by
this construction. When unused edges are present, this particular policy therefore lies
only in the closure of the softmax family and can be approximated by making their
probabilities small. Another full-support realisation may exist for a particular graph,
but reachability alone does not guarantee one. Shared neural parameters impose
additional cross-state compatibility constraints and need not realise the construction
exactly.

\paragraph{What exact Trajectory Balance implies.}
Let \(Z>0\) be a scalar source flow. For a complete trajectory
\(\tau=(s_0,\dots,s_T=x)\in\mathcal T(x)\), write
\[
P_F(\tau):=\prod_{t=0}^{T-1}P_F(s_{t+1}\mid s_t),
\qquad
P_B(\bar\tau\mid x):=\prod_{t=0}^{T-1}P_B(s_t\mid s_{t+1}),
\]
where \(\bar\tau\) denotes the reversed trajectory from \(x\) to \(s_0\). We assume that
\(P_B(\cdot\mid s')\) is normalised over valid parents and that, from every terminal
\(x\), the induced reverse chain terminates at \(s_0\).

The squared Trajectory-Balance residual used during training acts as a soft penalty away
from an optimum. The exact terminal-law statement is stronger: it applies when the
pathwise residual is zero for every complete trajectory, or equivalently when
\begin{equation}
\label{eq:app_exact_TB_pathwise}
Z\,P_F(\tau)=R(x)\,P_B(\bar\tau\mid x).
\end{equation}
Fix \(x\in\mathcal X\). Summing \eqref{eq:app_exact_TB_pathwise} over all complete
trajectories ending at \(x\) gives
\begin{align}
Z\sum_{\tau\in\mathcal T(x)}P_F(\tau)
&=
R(x)\sum_{\tau\in\mathcal T(x)}P_B(\bar\tau\mid x)
\nonumber\\
Zq_F(x)&=R(x).
\label{eq:app_tb_terminal_law}
\end{align}
The second equality uses normalisation of the reverse-trajectory distribution. Summing
over \(x\) yields
\[
Z=\sum_{x\in\mathcal X}R(x).
\]
Hence \(q_F(x)=R(x)/Z=\pi(x)\). A merely small or sampled Trajectory-Balance
residual does not imply this exact identity; it only encourages it through the training
objective.

\paragraph{Detailed Balance.}
The analogous conclusion follows by multiplying edgewise flow identities along a
complete trajectory. Suppose a nonnegative state-flow function
\(F:\mathcal S\to\mathbb R_+\) satisfies, for every edge
\((s,s')\in\mathcal E\),
\begin{equation}
\label{eq:app_exact_DB_edge}
F(s)\,P_F(s'\mid s)=F(s')\,P_B(s\mid s'),
\end{equation}
and for every terminal state \(x\in\mathcal X\),
\begin{equation}
\label{eq:app_terminal_flow_reward}
F(x)=R(x).
\end{equation}
Fixing a complete trajectory \(\tau=(s_0,\dots,s_T=x)\) and multiplying
\eqref{eq:app_exact_DB_edge} along the edges of \(\tau\) gives
\[
F(s_0)\prod_{t=0}^{T-1}P_F(s_{t+1}\mid s_t)
=
F(x)\prod_{t=0}^{T-1}P_B(s_t\mid s_{t+1}).
\]
Using \(F(x)=R(x)\), this is exactly the pathwise Trajectory-Balance identity
\eqref{eq:app_exact_TB_pathwise} with \(Z=F(s_0)\). Therefore
\(q_F(x)=R(x)/F(s_0)\), and normalisation gives
\(F(s_0)=\sum_xR(x)\).

In summary, exact realisability of a reward-defined terminal law on a finite state DAG
requires reachability of all positive-reward terminal states and sufficient expressivity
of the forward policy class. Under exact Trajectory Balance or Detailed Balance, the
terminal law induced by the forward policy is proportional to reward.

\section{Detailed computational regimes and factorised-target Fisher surrogates}
\label{app:computational_regimes_detailed}

This appendix expands the computational regimes described in
Section~\ref{sec:natgrad_gflownet_training}. When the objective gradient lies in the
identifiable score subspace \(\mathcal U_\theta=\operatorname{range}(F(\theta))\), the
undamped trajectory natural direction is
\[
v^\star(\theta)=F(\theta)^\dagger h(\theta),
\qquad
h(\theta):=\nabla_\theta \mathcal J(\theta),
\]
where, by Theorem~\ref{thm:riemannian_policy_induced}, the trajectory Fisher decomposes
as
\[
F(\theta)
=
\sum_{t\ge0}M_t(\theta),
\qquad
M_t(\theta)
:=
\mathbb E_{\tau\sim q_\theta}
\big[
u_t(\theta;\tau)u_t(\theta;\tau)^\top
\big].
\]
For a trajectory whose terminal state is \(s_T\), define the time-resolved active
occupancy of a nonterminal state \(s\)
\[
d_{\theta,t}(s)
\coloneqq
\Pr_{\tau\sim q_\theta}(t<T,\ s_t=s).
\]
For each \(t\), these occupancies form a sub-probability mass over nonterminal states in
variable-length constructions, and
\(d_\theta(s)=\sum_{t\ge0}d_{\theta,t}(s)\).
Thus the computational question is how the per-step second moments \(M_t(\theta)\) are
obtained: exactly, by Monte Carlo estimation, or by exploiting structure in the target
and construction process.

\subsection{Exact regimes}

When the trajectory law admits exact marginalisation, the Fisher can be computed without
statistical approximation. This includes finite state DAGs for which occupancies are
available by dynamic programming, as well as tabular settings where the statewise
decomposition in Corollary~\ref{cor:exact_tabular_fisher_main} is directly computable.
In such cases the natural gradient is an exact algorithmic object rather than a sampled
or surrogate preconditioner, although forming or solving the resulting system may still
be computationally expensive.

More precisely, conditioning on the active state gives
\[
M_t(\theta)
=
\sum_s d_{\theta,t}(s)
\sum_{a\in\mathcal A(s)}
\pi_\theta(a\mid s)\,
\nabla_\theta\log\pi_\theta(a\mid s)\,
\nabla_\theta\log\pi_\theta(a\mid s)^\top.
\]
Thus \(F=\sum_tM_t\) is exactly computable when the active occupancies, conditional action
probabilities, and policy scores are exactly evaluable and the state--action and
time-indexed sums are finite or exactly summable. This is not a tractability claim:
enumerating the states or forming a dense \(p\times p\) matrix can remain prohibitive
when \(p=\dim(\theta)\) is large.

In the tabular softmax case this reduces to the occupancy-weighted expression
\[
F(\theta)
=
\bigoplus_{s\in\mathcal S}d_\theta(s)C(\pi_s),
\]
where \(d_\theta(s)\) is the expected number of visits to \(s\) and
\(C(\pi_s)=\Diag(\pi_s)-\pi_s\pi_s^\top\).

\subsection{Monte Carlo regimes}

When exact marginalisation is unavailable but trajectories can be sampled, the same
decomposition yields a Monte Carlo estimator. For i.i.d.\ trajectories
\(\tau^{(1)},\dots,\tau^{(N)}\sim \mu\), define
\[
\widehat F_N(\theta)
=
\frac1N\sum_{i=1}^N
\omega_i
\sum_{t\ge0}
u_t(\theta;\tau^{(i)})u_t(\theta;\tau^{(i)})^\top,
\]
where \(\omega_i=1\) for on-policy sampling \(\mu=q_\theta\), and
\[
\omega_i=\frac{q_\theta(\tau^{(i)})}{\mu(\tau^{(i)})}
\]
for importance-corrected off-policy sampling. This estimator targets the decomposed
trajectory Fisher itself.

\begin{prop}[Monte Carlo Fisher consistency]
\label{prop:appendix_mc_regime}
Fix \(\theta\in\mathbb R^p\) with \(p<\infty\), and let
\[
Y(\tau)
:=
\omega(\tau)
\sum_{t\ge0}
u_t(\theta;\tau)u_t(\theta;\tau)^\top,
\qquad
\tau\sim \mu,
\]
where \(\omega(\tau)=1\) in the on-policy case and
\(\omega(\tau)=q_\theta(\tau)/\mu(\tau)\) in the importance-weighted case. Assume
\[
\mathbb E_{\mu}\|Y(\tau)\|_{\mathrm{op}}<\infty
\]
and, in the off-policy case, \(\mu(\tau)>0\) whenever \(q_\theta(\tau)>0\). Then
\[
\|\widehat F_N(\theta)-F(\theta)\|_{\mathrm{op}}
\xrightarrow[]{a.s.}0
\qquad
\text{as }N\to\infty.
\]
If \(\widehat h_N(\theta)\to h(\theta)\) in probability, then for every fixed
\(\lambda>0\),
\[
\widehat v_{N,\lambda}(\theta)
:=
(\widehat F_N(\theta)+\lambda I)^{-1}\widehat h_N(\theta)
\]
converges in probability to
\(v_\lambda(\theta)
=(F(\theta)+\lambda I)^{-1}h(\theta)\).
For the undamped direction, let \(P_{\mathcal U_\theta}\) be the orthogonal projector
onto the fixed subspace
\(\mathcal U_\theta=\operatorname{range}(F)\). If \(\mu_+(F)>0\), then the inverse of
\[
\left.
P_{\mathcal U_\theta}\widehat F_NP_{\mathcal U_\theta}
\right|_{\mathcal U_\theta}
\]
exists with probability tending to one, and its product with
\(P_{\mathcal U_\theta}\widehat h_N\) converges in probability to
\(F^\dagger h\). Convergence of the unrestricted pseudo-inverse
\(\widehat F_N^\dagger\) is not implied without additional rank and subspace stability.
\end{prop}

\begin{proof}
Since \(\lvert Y_{jk}\rvert\le \|Y\|_{\mathrm{op}}\), every entry is integrable.
Applying the scalar strong law to the finitely many entries gives
\(\widehat F_N\to\mathbb E_\mu[Y]\) entrywise almost surely and hence in operator
norm by equivalence of norms in finite dimension. By definition,
\(\mathbb E_\mu[Y]=F\) on-policy, and the same identity follows off-policy by change
of measure. For \(\lambda>0\), continuity of inversion gives
\((\widehat F_N+\lambda I)^{-1}\to(F+\lambda I)^{-1}\) almost surely. Combining this
with \(\widehat h_N\to h\) in probability gives the damped conclusion by Slutsky's
theorem; no independence between the two estimators is required.

For the undamped conclusion, let \(Q\) have orthonormal columns spanning
\(\mathcal U_\theta\). Then
\(Q^\top\widehat F_NQ\to Q^\top FQ\) almost surely in operator norm, while
\(Q^\top FQ\) is positive definite with smallest eigenvalue \(\mu_+(F)\).
Weyl's inequality makes the sampled restriction positive definite with probability
tending to one, and continuity of inversion gives
\[
Q(Q^\top\widehat F_NQ)^{-1}Q^\top\widehat h_N
\xrightarrow{p}
Q(Q^\top FQ)^{-1}Q^\top h
=F^\dagger h.
\]
This is the projected estimator stated in the proposition.
\end{proof}

This is a consistency result, not a central limit theorem. A \(\sqrt N\)-limit for
\(\operatorname{vec}(\widehat F_N-F)\) additionally requires finite second moments,
which can fail under heavy-tailed importance weights.

\subsection{Structure-exploitable regimes}

The most important regime is the one in which exact marginalisation is unavailable and
naive trajectory sampling is statistically inefficient, but the target has enough
locality or factorisation to make the Fisher expectation approximable. Suppose the
terminal reward factorises as
\begin{equation}
\label{eq:appendix_factorized_reward}
R(x)
=
\prod_{\alpha\in\mathcal C}\psi_\alpha(x_\alpha),
\end{equation}
or equivalently the terminal energy is additive,
\[
E(x)
=
-\sum_{\alpha\in\mathcal C}\log\psi_\alpha(x_\alpha).
\]
This equality is understood on the positive support of the factors; a zero factor
corresponds to the extended value \(E(x)=+\infty\).
This covers factor-graph targets, low-treewidth graphical models, sparse interaction
systems, and bounded-overlap motif models.

The relevance of such structure is that the local Fisher contribution
\[
M_t(\theta)
=
\mathbb E_{\tau\sim q_\theta}
\big[
u_t(\theta;\tau)u_t(\theta;\tau)^\top
\big]
\]
can be written by conditioning on the current partial state:
\[
M_t(\theta)
=
\sum_s d_{\theta,t}(s)
\sum_{a\in\mathcal A(s)}
\pi_\theta(a\mid s)\,
\nabla_\theta\log\pi_\theta(a\mid s)\,
\nabla_\theta\log\pi_\theta(a\mid s)^\top.
\]
Target factorisation does not by itself factorise this policy Fisher. It becomes useful
when it matches the construction process: local target factors can identify sensible
parameter blocks and can make policy occupancies or continuation summaries tractable
through dynamic programming, separator methods, or message passing. Replacing the
intractable quantities in \(M_t\) by these target-informed summaries defines a
surrogate; the exact expectation remains the one under \(q_\theta\).

Let \(\widetilde M_t(\theta)\) denote the resulting structured
approximation to \(M_t(\theta)\), and
define
\[
\widetilde F(\theta)
=
\sum_{t\ge0}\widetilde M_t(\theta).
\]
The approximation error is controlled additively across steps. Suppose
\(\widetilde M_t(\theta)\succeq0\) and, for each step \(t\),
\[
\|\widetilde M_t(\theta)-M_t(\theta)\|_{\mathrm{op}}
\le e_t,
\qquad
\sum_{t\ge0}e_t<\infty.
\]
Writing \(\Delta_F=\sum_{t\ge0}e_t\), the triangle inequality gives
\[
\|\widetilde F(\theta)-F(\theta)\|_{\mathrm{op}}
\le \Delta_F.
\]
If
\(\operatorname{range}(\widetilde F)=\operatorname{range}(F)=\mathcal U_\theta\)
and \(\mu_+(F)>\Delta_F\), then, for \(h\in\mathcal U_\theta\),
Weyl's inequality and the inverse identity restricted to \(\mathcal U_\theta\) give
\[
\|\widetilde F^\dagger h-F^\dagger h\|
\le
\frac{\Delta_F\|h\|}
{\mu_+(F)\bigl(\mu_+(F)-\Delta_F\bigr)}.
\]
For any \(\lambda>\Delta_F\), the damped relative error satisfies
\[
\left\|
(F+\lambda I)^{-1/2}
(\widetilde F-F)
(F+\lambda I)^{-1/2}
\right\|_{\mathrm{op}}
\le\frac{\Delta_F}{\lambda}<1,
\]
so Theorem~\ref{thm:structure_dependent_approx} directly controls the damped
surrogate direction. Thus the undamped statement requires a stable common identifiable
subspace, whereas the damped statement does not.

The approximations \(\widetilde M_t\) may arise in several ways.
When a low-treewidth target is paired with a factorisation-aligned construction policy,
junction-tree or dynamic-programming calculations can provide the required local
occupancy summaries exactly. In loopy sparse settings, belief propagation can provide
approximations to those summaries. In bounded-overlap constructions, truncated
neighbourhoods or separator summaries play an analogous role. The target factorisation
therefore supplies computational structure for the surrogate; it does not, on its own,
make the policy Fisher factorised.

A useful schematic example is a target-informed message-passing construction. Define the
conditional Fisher at state \(s\) by
\[
A_t(s)
\coloneqq
\sum_{a\in\mathcal A(s)}
\pi_\theta(a\mid s)\,
\nabla_\theta\log\pi_\theta(a\mid s)\,
\nabla_\theta\log\pi_\theta(a\mid s)^\top.
\]
Suppose there is a discrete local summary
\(c_t:\mathcal S\to\mathcal Z_t\), determined by a neighbourhood subgraph
\(\mathcal N_t\), such that \(A_t(s)=\overline A_t(c_t(s))\). Let
\[
p_t(c)
\coloneqq
\sum_{s:\,c_t(s)=c}d_{\theta,t}(s).
\]
Then
\[
M_t=\sum_{c\in\mathcal Z_t}p_t(c)\overline A_t(c).
\]
Let \(\widetilde p_t\) be a nonnegative message-passing approximation to \(p_t\), and
define
\[
\widetilde M_t^{\mathrm{BP}}
=
\sum_{c\in\mathcal Z_t}\widetilde p_t(c)\overline A_t(c).
\]
This surrogate is positive semidefinite. If
\(\sup_{c\in\mathcal Z_t}\|\overline A_t(c)\|_{\mathrm{op}}\le K_t\) and
\(\|\widetilde p_t-p_t\|_1\le\varepsilon_t\), then
\[
\|\widetilde M_t^{\mathrm{BP}}-M_t\|_{\mathrm{op}}
\le K_t\varepsilon_t.
\]
Thus \(e_t=K_t\varepsilon_t\) supplies the per-step error required above.
Graphical-model machinery is
therefore useful only to the extend that it accurately approximates the policy-induced occupancy summaries and the conditional Fisher is determined by the retained local
summary. If the latter dependence is itself approximated, its operator error must be added to \(e_t\). Target factorisation alone does not establish either condition.

Taken together, the three regimes describe different ways of accessing the same Fisher-defined natural direction. Small tabular state DAGs and exact occupancy computations belong to the exact regime; hypergrid-type benchmarks with trajectory sampling belong to the Monte Carlo regime; and factorised or locally coupled targets can enter the structure-exploitable regime when paired with an aligned construction and
accurate local summaries. The distinction is not between different optimisation principles, but between different computational routes to the same trajectory-Fisher update.

\paragraph{Coordinate-factorised grid surrogate.}
For the two-dimensional fixed-horizon grid benchmarks, let \(m_t(x,y)\) be the
policy state mass, normalised at each time \(t\), and define
\[
m_{x,t}(x)=\sum_y m_t(x,y),
\qquad
m_{y,t}(y)=\sum_x m_t(x,y),
\qquad
\widetilde m_t(x,y)=m_{x,t}(x)m_{y,t}(y).
\]
The surrogate replaces \(m_t\) by \(\widetilde m_t\) in the tabular block
\(m_t(x,y)C(\pi_\theta(\cdot\mid x,y,t))\); it does not diagonalise the categorical action covariance. In the reported fixed-horizon 8-Gaussians and Rings comparisons (Section~\ref{sec:lrw} and Appendix~\ref{sec:rings}), the implementation first computes
the exact joint occupancy by dynamic programming and only then forms \(m_{x,t}\) and \(m_{y,t}\). Natural~(approx) is therefore an ablation of the mean-field occupancy assumption, not a computational benchmark. Moreover, the product of the marginals can
place mass on jointly unreachable state--time pairs, so its Fisher may have a different range from the exact Fisher. An undamped comparison would then require projection onto a common identifiable subspace; all reported Natural~(approx) runs instead use a damped solve. The \(O(TH)\) storage cost is prospective: it would apply if coordinate messages were propagated directly. The reported exact-DP implementation retains the \(O(TH^2)\) joint-occupancy cost on a side-\(H\), horizon-\(T\) grid.

\section{Algorithm}
At iteration \(k\), let \(\mu_k\) be the proposal used to sample trajectories and
let \(\nu_k\) be the training law in the TB objective~\eqref{eq:tb_loss_training}, both
held fixed during the update. Assume that \(\mu_k\) has support wherever either
\(\nu_k\) or the current forward law \(q_\theta\) has support. Define
\(w_i=\nu_k(\tau^{(i)})/\mu_k(\tau^{(i)})\) for the objective gradient and
\(\omega_i=q_\theta(\tau^{(i)})/\mu_k(\tau^{(i)})\) for the Fisher of \(q_\theta\).
They coincide when \(\nu_k=q_\theta\); in the on-policy case
\(\mu_k=\nu_k=q_\theta\), both equal one. The resulting practical scheme is summarised
in Algorithm~\ref{alg:natgrad_gflownet}.

\begin{algorithm}[H]
\caption{Damped natural-gradient GFlowNet update}
\label{alg:natgrad_gflownet}
\small
\begin{algorithmic}[1]
\Require initial parameters \((\theta,\phi,\xi)\); damping \(\lambda>0\); step sizes
\(\eta_\theta,\eta_\phi,\eta_\xi\)
\Require proposal and reference rules \((\mu_k,\nu_k)\); Fisher mode
\(\mathsf M\in\{\mathrm{exact},\mathrm{sampled},\mathrm{structured}\}\)
\For{\(k=0,1,\dots,K-1\)}
    \State Sample trajectories \(\tau^{(1)},\dots,\tau^{(N)}\sim \mu_k\)
    \State Compute residuals \(\delta_i=\delta_{\theta,\phi,\xi}(\tau^{(i)})\) and
    scores \(z_i=\nabla_\theta\log q_\theta(\tau^{(i)})\)
    \State Compute weights
    \(w_i=\nu_k(\tau^{(i)})/\mu_k(\tau^{(i)})\),
    \(\omega_i=q_\theta(\tau^{(i)})/\mu_k(\tau^{(i)})\)
    \State Form the objective-gradient estimators
    \[
    \widehat h_\theta=\frac1N\sum_{i=1}^Nw_i\ell'(\delta_i)z_i,
    \qquad
    \widehat g_\psi=\frac1N\sum_{i=1}^N
    w_i\ell'(\delta_i)\nabla_\psi\delta_i,
    \quad \psi\in\{\phi,\xi\}
    \]
    \If{\(\mathsf M=\mathrm{exact}\)}
        \State \(\widehat F_{\mathrm{use}}\gets F(\theta)\) from exact occupancies
    \ElsIf{\(\mathsf M=\mathrm{sampled}\)}
        \State \(\displaystyle
        \widehat F_{\mathrm{use}}\gets\frac1N\sum_{i=1}^N\omega_i
        \sum_{t\ge0}u_t(\theta;\tau^{(i)})u_t(\theta;\tau^{(i)})^\top\)
    \Else
        \State Set \(\widehat F_{\mathrm{use}}\gets\widetilde F(\theta)\), a
        positive-semidefinite structured surrogate for \(F(\theta)\)
    \EndIf
    \If{\(\widehat F_{\mathrm{use}}\) is state-block diagonal}
        \State Initialise \(\Delta\theta\gets0\)
        \For{each state block \(s\)}
            \State \(\widehat F_s\gets(\widehat F_{\mathrm{use}})_{ss}\)
            \State Solve
            \((\widehat F_s+\lambda I)(\Delta\theta)_s
            =-\eta_\theta(\widehat h_\theta)_s\)
        \EndFor
    \Else
        \State Solve the damped linear system
        \[
        (\widehat F_{\mathrm{use}}+\lambda I)\Delta\theta=-\eta_\theta\,\widehat h_\theta
        \]
        \Statex \hspace{\algorithmicindent} or its structured/Fisher-vector-product analogue
    \EndIf
    \State Update forward parameters \(\theta\gets\theta+\Delta\theta\)
    \State Update backward-policy and source-flow parameters by Euclidean steps
    \[
    \phi\gets \phi-\eta_\phi\,\widehat g_\phi,
    \qquad
    \xi\gets \xi-\eta_\xi\,\widehat g_\xi
    \]
\EndFor
\end{algorithmic}
\end{algorithm}

\section{Additional experiments and plots}
% ============================================================

\subsection{Lazy Random Walk Benchmarks}
\label{app:lrw_benchmarks}
% ============================================================
We evaluate six trajectory-balance (TB) training strategies on standard benchmark targets
motivated by the \emph{Lazy Random Walk} (LRW) experiments of \cite{dallantonia2026avoid}. The LRW places a particle on a bounded $(2m{+}1)\times (2m{+}1)$ integer grid ($m=8$, so a $17\times17$ grid) and runs fixed-length trajectories of $T=20$ steps from the origin; actions are right, left, up, down, or stay, with boundary masking. Each run uses $15{,}000$ gradient steps with batch size 64.

We report (i)~the \emph{total variation} distance
$\text{TV}(\hat{p}, \pi) = \tfrac{1}{2}\|\hat{p} - \pi\|_1$
between the exact terminal distribution $\hat{p}$ and the normalised target $\pi$;
(ii)~the number of \emph{modes found} (grid cells assigned $\geq 10\%$ of their true mass under $\pi$, evaluated over the full grid via dynamic programming);
and (iii)~the \emph{ELBO} $\mathbb{E}_{\hat{p}}[\log R]$, the expected log-reward under the learned distribution, together with the \emph{reward gap} $\mathbb{E}_{\pi}[\log R] - \mathbb{E}_{\hat{p}}[\log R]
\geq 0$ (zero iff $\hat{p}=\pi$) and the
KL divergence $\mathrm{KL}(\hat{p}\|\pi)$. We compare Euclidean TB against three Fisher-Rao variants: sampled natural-gradient with batch-estimated occupancies, exact natural-gradient with DP-computed occupancies, and approximate natural-gradient with a factorised mean-field Fisher
\(\tilde m(x,y,t)=m_x(x,t)m_y(y,t)\). We also benchmark against Sinkhorn-Wasserstein-corrected TB and ACE/DTB \citep{dallantonia2026avoid}.

\subsubsection{8-Gaussians Target}

The target $\pi$ is a mixture of eight equally-spaced isotropic Gaussians
arranged on a circle of radius~$5$ grid units, truncated and normalised to the
$17\times17$ grid (Fig.~\ref{fig:8g_overview}d).
The mixture has $29$ modes at threshold $10\%$.

\begin{figure}[!ht]
  \centering
  \includegraphics[width=\textwidth]{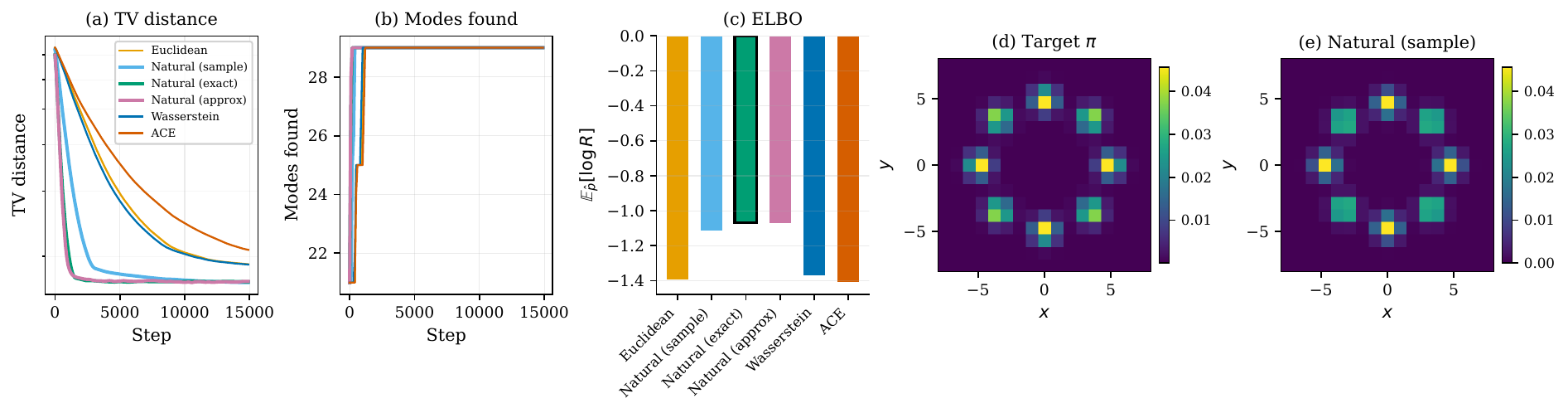}
  \caption{8-Gaussians benchmark.
    (a)~TV distance over 15\,000 training steps (lower is better).
    (b)~Modes found over training.
    (c)~ELBO $\mathbb{E}_{\hat{p}}[\log R]$ at convergence (higher is better).
    (d)~Target distribution $\pi$.
    (e)~Learned terminal distribution of the best method (Natural, sample).
    }
  \label{fig:8g_overview}
\end{figure}

Fig.~\ref{fig:8g_overview} summarises convergence and terminal fit: the
natural-gradient variants reduce TV faster (a), all methods recover all
\(29\) modes, so mode coverage is not the differentiating metric here (b), and the natural-gradient variants attain higher ELBOs (c);
panels (d--e) compare the target with the learned distribution of Natural
(sample). 

Table~\ref{tab:8g_ext} reports the complete final metrics after
\(15{,}000\) steps. The three natural-gradient variants (sample, exact, approx) achieve
TV~$\approx 0.170$--$0.171$, a relative improvement of
$\sim$$10\%$ over Euclidean ($0.190$) and Wasserstein ($0.190$),
and $\sim$$18\%$ over ACE ($0.208$).
The ELBO gap is more pronounced: natural-gradient methods reach
$\mathbb{E}_{\hat{p}}[\log R]\approx{-1.07}$ to ${-1.12}$,
versus $\approx{-1.37}$ to ${-1.41}$ for the baselines.
The reward gap for Natural (exact) is only $0.026$ nats, compared to
$0.35$ nats for Euclidean, confirming that Riemannian preconditioning
yields a tighter approximation to $\pi$.
The factorised approximation (Natural approx) matches the
performance of the more expensive exact and sampled Fishers at lower
computational cost, as it avoids a full DP sweep or matrix inversion.
% ---- Extended tables (TV, Modes, ELBO, Gap, KL, JSD) ----
\begin{table}[!ht]
  \centering
  \caption{\textbf{8 Gaussians: extended final metrics} after 15\,000 steps. Modes out of 29. ELBO$=\mathbb{E}_{\hat p}[\log R]$; Gap$=\mathbb{E}_\pi[\log R]-\text{ELBO}$ ($\downarrow$ = better). Best per column in \textbf{bold}.}
  \label{tab:8g_ext}
  \begin{tabular}{lcccccc}
    \toprule
    Method & TV\,$(\downarrow)$ & Modes\,$(\uparrow)$ & ELBO\,$(\uparrow)$ & Gap\,$(\downarrow)$ & KL$(\hat p\|\pi)\,(\downarrow)$ & JSD\,$(\downarrow)$ \\
    \midrule
    Euclidean & 0.1902 & \textbf{29} & -1.3941 & 0.3501 & 0.1713 & 0.0377 \\
    Natural (sample) & \textbf{0.1699} & \textbf{29} & -1.1158 & 0.0717 & \textbf{0.1029} & \textbf{0.0255} \\
    Natural (exact) & 0.1706 & \textbf{29} & \textbf{-1.0699} & \textbf{0.0259} & 0.1055 & 0.0268 \\
    Natural (approx) & 0.1708 & \textbf{29} & -1.0740 & 0.0300 & 0.1062 & 0.0269 \\
    Wasserstein & 0.1898 & \textbf{29} & -1.3723 & 0.3283 & 0.1666 & 0.0372 \\
    ACE & 0.2078 & \textbf{29} & -1.4074 & 0.3634 & 0.1766 & 0.0412 \\
    \bottomrule
  \end{tabular}
\end{table}

\subsection{Rings Target}
\label{sec:rings}
% ============================================================

The rings target is a mixture of three concentric annuli on the
$17\times17$ grid (Fig.~\ref{fig:rings_overview}d).
The continuous-mass support of each ring is mapped to the nearest
grid cells; the total number of modes (grid cells receiving
$\geq 10\%$ of their assigned mass) is 36. Fig.~\ref{fig:rings_overview} summarises convergence and the best learned
distribution. Fig.~\ref{fig:rings_dists} compares the terminal
distributions across methods, while Fig.~\ref{fig:rings_bars} and
Table~\ref{tab:rings_ext} report the final metrics.

\begin{figure}[!ht]
  \centering
  \includegraphics[width=\textwidth]{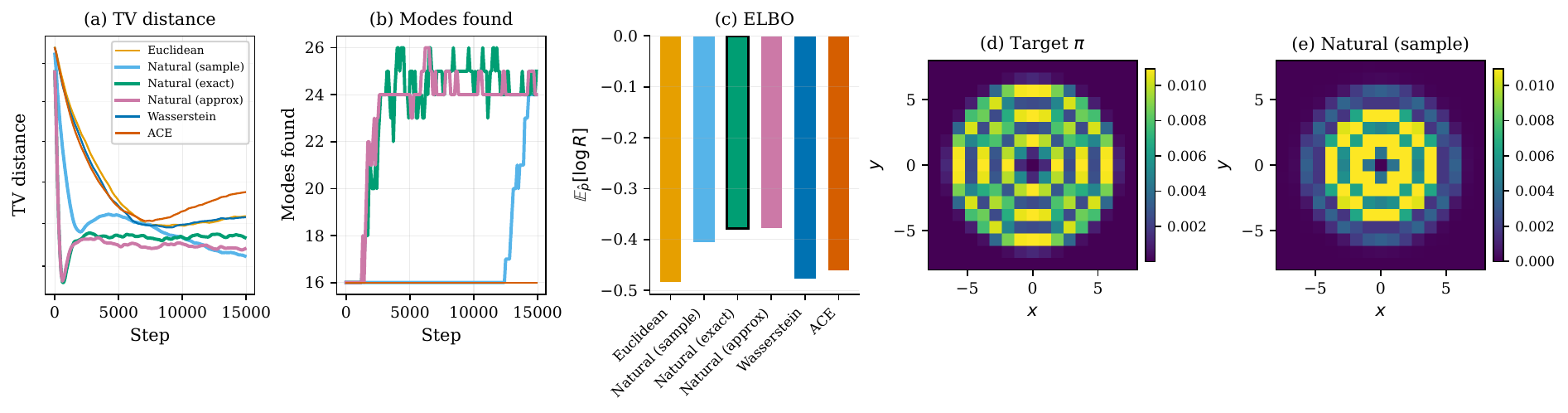}
  \caption{%
    \textbf{Rings benchmark.}
    (a)~TV distance; (b)~modes found; (c)~ELBO at convergence;
    (d)~target $\pi$; (e)~learned distribution of the best method
    (Natural, exact).
    The rings target is harder: natural-gradient methods recover
    significantly more modes and achieve higher ELBO than Euclidean,
    Wasserstein, and ACE.
  }
  \label{fig:rings_overview}
\end{figure}

\begin{figure}[!ht]
  \centering
  \includegraphics[width=\textwidth]{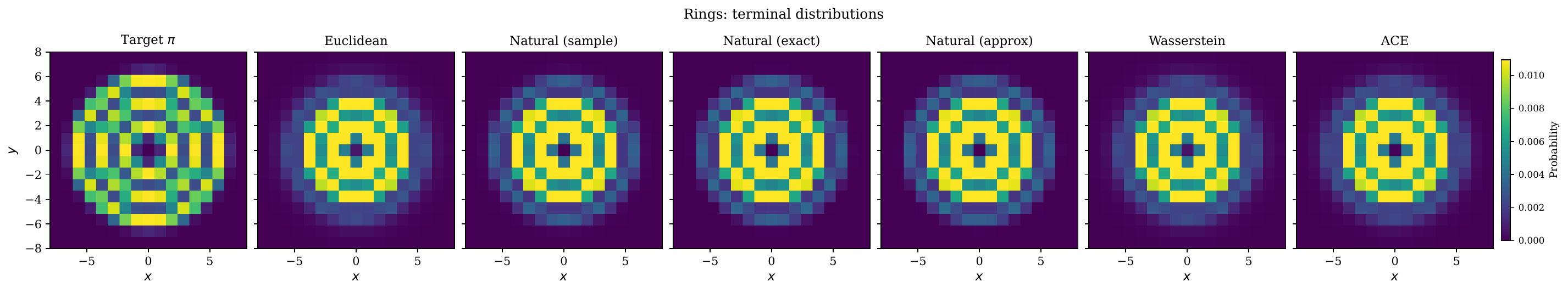}
  \caption{%
    \textbf{Rings: learned terminal distributions.}
    Euclidean, Wasserstein, and ACE collapse onto 16 of the 36 modes.
    All three natural-gradient variants recover 24-25 modes, indicating
    that Riemannian preconditioning is critical for exploration of the
    ring-shaped geometry.
  }
  \label{fig:rings_dists}
\end{figure}

\begin{figure}[!ht]
  \centering
  \includegraphics[width=\textwidth]{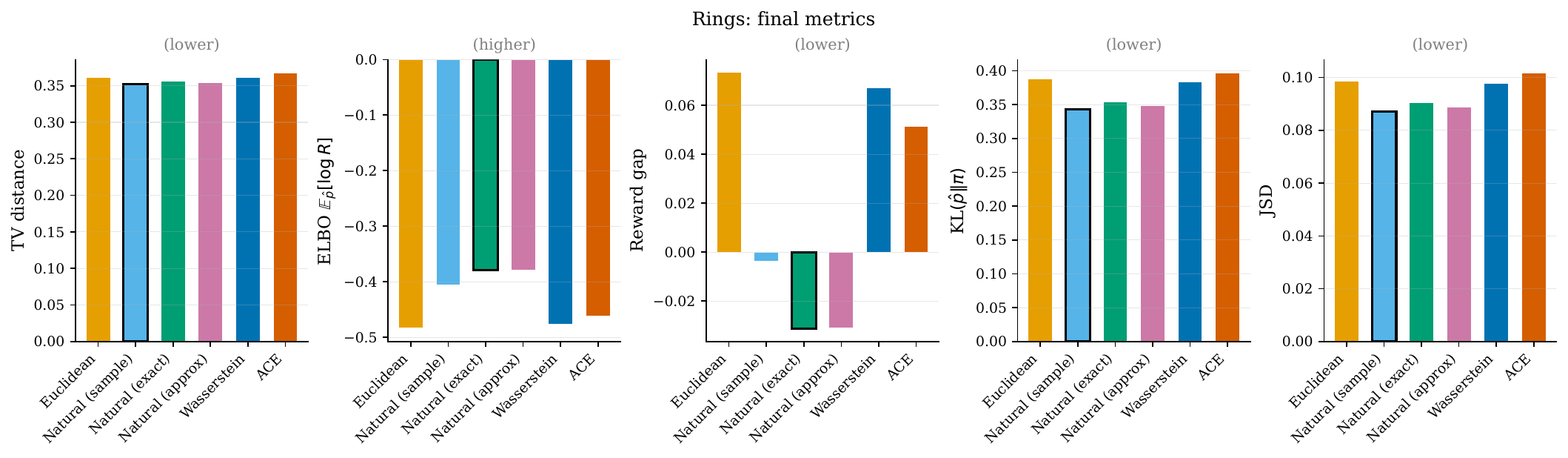}
  \caption{%
    \textbf{Rings: final-metric bar charts.}
    Natural-gradient methods recover 24-25 modes versus 16 for all baselines,
    and achieve meaningfully higher ELBO and lower KL divergence.
  }
  \label{fig:rings_bars}
\end{figure}

\paragraph{Experimental setup.}
Table~\ref{tab:rings_ext} shows that the rings benchmark is considerably
harder: Euclidean, Wasserstein, and ACE all converge to only 16 modes
($44\%$ of 36).
Natural (sample) and Natural (approx) both find 24 modes; Natural (exact)
recovers 25 modes, the best result across all methods.
TV differences are moderate in absolute terms ($0.352$-$0.367$) because the
target is nearly uniform over the ring cells; the mode count and ELBO are the
more informative signals.
Natural-gradient methods achieve a \emph{negative} reward gap
($\approx -0.03$ nats for exact and approx), reflecting that the
learned distribution over-weights high-reward ring cells relative to $\pi$
while still missing some modes.
The factorised natural-gradient approximation again matches the
full-Fisher variants at a fraction of the per-step cost.

\paragraph{Limitations} The rings evaluation is not intended as a replication of the
ACE benchmark~\cite{dallantonia2026avoid}: the original work targets a different regime
in which modes are spatially separated by wide, low-reward gaps, and the divergent explorer gains its advantage precisely by actively avoiding already-
visited high-reward regions to cross those barriers. Our rings target, by contrast, places three concentric annuli with narrow
inter-ring gaps; this is a setting that favours methods with strong geometric inductive bias (Riemannian preconditioning along the ring curvature) rather
than long-range divergent exploration, and is the regime in which our natural-
gradient variants excel.
Because we did not locate a public ACE implementation or a complete LRW
configuration, our ACE results should be viewed as an indicative baseline
rather than an exact reproduction. A more conclusive comparison would require
running ACE on rings with substantially larger inter-mode gaps, where its
exploration mechanism is most effective.

\begin{table}[!ht]
  \centering
  \caption{\textbf{Rings: extended final metrics} after 15\,000 steps. Modes out of 36. ELBO$=\mathbb{E}_{\hat p}[\log R]$; Gap$=\mathbb{E}_\pi[\log R]-\text{ELBO}$ ($\downarrow$ = better). Best per column in \textbf{bold}.}
  \label{tab:rings_ext}
  \begin{tabular}{lcccccc}
    \toprule
    Method & TV\,$(\downarrow)$ & Modes\,$(\uparrow)$ & ELBO\,$(\uparrow)$ & Gap\,$(\downarrow)$ & KL$(\hat p\|\pi)\,(\downarrow)$ & JSD\,$(\downarrow)$ \\
    \midrule
    Euclidean & 0.3616 & 16 & -0.4835 & 0.0734 & 0.3875 & 0.0987 \\
    Natural (sample) & \textbf{0.3521} & 24 & -0.4063 & -0.0038 & \textbf{0.3435} & \textbf{0.0872} \\
    Natural (exact) & 0.3565 & \textbf{25} & \textbf{-0.3788} & \textbf{-0.0313} & 0.3544 & 0.0905 \\
    Natural (approx) & 0.3540 & 24 & -0.3791 & -0.0310 & 0.3488 & 0.0889 \\
    Wasserstein & 0.3613 & 16 & -0.4771 & 0.0670 & 0.3837 & 0.0977 \\
    ACE & 0.3674 & 16 & -0.4616 & 0.0515 & 0.3967 & 0.1017 \\
    \bottomrule
  \end{tabular}
\end{table}
\subsection{Hypergrid}
\label{app:plots_hypergrid}
\paragraph{Experimental setup.}
We use a tabular forward policy with squared TB loss. All runs use learning
rates \(\eta_F{=}0.1\) for the forward policy and
\(\eta_B{=}\eta_Z{=}0.01\) for the backward policy and \(\log Z\); batch size
\(128\) (ACE: \(64\) per model to match the total sample budget), and five
seeds. Fisher variants use damping \(\lambda{=}10^{-3}\).
\begin{figure}[ht!]
    \centering
    \includegraphics[width=0.9\linewidth]{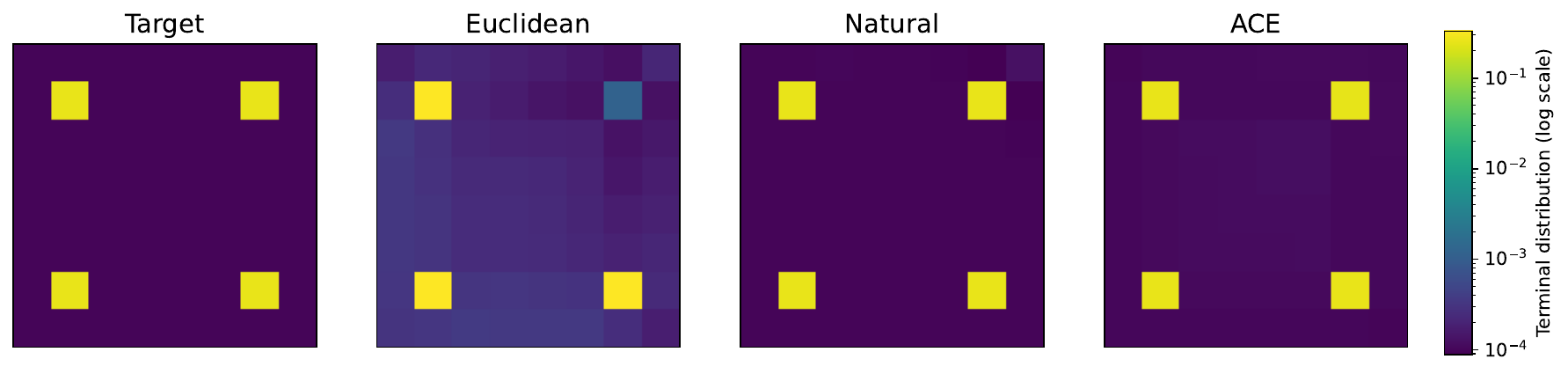}
    \caption{Terminal distribution $p_\theta(x)$ on the $8{\times}8$ grid, averaged over 5 seeds, compared to the target $\pi(x) = R(x)/Z$ (left). All panels share a logarithmic colour scale. Brighter pixels indicate higher mass; the four high-reward modes appear as a $2{\times}2$ pattern of bright cells.}
    \label{fig:vanilla_heatmap_h08}

    \vspace{0.5em}

    \includegraphics[width=0.75\linewidth]{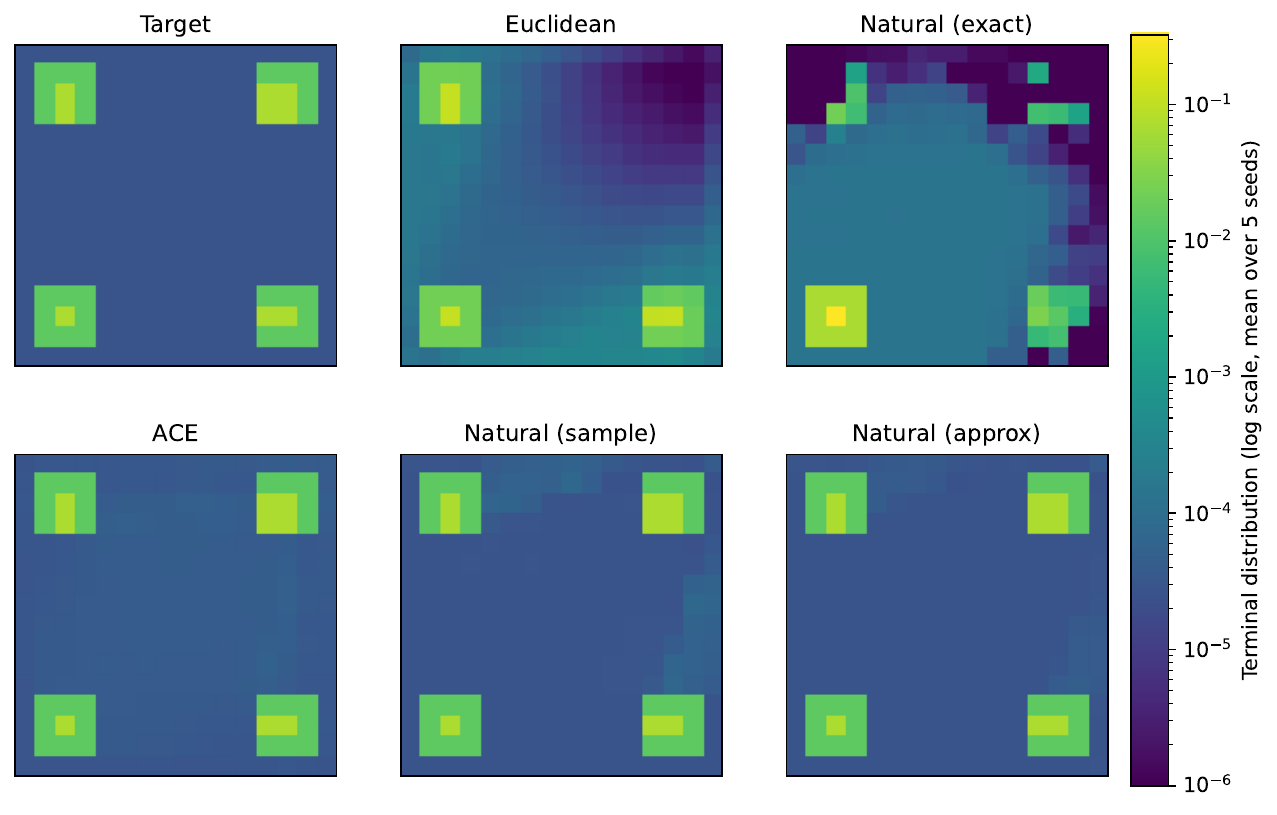}
    \caption{Terminal distribution $p_\theta(x)$ on the $16{\times}16$ grid, averaged over 5 seeds, compared to the target $\pi(x) = R(x)/Z$ (top-left). All panels share a logarithmic colour scale. Brighter pixels indicate higher mass; the four high-reward modes appear as a $2{\times}2$ pattern of bright cells.}
    \label{fig:weighted_hyper_colormaps}

\end{figure}
\subsection{Deceptive Grid World}
\label{app:deceptive_grid_world}

We consider the deceptive grid benchmark of~\cite{kim2025adaptive}, 
which is a modification of the original hypergrid task first 
introduced by~\cite{bengio2021flow}. It is a $D$-dimensional grid $\{0,\dots,H-1\}^{D}$, i.e., a hypercube containing $H^D$ cells. The GFlowNet samples trajectories from the origin that terminate at one of these cells. 

The reward of each terminal state $s=(s^{1},\ldots,s^{D})$ is defined as\footnote{We 
adopt the reward as coded in the implementation of the deceptive grid world environment available
in the public repository \url{https://github.com/alstn12088/adaptive-teacher/blob/main/grid/env.py},
which slightly differs from Eq.~(9) in~\cite{kim2025adaptive}.}
\begin{equation}
    R_{\text{deceptive}}(s) = R_0 +  R_1 - R_1 \prod_{d=1}^{D}
        \mathds{1}\!\left\{\,\abs{x_d - \tfrac{1}{2}} > \tfrac{1}{10}\right\}
      + R_2  \prod_{d=1}^{D}
        \mathds{1}\!\left\{\,\tfrac{3}{10} < \abs{x_d - \tfrac{1}{2}} < \tfrac{2}{5}\right\},
\end{equation}
where $x_d = \tfrac{s^{(d)}}{(H-1)} \in [0,1]$. We 
set \(R_0 = 10^{-5}\), \(R_1 = 10^{-1}\), and \(R_2 = 2\). 

This benchmark is commonly trained with Adam~\cite{kingma2014adam} for the forward policy, while in our runs plain SGD performs poorly, often collapsing to a restricted part of the support. Even with Adam, however, the deceptive grid world is designed so that simple exploration can fail to cover the high-reward regions reliably. Adaptive Teacher~\cite{kim2025adaptive} addresses this exploration issue by introducing an auxiliary policy, the \emph{teacher}, which guides sampling towards high-loss regions, while still using Adam for the parameter updates. Here we ask whether a complementary mechanism, based on Fisher-Rao preconditioning of the forward-policy update, can achieve comparable behaviour.

We therefore consider an Adam-style update with Fisher-Rao preconditioning for
the forward policy. The update keeps the adaptive first-moment averaging of
Adam but replaces the coordinate-wise second-moment rescaling by a Fisher-Rao
preconditioner. Adam and Fisher-Rao preconditioning are not equivalent in
general~\cite{kunstner2019limitations}: Adam rescales each coordinate of the
first-moment estimate by the inverse square root of an estimated second moment
of stochastic gradients, whereas natural-gradient methods precondition updates
using the Fisher information matrix, which defines a local geometry in
distribution space. Nevertheless, both approaches modify the geometry of the
parameter update. Our method follows the structure of \emph{AdaFisher}~\cite{gomes2025adafisher},
using Adam's first-moment recursion together with a Fisher-based
preconditioner, but adapts it to GFlowNet training. In this small grid setting,
we compute Fisher-vector products using the per-step decomposition of
Section~\ref{sec:fisher_rao_structure}, rather than using the Fisher approximation
employed in AdaFisher. This leads to the Fisher update described in Algorithm~\ref{alg:fisher_adam_gflownet}.

\begin{algorithm}[t]
\caption{Fisher-preconditioned Adam update for the forward policy}
\label{alg:fisher_adam_gflownet}
\begin{algorithmic}[1]
\Require initial parameters \((\theta,\phi,\xi)\); proposal law \(\mu\); reference law \(\nu\); damping \(\lambda>0\);
step size \(\alpha_\theta\); first-moment decay \(\beta_1\in[0,1)\);
backward / flow step sizes \(\eta_\phi,\eta_\xi\); batch size \(N\); iterations \(K\)
\State Initialise first moment \(m \gets 0\); iteration counter \(t \gets 0\)
\For{\(k=0,1,\dots,K-1\)}
      \State Sample trajectories \(\tau^{(1)},\dots,\tau^{(N)}\sim \mu_k\)
      \State Compute residuals \(\delta_i=\delta_{\theta,\phi,\xi}(\tau^{(i)})\)
           and importance weights \(\rho_i=\nu_k(\tau^{(i)})/\mu_k(\tau^{(i)})\)
    \State Form the forward-gradient estimator
    \[
    \widehat g_\theta
    \;=\;
    \frac{1}{N}\sum_{i=1}^N
       \rho_i\,\ell'(\delta_i)\,
       \nabla_\theta \log q_\theta(\tau^{(i)})
    \]
    \State \(\triangleright\) \emph{Adam first moment (with bias correction; no second moment, no \(\sqrt{\hat v}\)):}
    \[
    m \gets \beta_1\,m + (1-\beta_1)\,\widehat g_\theta,
    \qquad
    \widehat m \gets \frac{m}{1-\beta_1^{k+1}}
    \]
    \State \(\triangleright\) \emph{Apply the exact Fisher inverse to \(\widehat m\) by conjugate gradient:}
    \[
    \bigl(\widehat F(\theta) + \lambda I\bigr)\,\Delta\theta \;=\; \widehat m
    \]
    \Statex \hspace{\algorithmicindent}
        where each matrix--vector product \(\widehat F(\theta)\,u\)
        is implemented as the per-step Fisher--vector product
        of Section~\ref{sec:fisher_rao_structure}.
    \State Update forward parameters
    \[
    \theta \;\gets\; \theta - \alpha_\theta\,\Delta\theta
    \]
    \State Update backward-policy and source-flow parameters by Euclidean steps
    \[
    \phi \gets \phi - \eta_\phi\,\widehat g_\phi,
    \qquad
    \xi  \gets \xi  - \eta_\xi \,\widehat g_\xi
    \]
\EndFor
\end{algorithmic}
\end{algorithm}

We evaluate this update on two-dimensional ($D=2$) deceptive grids with
$H\in\{128,256\}$. This experiment tests whether Fisher-Rao preconditioning
can provide a competitive alternative to Adaptive
Teacher~\cite{kim2025adaptive}, which improves exploration in GFlowNet training
through an auxiliary teacher policy. Unlike Adaptive Teacher, the Fisher update does not introduce an additional model. Instead, it modifies the forward-policy parameter update using the Fisher-Rao geometry induced by the
forward policy.

\paragraph{Experimental setup.}
The forward policy is a two-hidden-layer MLP with hidden width $32$, leaky-ReLU activations, and a one-hot state encoding of width $2H$. The backward policy is fixed to be uniform, and $\log Z$ is learned. All methods are trained with the trajectory-balance loss, using batch size $16$, a replay buffer, and gradient clipping~\cite{kim2025adaptive}. For the Fisher runs (Algorithm~\ref{alg:fisher_adam_gflownet}), we use forward step size $\alpha_\theta=10^{-3}$, first-moment parameter $\beta_1=0.9$, damping $\lambda=10^{-3}$, and $20$ conjugate-gradient iterations. Fisher-vector products are computed from sampled trajectories using the per-step decomposition of Section~\ref{sec:fisher_rao_structure}. Results are averaged over five seeds.

We first examine exploration through cumulative mode discovery. Following the
benchmark convention, this is the number of distinct maximum-reward terminal
states visited during training. This count measures visitation of the
high-reward regions; distributional fit is assessed separately below. Fig.~\ref{fig:deceptive_visited} shows the matrix-free conjugate-gradient Fisher
update, labelled ``Natural (CG)'', together with Adaptive Teacher, which is the
main reference point for this experiment and uses Adam for the parameter
updates. We also include a K-FAC approximation~\cite{martens2015kfac}, which is more scalable than the
full Fisher solve and is therefore relevant for larger settings. Other
approximations, such as the diagonal Fisher approximation used in
AdaFisher~\cite{gomes2025adafisher}, could also be considered. We also report Euclidean Adam runs, both on-policy and with \(\varepsilon\)-exploration, and Euclidean SGD runs with the same two sampling choices. These additional runs are included to show separately the effect of the optimiser and of simple exploratory sampling.

Fig.~\ref{fig:deceptive_visited} shows that the exact Fisher update achieves
mode coverage close to Adaptive Teacher for a fixed training budget. This
suggests that Fisher-Rao preconditioning can provide a competitive mechanism
for improving coverage without introducing an auxiliary teacher policy. The
K-FAC approximation follows the same qualitative trend, although with weaker
coverage. The Euclidean runs illustrate the exploration difficulty discussed
in~\cite{kim2025adaptive}: simple \(\epsilon\)-exploration does not resolve the
problem by itself. The SGD variants remain far from full mode coverage,
indicating that the adaptive geometry introduced by Adam is important in this task.

Figs.~\ref{fig:deceptive_heatmaps_H128}
and~\ref{fig:deceptive_heatmaps_H256} show the learned terminal
distributions. These heatmaps provide a qualitative view of the mode-coverage
behaviour observed in Fig.~\ref{fig:deceptive_visited}. The matrix-free CG and
K-FAC runs recover the four high-reward regions and the surrounding support
reasonably well, whereas the SGD runs concentrate mass along a narrow region of
the grid.

\begin{figure}[ht!]
\centering
\includegraphics[width=0.9\linewidth]{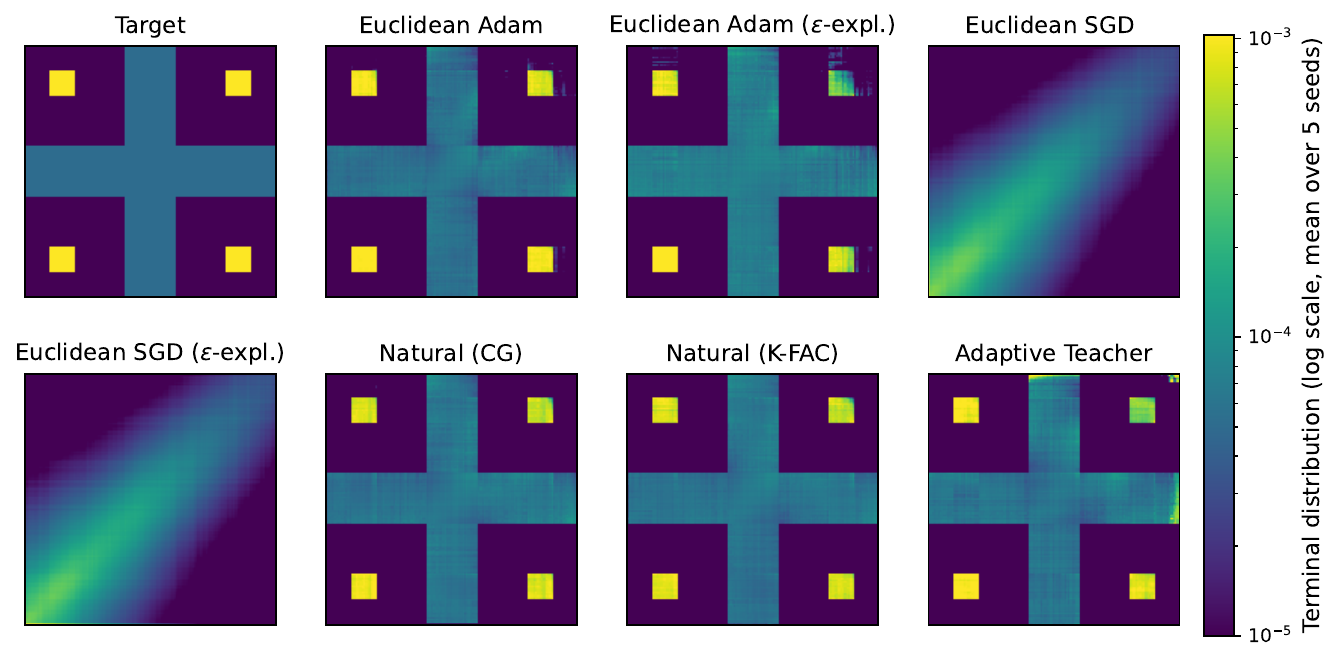}
    \caption{Terminal distribution on the deceptive grid (\(D=2\)) with \(H=128\) at $15{,}000$ training steps, averaged
over \(5\) seeds and shown on a logarithmic colour scale. The target
distribution is shown in the top-left panel. Brighter pixels indicate larger
terminal mass.}
    \label{fig:deceptive_heatmaps_H128}
\end{figure}

\begin{figure}[ht!]
\centering
\includegraphics[width=0.9\linewidth]{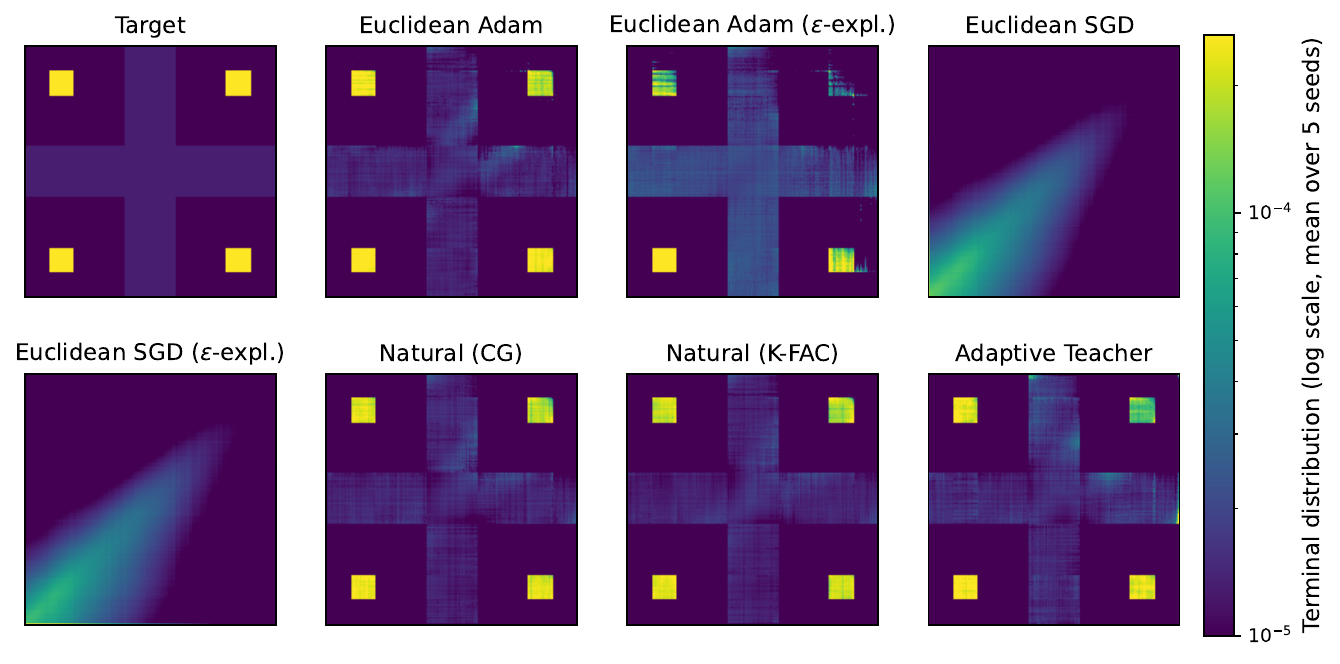}
    \caption{Terminal distribution on the deceptive grid (\(D=2\)) with \(H=256\) at $30{,}000$ training steps, averaged
over \(5\) seeds and shown on a logarithmic colour scale. The target
distribution is shown in the top-left panel. Brighter pixels indicate larger
terminal mass.}
    \label{fig:deceptive_heatmaps_H256}
\end{figure}

We next report distributional distances between the estimated learned terminal
distribution \(\hat p_\theta\) and the target distribution \(\pi(x)=R(x)/Z\).
Fig.~\ref{fig:deceptive_tv} shows the total variation distance,
\begin{equation}
    \mathrm{TV}(\hat p_\theta,\pi) =
\frac{1}{2}\sum_{x\in\mathcal X}
\abs{\hat p_\theta(x)-\pi(x)} \;.
\end{equation}
For consistency with~\cite{kim2025adaptive}, we also report the per-cell \(L_1\) distance,
\begin{equation}
    L_1(\hat p_\theta,\pi)
=
\frac{1}{|\mathcal X|}
\sum_{x\in\mathcal X}
\abs{\hat p_\theta(x)-\pi(x)} \;,
\qquad |\mathcal X|=H^D \;,
\end{equation}
shown in Fig.~\ref{fig:deceptive_l1}. Since
\begin{equation}
    L_1(\hat p_\theta,\pi)
=
\frac{2}{|\mathcal X|}\mathrm{TV}(\hat p_\theta,\pi) \;,    
\end{equation}
the two curves contain the same information up to a grid-size dependent
rescaling. We include both to make the comparison with prior results immediate.

\begin{figure}[ht!]
  \centering
  \includegraphics[width=\linewidth]{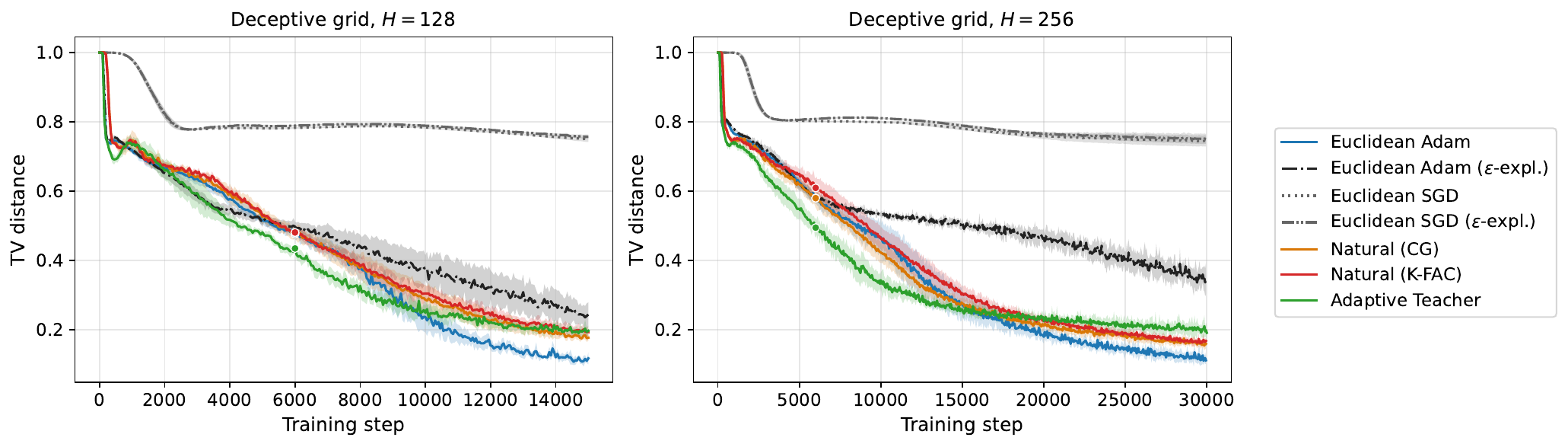}
  \caption{%
    Total variation distance on the deceptive grid
(\(D=2\)) for \(H=128\) (left) and \(H=256\) (right). Mean over \(5\)
seeds; bands are one standard deviation. Coloured dots mark the value at the
\(6{,}000\)-step budget of~\cite{kim2025adaptive} who use a width-$128$ MLP;
    we use width $32$.%
  }
  \label{fig:deceptive_tv}
\end{figure}

\begin{figure}[ht!]
  \centering
  \includegraphics[width=\linewidth]{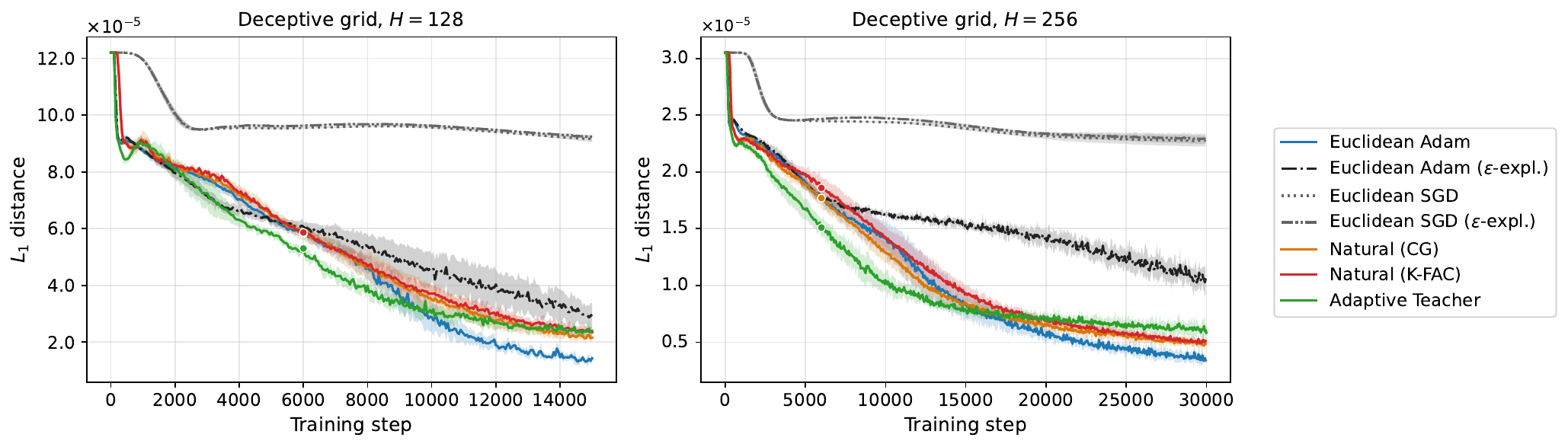}
  \caption{%
    Per-cell \(L_1\) distance on the deceptive grid
(\(D=2\)) for \(H=128\) (left) and \(H=256\) (right). Mean over \(5\)
seeds; bands are one standard deviation. Coloured dots mark the value at the
\(6{,}000\)-step budget of~\cite{kim2025adaptive} who use a width-$128$ MLP;
    we use width $32$.%
  }
  \label{fig:deceptive_l1}
\end{figure}

Table~\ref{tab:deceptive_half_H128} reports the complete \(H=128\) results at
its full \(15{,}000\)-step budget. For \(H=256\),
Tables~\ref{tab:deceptive_half_H256} and
\ref{tab:deceptive_full_H256} report results at \(15{,}000\) and \(30{,}000\)
steps, respectively. At \(H=256\), Adaptive Teacher achieves the strongest
mode coverage after \(15{,}000\) steps, while the Fisher variants remain
competitive and discover substantially more modes than Adam. Adam approaches
comparable coverage only near the full budget. These coverage gains, however,
do not translate into uniform improvements in TV, KL, or ELBO.

\begin{table}[t]
\centering
\caption{Deceptive grid, $d=2$, $H=128$, full training budget (15k steps), five seeds. Mean $\pm$ std; best per column in \textbf{bold}.}
\label{tab:deceptive_half_H128}
\resizebox{\linewidth}{!}{
\begin{tabular}{lcccccc}
\toprule
Method & TV ($\downarrow$) & \# modes ($\uparrow$) & ELBO ($\uparrow$) & Gap ($\downarrow$) & KL($\hat p \| \pi$) ($\downarrow$) & JSD ($\downarrow$) \\
\midrule
Euclidean (Adam) & \textbf{0.1171 $\pm$ 0.0170} & 676 / 676 $\pm$ 0.4 & \textbf{-0.4542 $\pm$ 0.0576} & \textbf{0.2281 $\pm$ 0.0576} & \textbf{0.1055 $\pm$ 0.0265} & \textbf{0.0154 $\pm$ 0.0036} \\
Euclidean (Adam, $\epsilon$-expl.) & 0.2433 $\pm$ 0.0352 & 610 / 676 $\pm$ 21 & -0.8249 $\pm$ 0.0747 & 0.5988 $\pm$ 0.0747 & 0.3497 $\pm$ 0.0443 & 0.0653 $\pm$ 0.0146 \\
Euclidean (SGD) & 0.7490 $\pm$ 0.0071 & 352 / 676 $\pm$ 17 & -8.1380 $\pm$ 0.0594 & 7.9119 $\pm$ 0.0594 & 7.1384 $\pm$ 0.0806 & 0.4390 $\pm$ 0.0053 \\
Euclidean (SGD, $\epsilon$-expl.) & 0.7567 $\pm$ 0.0065 & 345 / 676 $\pm$ 7.8 & -8.1775 $\pm$ 0.0638 & 7.9513 $\pm$ 0.0638 & 7.3366 $\pm$ 0.1086 & 0.4433 $\pm$ 0.0054 \\
Natural CG & 0.1766 $\pm$ 0.0091 & 666 / 676 $\pm$ 11 & -1.0575 $\pm$ 0.0258 & 0.8314 $\pm$ 0.0258 & 0.7606 $\pm$ 0.0669 & 0.0457 $\pm$ 0.0041 \\
Natural K-FAC & 0.1927 $\pm$ 0.0137 & 673 / 676 $\pm$ 6.3 & -1.3401 $\pm$ 0.0840 & 1.1140 $\pm$ 0.0840 & 1.1574 $\pm$ 0.1427 & 0.0545 $\pm$ 0.0060 \\
Adaptive Teacher~\citep{kim2025adaptive} & 0.1965 $\pm$ 0.0151 & \textbf{676 / 676 $\pm$ 0.0} & -0.7283 $\pm$ 0.0896 & 0.5021 $\pm$ 0.0896 & 0.3258 $\pm$ 0.0584 & 0.0418 $\pm$ 0.0038 \\
\bottomrule
\end{tabular}
}
\end{table}

\begin{table}[t]
\centering
\caption{Deceptive grid, $d=2$, $H=256$, half training budget (15k steps), five seeds. Mean $\pm$ std; best per column in \textbf{bold}.}
\label{tab:deceptive_half_H256}
\resizebox{\linewidth}{!}{
\begin{tabular}{lcccccc}
\toprule
Method & TV ($\downarrow$) & \# modes ($\uparrow$) & ELBO ($\uparrow$) & Gap ($\downarrow$) & KL($\hat p \| \pi$) ($\downarrow$) & JSD ($\downarrow$) \\
\midrule
Euclidean (Adam) & 0.2721 $\pm$ 0.0356 & 2031 / 2601 $\pm$ 251 & \textbf{-0.9361 $\pm$ 0.1245} & \textbf{0.6851 $\pm$ 0.1245} & \textbf{0.4169 $\pm$ 0.0867} & 0.0836 $\pm$ 0.0219 \\
Euclidean (Adam, $\epsilon$-expl.) & 0.5069 $\pm$ 0.0154 & 878 / 2601 $\pm$ 65 & -1.4407 $\pm$ 0.0905 & 1.1896 $\pm$ 0.0905 & 0.8786 $\pm$ 0.0234 & 0.2154 $\pm$ 0.0111 \\
Euclidean (SGD) & 0.7808 $\pm$ 0.0052 & 1150 / 2601 $\pm$ 9.6 & -8.3038 $\pm$ 0.0310 & 8.0528 $\pm$ 0.0310 & 7.7384 $\pm$ 0.0848 & 0.4653 $\pm$ 0.0041 \\
Euclidean (SGD, $\epsilon$-expl.) & 0.7889 $\pm$ 0.0053 & 1012 / 2601 $\pm$ 22 & -8.3647 $\pm$ 0.0342 & 8.1136 $\pm$ 0.0342 & 8.0864 $\pm$ 0.0996 & 0.4680 $\pm$ 0.0042 \\
Natural CG & 0.2753 $\pm$ 0.0185 & 2405 / 2601 $\pm$ 76 & -1.3872 $\pm$ 0.0896 & 1.1361 $\pm$ 0.0896 & 1.0838 $\pm$ 0.0934 & 0.0848 $\pm$ 0.0090 \\
Natural K-FAC & 0.3052 $\pm$ 0.0246 & 2406 / 2601 $\pm$ 81 & -1.6213 $\pm$ 0.0653 & 1.3702 $\pm$ 0.0653 & 1.4197 $\pm$ 0.1557 & 0.1020 $\pm$ 0.0087 \\
Adaptive Teacher~\citep{kim2025adaptive} & \textbf{0.2574 $\pm$ 0.0133} & \textbf{2572 / 2601 $\pm$ 7.7} & -1.0089 $\pm$ 0.0587 & 0.7579 $\pm$ 0.0587 & 0.6152 $\pm$ 0.0759 & \textbf{0.0705 $\pm$ 0.0041} \\
\bottomrule
\end{tabular}
}
\end{table}

\begin{table}[t]
\centering
\caption{Deceptive grid, $D=2$, $H=256$, full training budget (30k steps), five seeds. Mean $\pm$ std; best per column in \textbf{bold}.}
\label{tab:deceptive_full_H256}
\resizebox{\linewidth}{!}{
\begin{tabular}{lcccccc}
\toprule
Method & TV ($\downarrow$) & \# modes ($\uparrow$) & ELBO ($\uparrow$) & Gap ($\downarrow$) & KL($\hat p \| \pi$) ($\downarrow$) & JSD ($\downarrow$) \\
\midrule
Euclidean (Adam) & \textbf{0.1061 $\pm$ 0.0073} & 2524 / 2601 $\pm$ 66 & \textbf{-0.4702 $\pm$ 0.0388} & \textbf{0.2191 $\pm$ 0.0388} & \textbf{0.1207 $\pm$ 0.0210} & \textbf{0.0180 $\pm$ 0.0055} \\
Euclidean (Adam, $\epsilon$-expl.) & 0.3371 $\pm$ 0.0436 & 1672 / 2601 $\pm$ 119 & -1.0285 $\pm$ 0.1038 & 0.7774 $\pm$ 0.1038 & 0.5559 $\pm$ 0.0588 & 0.1198 $\pm$ 0.0193 \\
Euclidean (SGD) & 0.7440 $\pm$ 0.0154 & 1394 / 2601 $\pm$ 21 & -8.2004 $\pm$ 0.0388 & 7.9493 $\pm$ 0.0388 & 7.1046 $\pm$ 0.1205 & 0.4394 $\pm$ 0.0098 \\
Euclidean (SGD, $\epsilon$-expl.) & 0.7507 $\pm$ 0.0145 & 1218 / 2601 $\pm$ 19 & -8.1839 $\pm$ 0.0457 & 7.9328 $\pm$ 0.0457 & 7.2178 $\pm$ 0.1104 & 0.4421 $\pm$ 0.0086 \\
Natural CG & 0.1566 $\pm$ 0.0164 & 2564 / 2601 $\pm$ 19 & -1.0435 $\pm$ 0.0922 & 0.7924 $\pm$ 0.0922 & 0.9458 $\pm$ 0.1021 & 0.0436 $\pm$ 0.0049 \\
Natural K-FAC & 0.1706 $\pm$ 0.0108 & 2569 / 2601 $\pm$ 24 & -1.3890 $\pm$ 0.1536 & 1.1380 $\pm$ 0.1536 & 1.4652 $\pm$ 0.2858 & 0.0515 $\pm$ 0.0057 \\
Adaptive Teacher~\citep{kim2025adaptive} & 0.1911 $\pm$ 0.0205 & \textbf{2601 / 2601 $\pm$ 0.0} & -0.7969 $\pm$ 0.0435 & 0.5458 $\pm$ 0.0435 & 0.4614 $\pm$ 0.0545 & 0.0439 $\pm$ 0.0061 \\
\bottomrule
\end{tabular}
}
\end{table}

\section{Bayesian Structure Learning with Natural-Gradient GFlowNets}
\label{sec:dag_experiments}
\label{app:sachs_dag_details}

This appendix describes the final neural-policy experiments used for the
synthetic and Sachs studies. Both studies use the same order-free DAG construction, Linear Transformer policy, replay-based
trajectory-balance objective, and paired Adam/Fisher continuation protocol.

\subsection{Experiment setup and target}

Starting from the empty graph, a trajectory repeatedly adds a directed edge
that does not create a cycle, or selects STOP. Invalid additions, including
edges already present and cycle-forming edges, are masked. For a terminal graph
with \(k\) edges, the backward policy removes one present edge uniformly at
each step, so the backward probability of a particular reverse trajectory is

\[
  \log p_B(\tau)=-\log(k!).
\]

In the synthetic study, for each data seed \(0,1,2\), we sample a \(d=10\) upper-triangular Erd\H{o}s--R\'enyi adjacency matrix with

\[
  \Pr(i\to j)=\frac{2}{d-1}, \qquad i<j,
\]

and then randomly permute the node labels. Thus ``ER-2'' denotes expected
total degree \(2\), or approximately \(d\) edges per graph, rather than expected
in-degree \(2\). The three realised generating DAGs contain \(7\), \(8\), and
\(11\) edges. Each nonzero coefficient is sampled uniformly from
\([-2,-0.5]\cup[0.5,2]\). Writing \(W\in\mathbb R^{d\times d}\) for the
weighted adjacency matrix, with \(W_{ij}\) the coefficient of \(i\to j\), the
data matrix \(X\in\mathbb R^{N\times d}\), for \(N=100\), is generated as

\[
  X=E(I-W)^{-1},
  \qquad E_{ni}\overset{\mathrm{iid}}{\sim}\mathcal N(0,1).
\]

Here \(I\) is the \(d\times d\) identity, and \(n\) and \(i\) index observations
and variables, respectively.

The GFlowNet is not given the generating topological order. Data seeds
\(0,1,2\) are paired with policy-initialisation seeds \(1000,1001,1002\),
respectively.

The real-data study uses the observational subset of the Sachs
protein-signalling data \cite{sachs2005causal}, comprising \(N=853\)
measurements of \(d=11\) phosphoproteins. Each measurement is transformed by
the natural logarithm and each column is standardised. The 17-edge consensus
network used by \cite{deleu2022bayesian} is used only for structural diagnostics. Because the observations are non-interventional and the BGe score
is score-equivalent, this reference is not treated as an identifiable causal
ground truth.

For a candidate DAG \(G\) on \(d\) variables and an observed data matrix
\(\mathcal D\in\mathbb R^{N\times d}\), let \(E(G)\) denote its directed-edge
set and \(\mathrm{Pa}_j(G)\) the parent set of node \(j\). Both studies define
the terminal reward, up to a \(G\)-independent positive constant, by

\begin{equation}
  R(G)
  =
  \exp\!\left\{
    N^{-1/2}\mathrm{BGe}(G\mid\mathcal D)
    -\lambda_{\mathrm{sp}}|E(G)|
  \right\},
  \qquad \lambda_{\mathrm{sp}}=0.5,
  \label{eq:dag_tempered_target}
\end{equation}

where \(N^{-1/2}\) tempers the score and \(\lambda_{\mathrm{sp}}\) controls
sparsity. Any \(G\)-independent reward scaling can be absorbed into the GFlowNet
normaliser \(Z\). The log BGe score decomposes as

\begin{equation}
  \mathrm{BGe}(G\mid\mathcal D)
  =
  \sum_{j=1}^{d}
  s_j\!\left(\mathrm{Pa}_j(G)\mid\mathcal D\right).
  \label{eq:bge_decomposition}
\end{equation}

Here \(s_j(\mathrm{Pa}_j(G)\mid\mathcal D)\) is the local BGe log-score for node
\(j\) and its parent set.

We use the score-equivalent normal-Wishart parameterisation of
\cite{deleu2022bayesian}: prior mean \(\mu_0=0_d\), mean-precision
hyperparameter \(\alpha_\mu=1\), Wishart degrees of freedom \(\alpha_w=d+2\),
and Wishart scale matrix

\[
  T
  =
  \frac{\alpha_\mu(\alpha_w-d-1)}{\alpha_\mu+1}I_d
  =
  \tfrac12 I_d.
\]

Here \(0_d\) and \(I_d\) are the zero vector and identity matrix in dimension
\(d\), respectively.

The BGe component uses a uniform graph prior, after which the explicit edge
penalty in \eqref{eq:dag_tempered_target} supplies a modular sparsity prior.
The reported target is therefore a tempered, sparsity-regularised BGe law, not the untempered BGe posterior.

\subsection{Neural policy and training objective}

The forward policy represents every ordered pair \(i\ne j\) by a token formed
from learned source-node, destination-node, and edge-presence embeddings. Two
linear-attention layers with model width \(64\), four heads, and feed-forward
width \(128\) share information across all \(d(d-1)\) edge tokens. A shared
edge head produces one logit per directed edge, while a mean-pooled STOP head
produces the termination logit. The resulting logits are masked to the valid
action set before the categorical softmax.

For a replay trajectory \(\tau\) terminating at \(G_\tau\), the implementation uses the residual
\begin{equation}
  \delta_\theta(\tau)
  =
  \log Z
  +\sum_t\log\pi_\theta(a_t\mid G_t)
  -N^{-1/2}\mathrm{BGe}(G_\tau\mid\mathcal D)
  +0.5|E(G_\tau)|
  -\log p_B(\tau),
  \label{eq:dag_tb_residual}
\end{equation}
and minimises the mean Huber loss in \(\delta_\theta\), with threshold \(30\).
The forward-policy gradient is clipped to Euclidean norm \(1\) before each update. The scalar \(\log Z\) is updated for every method by SGD with learning
rate \(10^{-3}\); the backward policy is fixed.

At every update, \(64\) new trajectories are sampled; at each construction
state, exploration mixes the current policy with the uniform distribution over
valid actions with weight \(0.2\). These trajectories enter a circular replay
buffer of capacity \(2{,}000\), from which \(64\) complete trajectories are
sampled uniformly without replacement for the gradient update. Policy weight decay is \(10^{-5}\).

\subsection{Sampled neural Fisher and paired updates}

Let \(J_\theta(G)\) be the Jacobian of the valid-action logits at partial graph
\(G\), and define

\[
  C_G
  =
  \operatorname{Diag}(\pi_G)-\pi_G\pi_G^\top,
  \qquad
  \pi_G=\pi_\theta(\cdot\mid G).
\]

All partial graphs are reconstructed from the \(64\) replay trajectories in the
gradient batch. The implementation selects \(M=128\) states, retaining at least one state from each represented graph depth when the budget permits and filling the remaining budget uniformly. It then uses

\begin{equation}
  \widehat F_M
  =
  \frac1M
  \sum_{G\in\mathcal S_M}
  J_\theta(G)^\top C_G J_\theta(G).
  \label{eq:dag_sampled_fisher}
\end{equation}

The expectation over valid actions is therefore exact conditional on each
selected state, while the visited-state distribution is sampled. Equation
\eqref{eq:dag_sampled_fisher} is state-normalised rather than the unnormalised
sum over trajectory steps; its scale is handled by the learning rate and
damping. In the notation of Theorem~\ref{thm:structure_dependent_approx}, it is the surrogate \(\widetilde F\).

The nodewise factorisation of the BGe reward does not make \eqref{eq:dag_sampled_fisher} block diagonal. Acyclicity and STOP couple the
valid actions, while the Linear Transformer shares parameters across edge
tokens and construction depths. The final experiment therefore does not impose
the edge-diagonal or factorised Bernoulli Fisher. Parent-set factorisation motivates possible structured surrogates, but the reported update retains neural cross-parameter coupling through matrix-free Fisher-vector products.

Given the Huber-TB gradient \(\widehat h\) from the same replay batch, the Fisher
branch approximately solves

\[
  (\widehat F_M+\lambda I)v=\widehat h
\]

by preconditioned conjugate gradients (PCG), without forming a dense matrix.
Damping starts at \(\lambda=0.03\); if PCG misses its relative-residual tolerance
\(0.05\), one retry may multiply the damping by \(2\), up to \(0.3\). Damping
decays by a factor \(0.9\) after a sufficiently fast successful solve. PCG is
warm-started from the previous direction and uses at most \(40\) iterations per
attempt.

Two Hutchinson probes estimate one mean curvature value per parameter tensor.
An exponential moving average with coefficient \(0.9\), refreshed every ten
updates, gives a block-scalar numerical preconditioner for PCG. This trace
preconditioner changes convergence of the iterative solver only: it is not the
Fisher approximation being evaluated and does not replace the non-diagonal
system in \eqref{eq:dag_sampled_fisher}. The Fisher direction uses no gradient
momentum.

The proposed Fisher step is scaled to have Euclidean norm at most \(10^{-2}\)
and predicted local KL at most \(10^{-4}\). It is then backtracked using the maximum statewise categorical KL on the \(128\) Fisher states and \(256\)
additional states selected disjointly from the same replay batch. These controls
place the empirical update in the damped local regime studied in Theorem~\ref{thm:structure_dependent_approx}. The PCG residual measures linear solve error for \(\widehat F_M\); it does not estimate the statistical surrogate error \(\rho_\lambda\).

For each seed, Adam first trains a shared checkpoint for \(4{,}000\) updates.
The policy, \(\log Z\), replay contents, NumPy and PyTorch random states, and
optimiser states are saved. Exact copies then continue for \(500\) updates as

\begin{enumerate}[label=(\roman*),nosep]
  \item \textbf{Adam}, with learning rate \(3\times10^{-4}\);
  \item \textbf{KL-controlled Adam}, which forms the same Adam proposal and
        applies the Fisher branch's step-norm and empirical statewise-KL
        safeguards; and
  \item \textbf{Damped Fisher}, implemented with PCG and direction scale
        \(5\times10^{-2}\).
\end{enumerate}

The two Adam branches restore the warm-up Adam state. Damped Fisher starts its
iterative-solver state afresh from the same policy and replay checkpoint; all
three branches restore the \(\log Z\) optimiser and random states. The
KL-controlled branch is an ablation, not a proposed Fisher method: it separates
curvature preconditioning from the shared trust-region controls. The synthetic
study uses data seeds \(0\)--\(2\), and Sachs uses training seeds \(2001\)--\(2003\).

For each seed, \(256\) trajectories are sampled without exploration from the
shared checkpoint using an independent random stream. They are never inserted
into replay. Every \(50\) continuation updates, we evaluate the mean Huber TB
loss and the median and 90th percentile of \(|\delta_\theta|\) on this fixed
probe set. Each curve is divided by its checkpoint value, and its AUC is the
average of the normalised curve over the \(500\) continuation updates. Lower
AUC therefore means greater progress per update over this local continuation;
it is not an end-to-end or wall-clock measure.

Final diagnostics use \(2{,}000\) fresh, non-exploratory DAG samples per method
and seed. If \(\widehat p_{ij}\) is the sampled inclusion probability of edge
\(i\to j\), expected SHD is

\[
  \operatorname{E\text{-}SHD}
  =
  \sum_{i\ne j}
  \left[
    A^*_{ij}(1-\widehat p_{ij})
    +(1-A^*_{ij})\widehat p_{ij}
  \right].
\]

Directed AUPRC ranks \(\widehat p_{ij}\); skeleton AUPRC ranks
\(\widehat p_{ij}+\widehat p_{ji}\) for each unordered node pair. In the
synthetic study, \(A^*\) is the generating DAG. For Sachs, these quantities are descriptive comparisons with the consensus network rather than causal identification metrics. Reported uncertainties are mean \(\pm\) standard error over the three seeds.

\subsection{Synthetic ER-2 results}

Table~\ref{tab:synthetic_fisher_appendix} summarises the continuation and
structural metrics, while Fig.~\ref{fig:synthetic_fisher_convergence} shows the
corresponding fixed-probe curves.

\begin{table}[ht!]
  \centering
  \small
  \setlength{\tabcolsep}{2.7pt}
  \caption{Synthetic \(d=10\) continuation results over three paired seeds.
  AUCs integrate checkpoint-normalised fixed-probe quantities over continuation
  updates; lower is better. Structural metrics use \(2{,}000\) final DAG
  samples per seed. Values are mean \(\pm\) standard error.}
  \label{tab:synthetic_fisher_appendix}
  \resizebox{\textwidth}{!}{%
  \begin{tabular}{lcccccc}
    \toprule
    Method
    & Loss AUC \(\downarrow\)
    & Median AUC \(\downarrow\)
    & P90 AUC \(\downarrow\)
    & Updates to 0.75 \(\downarrow\)
    & E-SHD \(\downarrow\)
    & Skel. AUPRC \(\uparrow\) \\
    \midrule
    Adam
    & \(0.856{\pm}0.023\)
    & \(0.906{\pm}0.015\)
    & \(0.928{\pm}0.013\)
    & \(417{\pm}36\)
    & \(20.27{\pm}0.53\)
    & \(0.950{\pm}0.022\) \\
    KL-controlled Adam
    & \(0.819{\pm}0.015\)
    & \(0.871{\pm}0.006\)
    & \(0.908{\pm}0.005\)
    & \(368{\pm}36\)
    & \(20.24{\pm}0.51\)
    & \(0.959{\pm}0.017\) \\
    Damped Fisher
    & \(\mathbf{0.813{\pm}0.013}\)
    & \(\mathbf{0.868{\pm}0.019}\)
    & \(\mathbf{0.901{\pm}0.004}\)
    & \(\mathbf{351{\pm}24}\)
    & \(20.23{\pm}0.50\)
    & \(0.964{\pm}0.016\) \\
    \bottomrule
  \end{tabular}%
  }
\end{table}

\FloatBarrier
\begin{figure}[ht!]
  \centering
  \includegraphics[width=\textwidth]{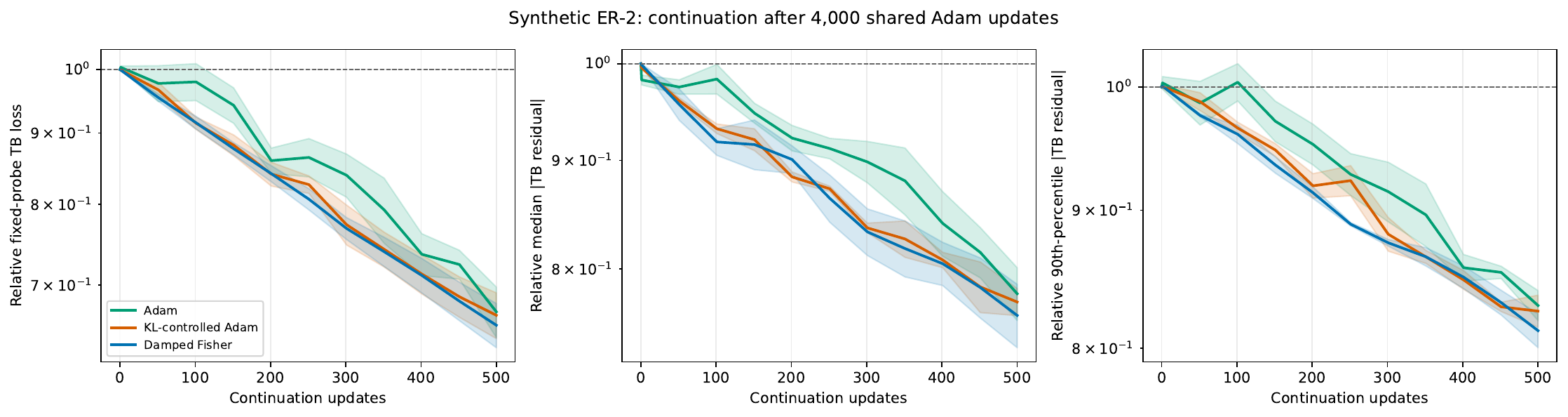}
  \caption{Synthetic ER-2 continuation from checkpoints shared after
  \(4{,}000\) Adam updates, shown as mean \(\pm\) standard error over three
  paired seeds. Damped Fisher improves the integrated fixed-probe quantities
  relative to Adam. Its smaller separation from KL-controlled Adam shows that the trust-region safeguards explain much of the gain.}
  \label{fig:synthetic_fisher_convergence}
\end{figure}
\FloatBarrier

Damped Fisher lowers loss AUC relative to Adam in every paired seed, with mean
paired difference \(-0.0429\pm0.0127\). Its difference from KL-controlled Adam
is smaller, \(-0.0055\pm0.0031\). All three Damped Fisher runs accept every
continuation update; their median PCG relative residual is approximately
\(0.045\), and the maximum statewise KL remains near the prescribed
\(10^{-4}\) cap. Final E-SHD is effectively tied, and the skeleton-AUPRC
uncertainties overlap. The synthetic result therefore supports improved local
optimisation per update over Adam; the small remaining separation from KL-controlled Adam should be viewed as proof-of-concept evidence for the Fisher direction.

\subsection{Sachs results}

The aggregate Sachs results are reported in
Table~\ref{tab:sachs_fisher_main}, with the corresponding fixed-probe
continuation curves shown in Fig.~\ref{fig:sachs_fisher_convergence}. Damped
Fisher attains relative fixed-probe loss
AUC \(0.734\pm0.004\), compared with \(0.751\pm0.005\) for Adam and
\(0.734\pm0.006\) for KL-controlled Adam. The corresponding P90-residual AUCs
are \(0.870\pm0.007\), \(0.881\pm0.009\), and \(0.876\pm0.009\), respectively.
Relative to Adam, the paired loss-AUC difference is
\(-0.0166\pm0.0072\); relative to KL-controlled Adam, it is
\(0.0009\pm0.0089\). Thus Damped Fisher improves on plain Adam, but is not
separated from the controlled ablation on mean loss convergence.

\begin{figure}[ht!]
  \centering
  \includegraphics[width=0.9\textwidth]{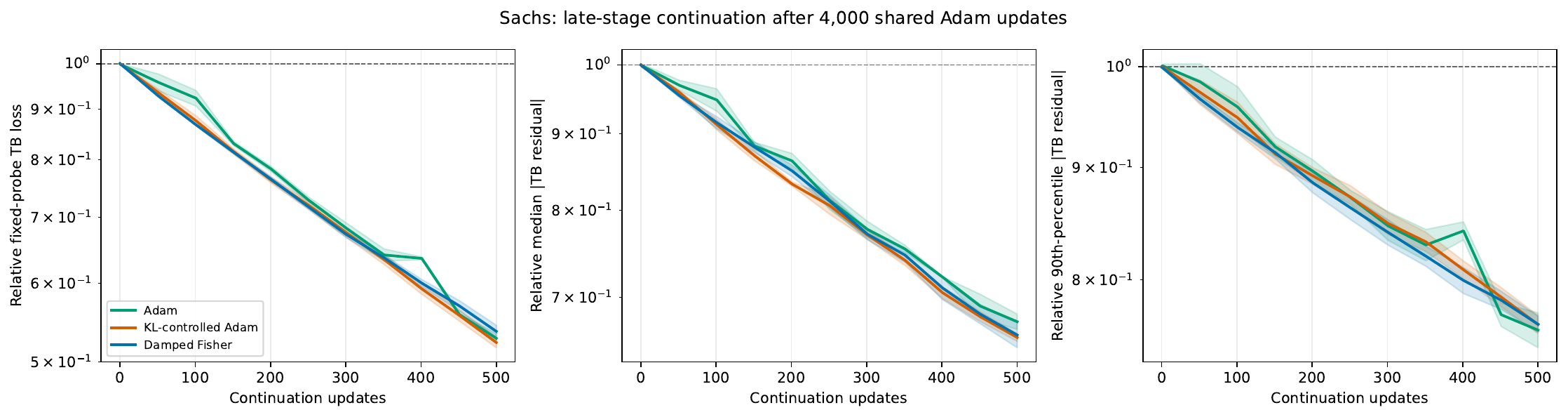}
  \caption{Late-stage Sachs continuation from checkpoints shared after
  \(4{,}000\) Adam updates, shown as mean \(\pm\) standard error over three
  paired seeds. Damped Fisher and KL-controlled Adam improve integrated
  fixed-probe loss relative to Adam, but their mean loss AUCs are essentially
  tied. The figure measures the local \(500\)-update continuation rather than
  end-to-end training.}
  \label{fig:sachs_fisher_convergence}
\end{figure}

\begin{figure}[ht!]
  \centering
  \includegraphics[width=0.8\textwidth]{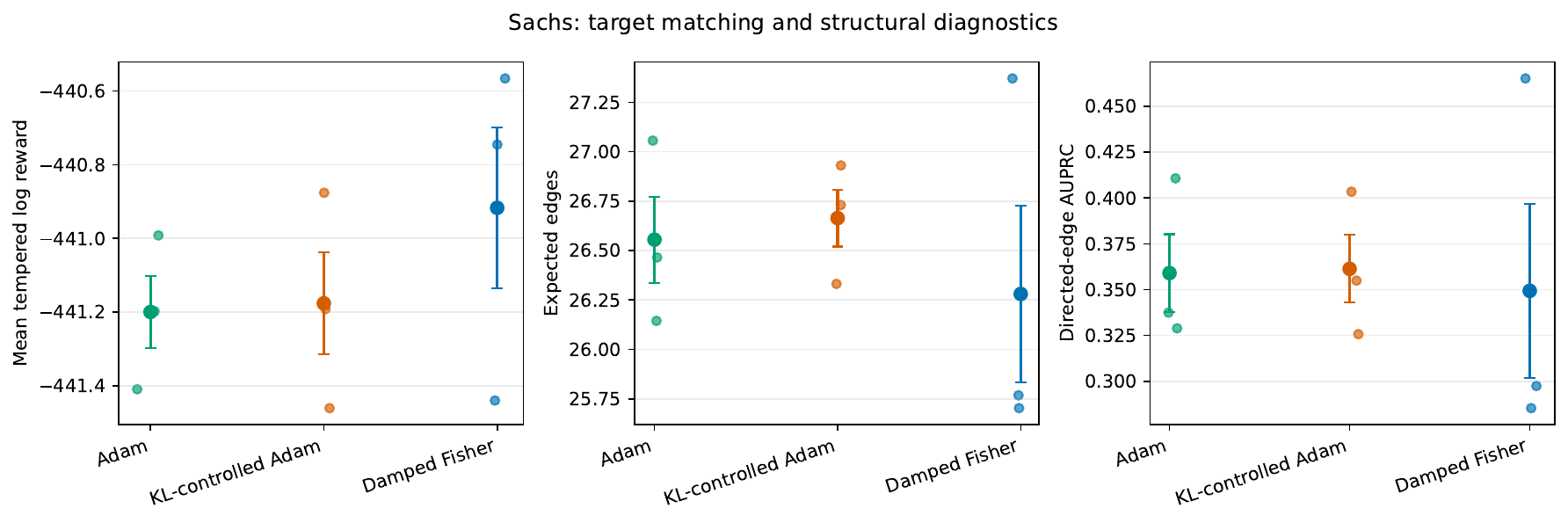}
  \caption{Final Sachs diagnostics from \(2{,}000\) DAG samples per seed.
  Small points are individual seeds; large markers and error bars are mean
  \(\pm\) standard error. Damped Fisher has slightly higher target score and
  lower graph size, but the methods overlap and directed AUPRC does not improve.
  The consensus network is only a structural reference.}
  \label{fig:sachs_final_diagnostics}
\end{figure}
\FloatBarrier

Figure~\ref{fig:sachs_final_diagnostics} summarises the final target and
structural diagnostics. Directed AUPRC is \(0.349\pm0.047\) for Damped Fisher,
\(0.359\pm0.021\) for Adam, and \(0.361\pm0.018\) for KL-controlled Adam;
E-SHD is \(28.80\pm0.19\), \(28.83\pm0.02\), and \(28.98\pm0.11\),
respectively. Expected graph size changes only from \(26.56\pm0.22\) for Adam
and \(26.66\pm0.14\) for KL-controlled Adam to \(26.28\pm0.45\) for
Damped Fisher. Figure~\ref{fig:sachs_posterior_graphs} shows the Damped Fisher
posterior marginals for illustration; its \(0.25\) threshold is not used for
evaluation.

The corrected Sachs experiment supports a modest per-update optimisation effect
relative to Adam, but no reliable structural advantage or advantage over the
KL-controlled loss update. It does not establish causal identification or
intrinsic sparsification by Fisher geometry. All methods target the same
edge-penalised law; the Fisher
models covariance among STOP and the competing valid edge additions at the
visited partial graphs. Together, the synthetic and Sachs studies are
proof-of-concept tests of a sampled, non-diagonal neural Fisher on globally
constrained graph spaces.

\end{document}